\documentclass[11pt]{article}    

\usepackage[T1]{fontenc}
\usepackage{lmodern }
\usepackage{graphicx}
\usepackage{url}
\usepackage{amsmath}    
\usepackage{amsthm}
\usepackage{amssymb}
\usepackage{fullpage}
\usepackage{epstopdf}
\usepackage{enumerate}
\usepackage{color}
\usepackage{mathtools}
\usepackage[table]{xcolor}
\usepackage{verbatim}
\usepackage{multirow}
\usepackage{subcaption}
\usepackage{setspace}
\usepackage[authoryear,round,comma,compress]{natbib}
\bibpunct{(}{)}{,}{a}{}{,}
\usepackage[english]{babel}
\usepackage[utf8]{inputenc}
\usepackage[colorinlistoftodos]{todonotes}
\usepackage[version=3]{mhchem}
\usepackage{chemfig}
\usepackage{fancyhdr}
\usepackage{amsfonts}
\usepackage{microtype}
\usepackage{setspace}
\usepackage[usenames,dvipsnames]{pstricks}
\usepackage{epsfig}
\usepackage{pst-grad} 
\usepackage{pst-plot} 
\usepackage[space]{grffile} 
\usepackage{etoolbox} 
\usepackage{float}
\usepackage{mathrsfs}
\usepackage[margin=1in]{geometry}
\usepackage{bm}
\usepackage{bookmark}
\usepackage{hyperref}
\usepackage{bbm}
\usepackage{setspace}
\usepackage{framed}
\usepackage{verbatim}
\usepackage{tikz,pgfplots}
\usetikzlibrary{arrows,intersections}
\usepackage{upgreek}
\usepackage{eufrak}
\usepackage[ruled, lined, noend, linesnumbered]{algorithm2e}
\usepackage{accents}
\usepackage{cprotect}
\usepackage{breqn}
\usepackage{diagbox}
\usepackage{clipboard}
\newclipboard{review}

\usepackage[skip=1ex]{caption}
\newcommand{\oldnewam}[1]{\color{black}{#1} \color{black}}

\renewcommand{\l}{\left}
\renewcommand{\r}{\right}
\makeatletter
\def\blfootnote{\xdef\@thefnmark{}\@footnotetext}
\makeatother

\makeatletter
\newcommand\mathgreek[1]{\expandafter\@mathgreek\csname c@#1\endcsname}
\newcommand\@mathgreek[1]{%
	\ifcase#1\or C_1\or C_2\or C_3\or C_4\or C_5\or C_6\or
	C_7\or C_8\or C_9\or C_{10}\or C_{11}\or C_{12}\or C_{13}\or  C_{14}\or C_{15}\or
	C_{16}\or C_{17}\or C_{18}\or C_{19}\or C_{20}\or C_{21}\or C_{22}\or C_{23}\or C_{24}\or C_{25}\or C_{26}\or C_{27}\or C_{28}\or C_{29} \or C_{30} \or C_{31} \or C_{32} \or C_{33}
		\@ctrerror\fi}

\newcounter{greekvars}
\renewcommand\thegreekvars{\mathgreek{greekvars}}

\newcommand\constvar[1][]{%
	\if\relax\detokenize{#1}\relax
	\stepcounter{greekvars}%
	\else
	\refstepcounter{greekvars}\ltx@label{#1}%
	\fi
	\thegreekvars
}
\makeatother
\newcommand{\constref}[1]{\ref*{#1}}

\usepackage{tikz,colortbl}
\usetikzlibrary{calc}
\usepackage{zref-savepos}
\newcounter{NoTableEntry}
\renewcommand*{\theNoTableEntry}{NTE-\the\value{NoTableEntry}}

\renewcommand{\cal}{\mathcal}

\newcommand{\vectorgreek}[1]{\ensuremath \boldsymbol{#1}}

\long\def\symbolfootnote[#1]#2{\begingroup%
	\def\thefootnote{\fnsymbol{footnote}}\footnote[#1]{#2}\endgroup}

\DeclareMathOperator*{\argmin}{arg\,min}

\newcommand {\cX}{\cal{X}}

\newcommand {\cB}{\cal{B}}

\newcommand {\cE}{\cal{E}}
\newcommand {\cS}{\cal{S}}
\newcommand {\cI}{\cal{I}}

\newcommand {\urho}{\underline{\rho}}

\newcommand {\cF}{{\cal{F}}}
\newcommand {\cH}{{\cal{H}}}
\newcommand {\cK}{{\cal{K}}}
\newcommand {\cO}{{\cal{O}}}
\newcommand {\cM}{{\cal{M}}}
\newcommand {\cD}{{\cal{D}}}
\newcommand {\cQ}{{\cal{Q}}}

\newcommand {\cA}{{\cal{A}}}
\newcommand {\cP}{{\cal{P}}}
\newcommand {\cR}{{\cal{R}}}
\newcommand {\cL}{{\cal{L}}}

\newcommand {\cW}{{\cal{W}}}
\newcommand {\cV}{{\cal{V}}}
\newcommand{\sfB}{\mathsf{B}}

\newcommand{\sfp}{\mathsf{p}}

\newcommand{\sfP}{\mathsf{P}}

\newcommand{\sfm}{\mathsf{m}}

\newcommand{\vomega}{\vectorgreek{\omega}}

\newcommand {\beq}{\begin{equation}}
\newcommand {\eeq}{\end{equation}}
\newcommand {\beqn}{\begin{equation*}}
\newcommand {\eeqn}{\end{equation*}}
\newcommand {\bear}{\begin{eqnarray}}
\newcommand {\eear}{\end{eqnarray}}
\newcommand {\bearn}{\begin{eqnarray*}}
	\newcommand {\eearn}{\end{eqnarray*}}
\newcommand {\bal}{\begin{align}}
\newcommand {\eal}{\end{align}}
\newcommand {\baln}{\begin{align*}}
\newcommand {\ealn}{\end{align*}}
\newcommand {\bbmatrix}{\begin{bmatrix}}
	\newcommand {\ebmatrix}{\end{bmatrix}}
\newcommand {\bpmatrix}{\begin{pmatrix}}
	\newcommand {\epmatrix}{\end{pmatrix}}

\newcommand{\Indlr}[1]{\mathbbm{1}\l\{#1 \r\}}

\newcommand {\rR}{\mathbb{R}}
\newcommand {\rE}{\mathbb{E}}
\newcommand {\Var}{\mathbb{V}\mathrm{ar}}
\newcommand {\rP}{\mathbb{P}}
\newcommand {\rI}{\mathbb{I}}
\newcommand {\rN}{\mathbb{N}}

\newcommand{\Holder}{H\"{o}lder}

\newcommand{\rPlr}[1]{\rP\l\{ #1\r\}}
\newcommand{\rElr}[1]{\rE\l[ #1\r]}
\newcommand{\Varlr}[1]{\Var\l[ #1\r]}

\newcommand{\sgn}{\mathrm{sign}}
\newcommand{\bxi}{\bm{\xi}}
\newcommand{\lbeta}{\underline{\beta}}
\newcommand{\ubeta}{\bar{\beta}}
\newcommand{\tbeta}{\text{\scalebox{.6}{$\tilde{\beta}$}}}

\newtheorem{theorem}{Theorem}
\newtheorem{definition}[theorem]{Definition}
\newtheorem{lemma}[theorem]{Lemma}
\newtheorem*{lemma*}{Lemma}
\newtheorem{prop}[theorem]{Proposition}
\newtheorem{corollary}[theorem]{Corollary}
\newtheorem{assumption}{Assumption}
\newtheorem{example}{Example}
\newtheorem{remark}{Remark}
\numberwithin{theorem}{section}
\numberwithin{equation}{section}

\normalsize

\begin{document}
	\title{\vspace{-0.4cm}	
		Smoothness-Adaptive Contextual Bandits\\\vspace{0.3cm}
	}
	\author{
		{\sf Yonatan Gur}
		\\Stanford University\and
		{\sf Ahmadreza Momeni}
		\\Stanford University\and
		{\sf Stefan Wager}\thanks{Correspondence: {\tt ygur@stanford.edu}, {\tt amomenis@stanford.edu}, {\tt swager@stanford.edu}.}
		\\ Stanford University
		\vspace{0.5cm}}
	\date{\today}
	
	\maketitle

	\vspace{-0.1cm}
\setstretch{1.17}
\begin{abstract}
\noindent We study a non-parametric multi-armed bandit problem with stochastic covariates, where a key complexity driver is the smoothness of payoff functions with respect to covariates. Previous studies have focused on deriving minimax-optimal algorithms in cases where it is a priori known how smooth the payoff functions are. In practice, however, the smoothness of payoff functions is typically not known in advance, and misspecification of smoothness may severely deteriorate the performance of existing methods. In this work, we consider a framework where the smoothness of payoff functions is not known, and study when and how algorithms may adapt to unknown smoothness. First, we establish that designing algorithms that adapt to unknown smoothness of payoff functions is, in general, impossible. However, under a self-similarity condition (which does not reduce the minimax complexity of the dynamic optimization problem at hand), we establish that adapting to unknown smoothness is possible, and further devise a general policy for achieving smoothness-adaptive performance. Our policy infers the smoothness of payoffs throughout the decision-making process, while leveraging the structure of off-the-shelf non-adaptive policies. We establish that for problem settings with either differentiable or non-differentiable payoff functions, this policy matches (up to a logarithmic scale) the regret rate that is achievable when the smoothness of payoffs is known a priori.

\vspace{0.18cm}
	\end{abstract}

\setstretch{1.48}
\section{Introduction}

A well-studied dynamic optimization framework that captures the trade-off between new information acquisition (exploration) and optimization of payoffs based on available information (exploitation) is the multi-armed bandit (MAB) framework, originated by the work of \cite{thompson1933likelihood} and \cite{robbins1952aspects}. An important generalization of this framework, where the decision maker also has access to covariates that can be informative about the effectiveness of different actions, is typically referred to as the contextual MAB problem \citep{woodroofe1979one}. The contextual MAB framework has been applied for analyzing sequential experimentation in many application domains, including pricing \citep[e.g.,][]{cohen2016feature,qiang2016dynamic,ban2019personalized,bastani2019meta,javanmard2019dynamic, wang2019multi}, product recommendations \citep[e.g.,][]{chu2011contextual,Chandrashekar2017recommendation,bastani2018sequential,agrawal2019mnl,gur2019adaptive,kallus2020dynamic}, and healthcare \citep[e.g.,][]{tewari2017ads,chick2018bayesian,zhou2019tumor,bastani2020online}. 
	
Following \citet{woodroofe1979one}, most of the analysis of contextual MAB problems assumes a parametric (usually linear) model for the payoff functions that are associated with different actions; see, e.g., \citet{goldenshluger2013linear} and \citet{bastani2020online} for some notable results. Recently, however, there has been a growing interest in studying \emph{non-parametric} contextual MAB formulations, which make fewer structural assumptions, are typically more robust, and can be applied to a more general class of problems, especially when less is known about the structure of payoff functions. One of the main findings of this line of work is that, in non-parametric contextual MAB formulations, the smoothness of the payoff functions is a key driver of the difficulty of the dynamic optimization problem at hand \citep{rigollet2010nonparametric,perchet2013multi,hu2019smooth}.
Qualitatively, the smoother the payoff functions are, the further one may extrapolate payoff patterns over the covariate space---and the less one must explore in order to guarantee good performance.

We next illustrate this phenomenon through the problem of artwork selection on Netflix. When Netflix recommends a title, it also needs to select an image to display along with the recommendation. Different images may induce different probabilities of playing the movie. Given the personal viewing history of the user, for each recommended title Netflix aims to select imagery that maximizes the probability of playing that title. A simple version of this problem is described in \cite{Chandrashekar2017recommendation}, where two different artworks are available for the movie \textit{Good Will Hunting} (see the top parts of Figure \ref{fig-GWH}).

\begin{figure}[ht]
\centering	\includegraphics[width=0.8\textwidth]{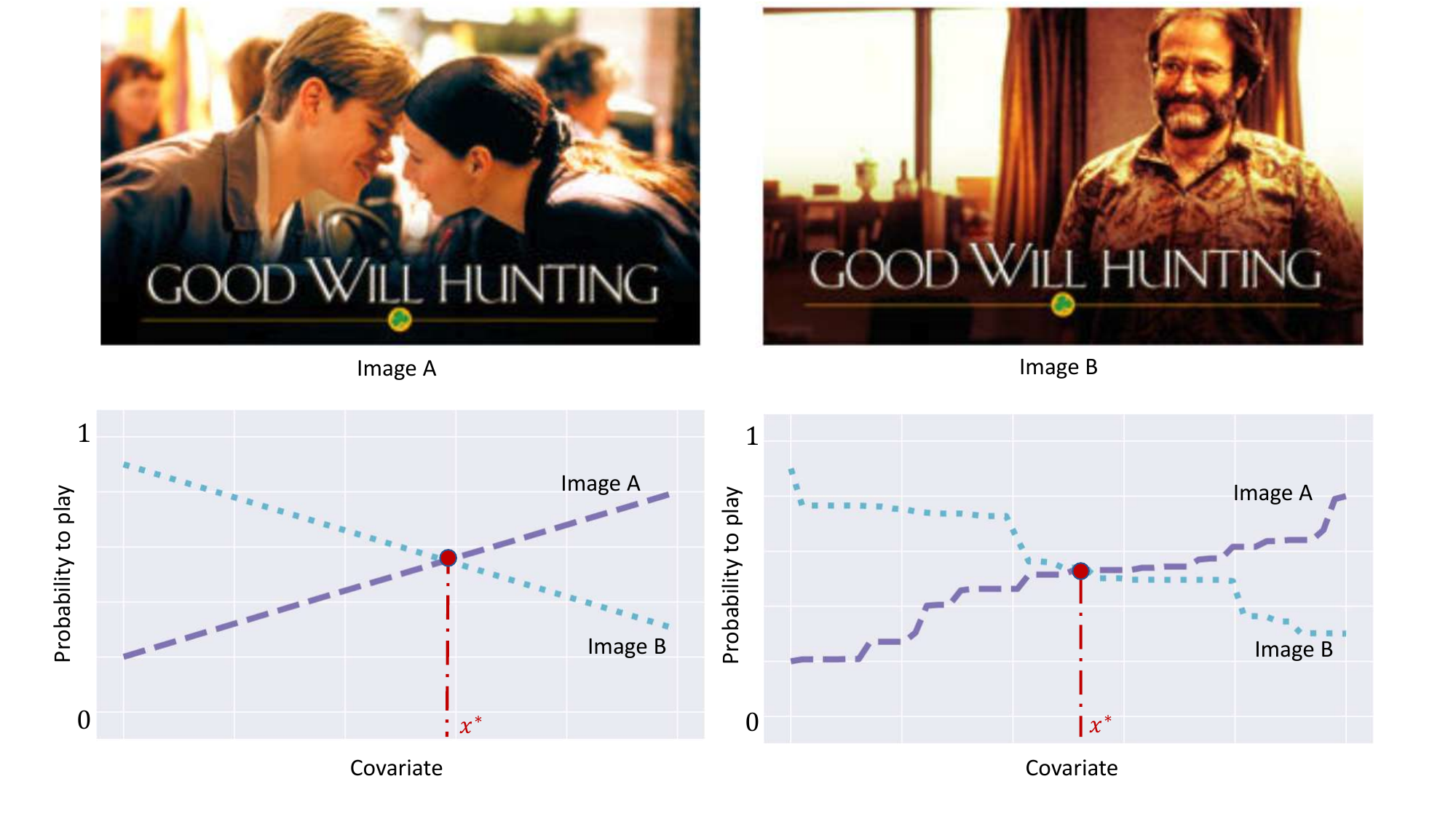}\vspace{-0.1cm}
\caption{\footnotesize \textit{Top:} Example of artwork selection on Netflix for recommending the movie \textit{Good Will Hunting} (for details and discussion see \citealt{Chandrashekar2017recommendation}); \textit{Bottom:} The probability of users playing the recommended title as a function of a covariate (the normalized difference of romance and comedy scores assigned to each user) when either image A (dashed line) or image B (dotted line) is shown, in two different scenarios: \textit{Bottom Left:} users' behavior changes linearly as a function of the covariate; \textit{Bottom Right:} users' behavior changes more abruptly as a function of the covariate. In each case, $x^\ast$ denotes the covariate at which the optimal imagery switches.}\vspace{-0.3cm}
\label{fig-GWH}
\end{figure}

The bottom of Figure~\ref{fig-GWH} includes two plots that illustrate different scenarios of how the probability of playing the title changes as a function of a one-dimensional covariate for two artwork options: image A and image B.
\color{black}(Examples of such covariates include the age of the viewer, their frequency of watching movies, or scores that are based on their particular viewing history, such as their tendency to watch romances versus comedies; for more details see \citealt{Chandrashekar2017recommendation}.)
\color{black}
Let $x^\ast$ denote a covariate in which the optimal imagery switches. In particular, for covariates that belong to the interval $[0,x^\ast]$ the optimal imagery to display is image B; otherwise, it is image A. In the scenario illustrated on the bottom-left part of Figure~\ref{fig-GWH}, users' behavior is smooth and changes linearly with respect to the covariate. In this case, any observation of a user's behavior, even when the covariate is not close to $x^\ast$, is informative and can be utilized to estimate the probability lines and the crossing point $x^\ast$. By contrast, the bottom-right part of Figure~\ref{fig-GWH} depicts a scenario where the probability of playing each title changes more abruptly as a function of the covariate. In this case, observations with covariates that are not close to $x^\ast$ are less informative and cannot be easily utilized to estimate the crossing point~$x^\ast$. As a result, the second scenario requires more experimentation in order to determine optimal decision regions. When payoffs are not monotone or smooth functions of the covariates, optimal decision regions might be non-convex and complex to identify, and required experimentation rates further increase.
	
Previous studies of non-parametric contextual MAB problems typically assume prior knowledge of the worst-case smoothness of payoff functions. A standard approach is to assume that payoff functions are $(\beta, L)$-H\"{o}lder (see Definition \ref{def-holder-class}) for some \textit{known} parameters $\beta$ and $L$, and develop policies that are predicated on this assumption. For example, \citet{rigollet2010nonparametric} and \citet{perchet2013multi} develop minimax rate-optimal algorithms when payoff functions are assumed to be Lipschitz or ``rougher" (that is, when $0 < \beta \leq 1$); more recently, \citet*{hu2019smooth} extends this analysis to the ``smoother'' case (where $\beta > 1$).

In practice, however, the class of functions to which payoff functions belong is often unknown, and misspecification of smoothness may cause significant deterioration in the performance of existing methods (see Example \ref{exp:cost-smoothness-misspecification} in \S\ref{subsec:costExample}). While underestimating the smoothness of payoff functions leads to excessive and unnecessary experimentation, overestimating the smoothness might lead to insufficient experimentation; both cases may result in poor performance relative to that which could have been achieved with accurate information on the smoothness. The focus of this paper is on studying when and how algorithms may \emph{adapt} to unknown smoothness, in the sense of achieving, without prior knowledge of the smoothness of payoffs, the best performance that is achievable when smoothness is a priori known.

\color{black}
\subsection{Main Contributions}\label{subsection-contributions}
\color{black}
Our contributions are in (1) formulating a non-parametric contextual MAB problem where the smoothness of payoff functions is a priori unknown;
(2) analyzing the complexity of adapting to smoothness, and establishing that smoothness adaptivity is in general impossible;
and (3) identifying a self-similarity condition that makes it possible to achieve smoothness adaptivity, and devising a general policy that leverages this condition to guarantee rate optimality without prior information on the smoothness of payoffs. More specifically, our contribution is along the following lines.

\textit{(1) Modeling.} We formulate a non-parametric contextual MAB problem where the smoothness of payoff functions is a priori unknown: the payoff functions are assumed to belong to a H\"{o}lder class of functions with some unknown H\"{o}lder exponent. We identify a policy as \emph{smoothness-adaptive} if for any problem instance it guarantees the optimal regret rate as a function of the H\"{o}lder exponent that characterizes that instance, up to a multiplicative term that is poly-logarithmic in the horizon length, and a multiplicative constant that may depend on other problem parameters (such as the dimension of the covariate space); see Definition \ref{def:adaptivity}. In that sense, smoothness-adaptive policies guarantee (up to a logarithmic factor) the minimax regret rate that characterizes the achievable performance when the smoothness parameter is a priori known. Our formulation allows for any arbitrary range of the smoothness parameter, and thus captures a large variety of real-world phenomena.

\textit{(2) Impossibility of adaptation.} We establish a lower bound on the best achievable performance when two different classes of payoff functions (characterized by two different smoothness exponents) are considered simultaneously. Through this lower bound we show that adaptively achieving rate-optimal performance uniformly over different classes of smooth payoff functions is impossible. In that sense, adapting to unknown smoothness carries a non-trivial cost in sequential experimentation. This is despite the fact that smoothness-adaptive estimation of non-parametric functions is possible (see, e.g., \citealt{lepskii1992asymptotically}). Thus, this impossibility result highlights the fundamental difference between the complexities of non-parametric function estimation and the non-parametric contextual MAB problem.

\color{black}
From a formal perspective, the lower bound we establish is based on reducing the problem at hand to a hypothesis-testing problem by introducing a novel construction of a set of problem instances. This set consists of a nominal problem instance with smoothness parameter $\gamma$ and some other problem instances with smoothness parameter $\beta \le \gamma$, each of which differs from the nominal one only over a specific region of the covariate space. These problem instances connect the amount of exploration to the ability to identify the correct smoothness parameter, and designed for establishing that if a policy guarantees rate-optimal performance over a class of problems with smooth payoff functions, it is likely to underexplore when payoff functions are ``rougher."
\color{black}

\textit{(3) Smoothness adaptivity and policy design under self-similar payoffs.} To advance beyond the general impossibility of adapting to unknown smoothness, we turn to consider smoothness adaptivity when payoff functions are \emph{self-similar}. Self-similarity has been used to study adaptivity problems in the statistics literature. For example, while constructing smoothness-adaptive confidence bands is impossible in general \citep{low1997nonparametric}, it becomes possible under a self-similarity assumption. In particular, smoothness-adaptive confidence bands can be obtained by applying Lepski's approach for identifying the optimal bandwidth (which corresponds to the true smoothness) by ``comparing" estimation bias and stochastic error \citep{lepski1997optimal}; see, e.g., \cite{picard2000adaptive}, \cite{gine2010confidence}, \cite{bull2012honest}, and \cite{bull2013adaptive}. In the absence of direct access to the estimation bias, Lepski's method provides a general approach for constructing a proxy for it through the absolute difference between estimators with different bandwidths. A new variant of Lepski's method that is tailored to the dynamic nature of our problem is advanced in our smoothness estimation sub-routine in~\S\ref{subsec:smoothness-estimation}.

\color{black}
First, we establish that self-similarity does not reduce the minimax complexity of the problem at hand. Then, we show that when payoffs are self-similar, it is possible to design smoothness-adaptive policies. We devise a general policy \oldnewam{termed Smoothness-Adaptive Contextual Bandits (\texttt{SACB})} that, under self-similarity, guarantees rate-optimal performance without prior information on the smoothness of payoffs. \color{black}
The \texttt{SACB} policy adapts Lepski's method to efficiently estimate the smoothness of self-similar payoff functions throughout the sequential decision process, while leveraging the structure of effective off-the-shelf non-adaptive policies that are designed to perform well under accurate smoothness specification.
\color{black}
We establish that when our policy is paired with off-the-shelf input policies that guarantee the optimal regret rate under accurate smoothness specification, it guarantees (up to a logarithmic factor) the latter regret rate without any prior information on the smoothness of payoffs.

\color{black}
The \texttt{SACB} policy and its analysis therefore show how a variant of Lepski's approach, which was designed for learning smoothness in \emph{static} settings, can be appropriately adapted for obtaining performance guarantees in a \emph{dynamic} operational setting that is applied in many practical settings. Particularly, we show that such an approach, together with leveraging the structure of effective non-adaptive policies, can lead to smoothness adaptivity and near-optimality in the non-parametric contextual MAB problem without prior knowledge of the smoothness of payoffs.
\color{black}
We demonstrate our approach by leveraging non-adaptive policies designed for payoff functions that are at most Lipschitz smooth \citep{perchet2013multi} and at least Lipschitz smooth \citep{hu2019smooth} to guarantee rate optimality without prior information on the underlying payoff smoothness.

\subsection{Related Literature}\label{sec:lit-rev}

\noindent \textbf{Parametric and Non-parametric Approaches to Contextual MAB.}
Most of the literature on contextual bandits assumes parametric payoff functions. Some researchers have studied this setting when covariates are independently drawn from an identical distribution. For example, \cite{goldenshluger2013linear}, \cite{bastani2017mostly}, and \cite{bastani2020online} consider linear payoff functions. On the other hand, \cite{langford2008epoch} and \cite{dudik2011efficient} study the problem of finding the best mapping from covariates to arms among a finite set of hypotheses. In addition, \cite{wang2005bandit} considers a general relationships between the parameters of payoff functions and covariates.
In contrast to these studies, some other papers consider settings with covariates that are selected by an \textit{adversary}; see \cite{bubeck2012regret} and the references therein. 

In addition to these parametric approaches, the contextual MAB problem has also been addressed from a \textit{non-parametric} point of view to account for general relationships between covariates and mean rewards. \cite{yang2002randomized}, which initiated this line of research, combined an $\epsilon$-greedy-type policy with non-parametric estimation methods such as nearest neighbors to achieve \textit{strong consistency}. This solution concept ensures that the total reward collected by the agent is almost surely asymptotically equivalent to those obtained by always pulling the best arm. Following this work, stronger results have been established for the regret rate. \cite{rigollet2010nonparametric} introduces the UCBogram policy, which decomposes the covariate space into bins and follows a traditional UCB policy in each bin separately. \cite{perchet2013multi} improves upon this result by introducing the Adaptively Binned Successive Elimination (\texttt{ABSE}) policy, which implements an increasing refinement of the covariate space and achieves the minimax regret rate. Recently, \cite{hu2019smooth} extends this framework to the case of smooth differentiable functions. Finally, \cite{reeve2018k} proposes a kNN-UCB policy that achieves the minimax regret rate and also adapts to the intrinsic dimension of data. All these studies, however, assume that the smoothness of the payoff functions is known a priori.

\textbf{Non-parametric Continuum-Armed Bandit.} A problem in the literature that shares some similarities with the non-parametric contextual MAB problem is the non-parametric continuum-armed bandit problem (see, e.g., \citealt{agrawal1995continuum} and  \citealt{bubeck2009pure}), where, given a single unknown function $f:\rR^d\rightarrow \rR$, at each time $t$ the agent selects an action $a_t$ from the continuous action set, and then collects and observes a noisy reward with mean $f(a_t)$. While the continuum-armed bandit problem and the contextual bandit problem we consider here are fundamentally different problems, they share the feature that smoothness plays a key role in terms of minimax complexity; see, e.g., \cite{kleinberg2005nearly}, \cite{auer2007improved}, and \cite{bubeck2009online}. More recently, \cite{locatelli2018adaptivity} studies the problem of adapting to unknown smoothness in a continuum-armed bandit problem in the case where the unknown function $f(\cdot)$ is at most Lipschitz smooth, and shows that under conventional assumptions it is impossible to minimize the cumulative regret rate when the underlying smoothness is unknown.

\textbf{Model Selection in Contextual MAB.} One may view selecting the right smoothness parameter as a model selection problem. There are a few recent works that study such problems in settings that are different from ours. For instance, \cite{chatterji2020osom} considers a linear contextual MAB where the payoffs may or may not depend on covariates, and provides a policy that can achieve optimality simultaneously for both cases. In addition, \cite{foster2019model} studies the problem of adapting to the dimension of the class of payoff functions in a linear contextual MAB setting.

\textbf{Adaptive Non-parametric Methods.}
For the general theory on adaptive non-parametric estimation we refer the reader to \cite{lepskii1992asymptotically}. This line of research includes various approaches. For example, \cite{donoho1994ideal}, \cite{donoho1995wavelet}, and \cite{juditsky1997wavelet} deploy techniques based on wavelets, \cite{lepski1997optimal} proposes a kernel-based method, and \cite{goldenshluger1997spatially} develops a method that is based on local polynomial regression.

A related line of research studies the construction of adaptive non-parametric confidence intervals. The work of \cite{low1997nonparametric} showed that, in general, it is impossible to construct adaptive confidence bands simultaneously over different classes of H\"{o}lder functions; for recent results on the impossibility of adaptive confidence intervals, see \citet{armstrong2018optimal} and the references therein. Following that work, several studies have focused on identifying conditions under which adaptive confidence band construction is feasible. A well-studied condition is that of self-similarity, which is first used in this covariate by \cite{picard2000adaptive} using wavelet methods for pointwise purposes. Later on, self-similarity is also used by \cite{gine2010confidence} to construct confidence bounds over finite intervals. \oldnewam{In addition, \cite{bull2012honest} and \cite{bull2013adaptive} use a self-similarity condition in the covariate of constructing honest and adaptive confidence bands.} Aside from these studies, self-similarity is used in a variety of other non-parametric problems and applications, including high-dimensional sparse signal estimation \citep{nickl2013confidence}, binary regression \citep{mukherjee2018optimal}, and $L_p$-confidence sets \citep{nickl2016sharp}, to mention a few.

In a contextual MAB setting, \citet{qian2016randomized} are the first to consider a self-similarity condition for establishing performance guarantees without precise smoothness knowledge in a non-differentiable case (with $0<\beta\le 1$). By applying a standard version of Lepski's method \citep{lepski1997optimal, picard2000adaptive, gine2010confidence, bull2012honest, bull2013adaptive}, they propose a policy with a multiplicative smoothness adaptivity ``cost" of order $\l(\log T\r)^{c^\prime(d)}$, where $T$ is the horizon length and $c^\prime(d)$ is a function that grows quadratically with the covariate dimension $d$ (which tends to be large in practical settings). Since $c^\prime(d) $ scales with $d$, the method they provide is not smoothness adaptive in the sense that is defined in the present paper. We propose a novel meta-policy based on a variant of Lepski's method that is tailored to the dynamic nature of the problem at hand, and establish for it a multiplicative smoothness adaptivity ``cost" of order $\l(\log T\r)^{c}$, where $c$ is a constant that does not grow with the dimension $d$. Our paper further grounds self-similarity as an important condition for adapting to unknown smoothness through: $(i)$~establishing that, in general, smoothness-adaptive policy design is impossible without imposing additional conditions; $(ii)$ showing that self-similarity does not reduce the minimax complexity of the problem; and $(iii)$ showing that self-similarity can be leveraged in a new way that allows the design of smoothness-adaptive policies for a general class of problems.

	\vspace{-0.1cm}
\section{Problem Formulation}\label{sec: prob_form}\vspace{-0.1cm}
	
We next formulate the non-parametric contextual MAB problem with unknown smoothness. \S\ref{sec-model-assumptions} includes our main modeling assumptions. In \S\ref{subsec:costExample} we discuss and illustrate the performance reduction that is caused by misspecifying the smoothness under existing methods. In \S\ref{subsec:smoothness-adaptive} we formalize the adaptivity notion that is used as a policy design goal in the analysis that will follow.

\medskip
\noindent \textbf{Reward and Feedback Structure.} Let $\cK = \{1,2\}$ be a set of actions (arms) and let $\mathcal{T} = \{1,\dots,T\}$ denote a sequence of decision epochs. (We focus on a setting with two actions only for ease of exposition and analysis and expect all results to hold for any action set of finite cardinality.) At each time period $t\in\cal{T}$, a decision maker observes a covariate $X_t \in  [0,1]^d$ that is realized according to an unknown distribution $\bm{\mathrm{P}}_X$, and then selects one of the two actions. When an action $k\in\cK$ is selected at time $t\in\cal{T}$, a reward
$
Y_{k,t} \sim \bm{\mathrm{P}}^{(k)}_{Y|X}
$
is realized and observed such that $Y_{k,t} \in \{0,1\}$, where $\bm{\mathrm{P}}^{(k)}_{Y|X}$ denotes the payoff distribution conditional on the covariate $X_t$ and the selected action $k$. Equivalently, the rewards $Y_{k,t}$ may be expressed as follows:\vspace{-0.15cm}
\[
Y_{k,t} = f_k(X_t) + \epsilon_{k,t},\vspace{-0.10cm}
\]
where $f_k(X_t)=\rElr{Y_{k,t} \;\middle|\; X_t}$ and $\epsilon_{k,t}$ is a random variable such that $\rElr{\epsilon_{k,t} \;\middle|\; X_t} = 0$. The conditional distributions $\bm{\mathrm{P}}^{(k)}_{Y|X}$ and the payoff functions $f_k$ are assumed to be unknown.

\medskip
\noindent \textbf{Admissible Policies.} Let $U$ be a random variable defined over probability space $(\mathbb{U}, \mathcal{U}, \bm{\mathrm{P}}_u)$. Let
$\pi_t:[0,1]^{d \times t} \times [0,1]^{t-1} \times \mathbb{U} \rightarrow \mathcal{K}$ for $t = 1,2,3,\dots$ be a sequence of measurable functions given by\vspace{-0.05cm}
\[
\pi_t = \begin{cases}
\pi_1(X_1, U) & t=1,\\
\pi_t(X_{t}, \dots, X_{1}, Y_{t-1}, \dots, Y_1 ,U) & t=2,3,\dots
\end{cases}.
\]

\noindent (We abuse notation by also denoting the action at time $t$ by $\pi_t \in \mathcal{K}$.) The mappings $\{\pi_t;t=1,\dots,T\}$ and the distribution $\bm{\mathrm{P}}_u$ together define the class of admissible policies, denoted by $\Pi$.

\medskip
\noindent \textbf{Performance.} For a problem instance $\sfP=\l(\bm{\mathrm{P}}_X,\bm{\mathrm{P}}^{(1)}_{Y|X},\bm{\mathrm{P}}^{(2)}_{Y|X}\r)$, let $\pi^\ast(\sfP) = (\pi^\ast_t(\sfP); t=1,2,\dots)$ denote the oracle rule, which under knowledge of the problem instance $\sfP$ (including the functions $f_k$), prescribes at each time $t$ the best action given the realized covariate $X_t$; that is, $\pi_t^\ast(\sfP) = \arg\max_{k\in \mathcal{K}} f_{k}(X_t)$ for all~$t\in\cal{T}$.
The performance of a policy $\pi=\{\pi_t;\;t=1,\dots,T\}$ is measured in terms of expected regret relative to the oracle performance:\vspace{-0.1cm}
\[
\mathcal{R}^\pi(\sfP;T) \coloneqq \mathbb{E}^\pi \left[ \sum\limits_{t=1}^T f_{\pi^\ast_t(\sfP)}(X_t) - f_{\pi_t}(X_t) \right].\vspace{-0.1cm}
\]
A prominent characteristic of a given problem instance $\sfP$ that directly impacts the achievable regret rate is the smoothness with which the payoff functions $f_1$ and $f_2$ and, correspondingly, the conditional distributions $\bm{\mathrm{P}}^{(k)}_{Y|X}$ and $\bm{\mathrm{P}}^{(2)}_{Y|X}$ vary over the covariate space. This characteristic is formulated in \S\ref{sec-model-assumptions}, along with other key model assumptions.

\subsection{Model Assumptions}\label{sec-model-assumptions}
We next detail our main model assumptions, which are conventional in the non-parametric contextual MAB literature (see, e.g., \citealt{perchet2013multi}). Our first model assumption addresses the smoothness of payoff functions. Before advancing it, we first formalize how payoff functions can change as a function of the covariates, using H\"{o}lder smoothness. For any multi-index $s=(s_1,\dots,s_d) \in \rN^d$ and any $x=(x_1,\dots,x_d)\in \rR^d$, we define $|s| = \sum_{i=1}^d s_i$, $s! = s_1! \dots s_d!$, $x^s=x_1^{s_1}\dots  x_d^{s_d}$, and $\|x\| = \l(x_1^2+\dots +x_d^2\r)^{\frac{1}{2}}$. Let $D^s$ denote the differential operator $D^s\coloneqq \frac{\partial^{s_1+\dots+s_d}}{\partial x_1^{s_1}\dots \partial x_d^{s_d}}$.
Let $\beta > 0$. Denote by $\lfloor \beta \rfloor$ the maximal integer that is strictly less than $\beta$, e.g., $\lfloor 1 \rfloor = 0$. For any $x\in \rR^d$ and any $\lfloor \beta \rfloor$ times continuously differentiable function $g(\cdot)$ on $\rR^d$, we denote by $g_x$ its Taylor expansion of degree $\lfloor \beta \rfloor$ at point $x$:
$
g_x(x^\prime)
\coloneqq
\sum_{|s| \le \lfloor \beta \rfloor} \frac{(x-x^\prime)^s}{s!} D^sg(x).
$

\begin{definition}[H\"{o}lder functions]\label{def-holder-class}
	The H\"{o}lder class of functions $\cH_\cX(\beta,L)$ for the parameters $\beta>0$ and $L>0$ and the set $\cX\subseteq \rR^d$ is defined as the set of functions $f:\cX \rightarrow \rR$ that are $\lfloor \beta \rfloor$ times continuously differentiable and, for any $x,x^\prime \in \cX$, satisfy the following inequality:
	\[
	\l| f(x^\prime) -f_x(x^\prime) \r|
	\le
	L\|x-x^\prime\|_\infty^{\beta}.
	\]
	Furthermore, let $\cH_\cX(\beta)\coloneqq \bigcup\limits_{0\le L<\infty} \cH_\cX(\beta,L)$. We drop the indication $\cX$ whenever $\cX=[0,1]^d$.
\end{definition}
\begin{assumption}[Smoothness]\label{assumption-Holder-smoothness}
	The payoff functions $f_k$, $k\in \cK$, belong to the H\"{o}lder class of functions $\cH(\beta, L)$ for some $L>0$ and $\beta \in [\lbeta, \ubeta]$ with $0<\lbeta\le 1$.
\end{assumption}

\noindent Our second assumption requires the distribution of covariates to be bounded from above and away from zero. Consequently, in every region of the covariate space, sufficiently many samples can be collected to estimate the payoff functions.
\begin{assumption}[Covariate distribution]\label{assumption-covar-dist}
	The distribution $\bm{\mathrm{P}}_X$ is equivalent to the Lebesgue measure on $[0,1]^d$; that is, there exist constants $0 < \underline{\rho} \le \bar{\rho} $ such that $p_X$, the density of~$\bm{\mathrm{P}}_X$, satisfies $\underline{\rho} \le p_X(x) \le \bar{\rho}$ for all $x \in [0,1]^d$.
\end{assumption}

\noindent Our third assumption, known as the \textit{margin} condition, captures the interplay between the payoff functions and the covariate distribution.

\begin{assumption}[Margin condition]\label{assumption-margin}
	There exist some $\alpha >0$
	and $C_0 > 0 $ such that
	\[
	\bm{\mathrm{P}}_X\l\{
	0 < \l| f_1(X) - f_2(X)\r| \le \delta
	\r\}
	\le
	C_0 \delta^\alpha, \qquad \forall \, 0< \delta \le 1.
	\]
\end{assumption}
\noindent The mass of covariates near the decision boundary is a key complexity driver: the larger the parameter~$\alpha$, the faster this mass shrinks near the boundary, and the easier the problem becomes. Together, the above three assumptions characterize the general class of problems that we consider.

\begin{definition}[Class of problems]\label{def:class-prob}
	For any $\beta\ge0$ and $\alpha\ge 0$, we denote by $\cP(\beta, \alpha, d) = \cP(\beta, L, \alpha, C_0, \underline{\rho}, \bar{\rho})$ the class of problems $\sfP=\l(\bm{\mathrm{P}}_X,\bm{\mathrm{P}}^{(1)}_{Y|X},\bm{\mathrm{P}}^{(2)}_{Y|X}\r)$ that satisfy Assumption~\ref{assumption-Holder-smoothness} for $\beta$ and $L>0$, Assumption~\ref{assumption-covar-dist} for some $\bar{\rho} \ge \underline{\rho}>0$, and Assumption~\ref{assumption-margin} for $\alpha$ and some $C_0>0$.
\end{definition}

\noindent It is worth noting the relation between the smoothness condition and the margin condition. The smoothness of payoff functions also determines how they might change near the decision boundary, which affects the mass of covariates in that region. That is, smooth payoff functions (large $\beta$) imply a larger mass of covariates near the decision boundary (small $\alpha$). This relationship is formalized in the following proposition, which is a simple extension of Proposition 3.1 in \cite{perchet2013multi}.

\begin{prop}[Margin condition and smoothness]\label{prop-margin-smoothness-relationship}
	Assume that Assumption \ref{assumption-Holder-smoothness} holds with parameters $(\beta,L)$, and that Assumption \ref{assumption-margin} holds with parameter $\alpha$. Then, the following statements hold:
	\begin{enumerate}
		\item If $\alpha \cdot \min\left\{1,\beta\right\} >1$, then a given action is either always or never optimal, and the oracle policy $\pi^\ast$ dictates selecting only one action all the time;
		\item If $\alpha \cdot \min\left\{1,\beta\right\} \le 1$, then there exist problem instances in $\cP(\beta,\alpha,d)$ with non-trivial oracle policies.
	\end{enumerate}
\end{prop}
\noindent Based on this proposition, when $\alpha > \frac{1}{\min\{1,\beta \}}$, the problem becomes equivalent to the classic stochastic MAB problem without covariates. Hence, we will assume that $0<\alpha \le \frac{1}{\min\{1,\beta \}}$  in the rest of the paper.

\subsection{The Cost of Smoothness Misspecification}\label{subsec:costExample}

We next demonstrate the loss that might be incurred by existing policies when the smoothness parameter is misspecified. When the problem instance belongs to $\cP(\beta, \alpha,d)$, the minimax regret rate is
\begin{equation}
\label{eq:opt_rate}
\inf_{\pi\in\Pi}\sup\limits_{\sfP \in \cP(\beta, \alpha,d)}\mathcal{R}^\pi(\sfP;T)  = \Theta\left(T^{\zeta(\beta, \alpha,d)}\right), \ \ \ \text{where}\ \
\zeta(\beta, \alpha,d) = 1 - \frac{\beta(1 + \alpha)}{2\beta + d}.
\end{equation}
This characterization was established by \citet{rigollet2010nonparametric} and \citet{perchet2013multi}
in the case of $\beta \leq 1$. With further assumptions on the regularity of the decision regions, \citet{hu2019smooth} establishes a similar characterization for $\beta>1$, up to an additional multiplicative term of $\l(\log T\r)^{\frac{2\beta+d}{2\beta}}$ that appears in their upper bound.

\cite{perchet2013multi} provides the \texttt{ABSE} policy and establishes that, when tuned with the correct smoothness parameter, this policy guarantees the minimax regret rate in (\ref{eq:opt_rate}) whenever $\beta \leq 1$. The design of the \texttt{ABSE} policy and the performance it achieves are nevertheless predicated on accurate knowledge of smoothness. The following example demonstrates that when the smoothness parameter is misspecified, the \texttt{ABSE} policy cannot guarantee rate-optimality anymore.

\begin{example}[Cost of smoothness misspecification for \texttt{ABSE}]\label{exp:cost-smoothness-misspecification}
	Fix a smoothness parameter $0<\beta\le1$ and a margin parameter $\alpha\le \frac{1}{\beta}$. Let $\texttt{ABSE}(\tilde{\beta})$ denote the \texttt{ABSE} policy tuned by a misspecified smoothness parameter $0<\tilde{\beta}\le1$. Then, there exist constants $\underline{C}^{\texttt{ABSE}}$ and $T_0$ independent of $T$ such that for all $T\ge T_0$, the following holds:
	\begin{enumerate}
		\item If $ \tilde{\beta} < \beta \le 1$, then
		$
		\sup\limits_{\sfP \in \cP(\beta, \alpha, d)}\mathcal{R}^{\texttt{ABSE}(\tilde{\beta})}(\sfP;T)
		\ge
		\underline{C}^{\texttt{ABSE}} T^{\zeta(\tilde{\beta}, \alpha,d)}
		$;
		\vspace{-.05cm}
		\item If $0< \beta < \tilde{\beta}$, then
		$
		\sup\limits_{\sfP \in \cP(\beta, \alpha, d)} \mathcal{R}^{\texttt{ABSE}(\tilde{\beta})}(\sfP;T)
		\ge
		\underline{C}^{\texttt{ABSE}} T.
		$
	\end{enumerate}
\end{example}

\noindent When smoothness is underestimated, the worst-case regret rate is equal to the minimax regret rate over the class of problems with ``rougher" payoff functions, and when smoothness is overestimated, the worst-case regret is linear in the horizon length. Similar results can be obtained for other policies proposed in the literature, including for the case of $\beta>1$; in \S\ref{section:impossibility} we provide a broad impossibility result that generalizes this example.

\subsection{The Smoothness-Adaptive Property}\label{subsec:smoothness-adaptive}
Next, we formalize a notion of adaptivity as our policy design goal. We say that a policy is smoothness adaptive if, for any problem instance, it achieves the optimal regret rate as a function of the H\"{o}lder exponent $\beta$ that characterizes that instance, up to a multiplicative term that is poly-logarithmic in the horizon length and a multiplicative constant that may depend on other problem parameters.
\begin{definition}[Smoothness-adaptive policies]\label{def:adaptivity}
	Fix two H\"{o}lder exponents $\lbeta < \ubeta$, and dimension $d$. Define the set of problem instances\vspace{-0.1cm}
	\[
	\cP^{\mathrm{all}}\coloneqq
	\cP^{\mathrm{all}}(\lbeta, \ubeta, d)
	=
	\bigcup_{\lbeta \le \beta \le \ubeta} \;\bigcup_{0<\alpha \le 1 \vee \frac{1}{\beta}} \cP(\beta, \alpha, d).
	\vspace{-.1cm}
	\]
Given a family of problem instances $\cP \subseteq \cP^{\mathrm{all}}$, a policy $\pi \in \Pi$ is said to be smoothness adaptive if, for any $\lbeta \le \beta \le \ubeta$ and $0<\alpha \le \frac{1}{\min\{1,\beta\}}$, there exist some function $\iota(\beta, \lbeta, \ubeta, \alpha)>0$ independent of $d$ and some $\bar{C}>0$ such that\vspace{-0.1cm}
	\[
	\sup\limits_{\sfP \in \cP \cap \cP(\beta, \alpha,d)}\mathcal{R}^\pi(\sfP;T)  \le \bar{C} \l( \log T \r)^{\iota(\beta, \lbeta, \ubeta, \alpha)} T^{\zeta(\beta, \alpha,d)},
	\vspace{-.1cm}
	\]
where the function $\zeta(\beta, \alpha,d)$ is as given in (\ref{eq:opt_rate}).
\end{definition}

\noindent Without access to prior knowledge of the smoothness parameters, smoothness-adaptive policies guarantee (up to a logarithmic factor) the minimax regret rate that characterizes the achievable performance when smoothness parameters are a priori known. We note that a related property has been suggested and analyzed in the covariate of adaptive confidence bands, with respect to the width of the confidence band rather than the accumulated regret; see, e.g., \cite{nickl2013confidence}.

\section{Impossibility of Costless Adaptation to Smoothness}\label{section:impossibility}

In this section, we discuss the possibility of adapting to the smoothness of payoff functions. The objective we consider is to design policies that are smoothness adaptive (see Definition~\ref{def:adaptivity}), that is, that achieve the rate of convergence detailed in \eqref{eq:opt_rate} without prior knowledge of the smoothness parameter $\beta$ that characterizes the payoff functions $\left\{f_k\right\}$. Our first key result, however, shows that this is impossible. 

In the following analysis we consider a setting with a pair of smoothness parameters $0< \beta < \gamma $, for which we know that $\sfP$ is $\beta$-smooth, i.e., $\sfP \in \cP(\beta, \alpha, d)$, but we do not know whether $\sfP$ is also $\gamma$-smooth, i.e., whether $\sfP \in \cP(\gamma, \alpha, d)$.
We show that there exist pairs $(\beta, \, \gamma)$ such that any admissible policy $\pi$ that (nearly) achieves the optimal regret rate over the smoother class $\cP(\gamma, \alpha, d)$ cannot simultaneously (nearly) achieve optimal rates over the rougher one. Therefore, without imposing additional requirements on the class of problems $\cP(\beta, \alpha, d)$, no admissible policy can be \emph{smoothness adaptive}.

\begin{theorem}[Impossibility of adapting to smoothness]\label{theorem-impossibility-adaptation-smoothness}
	Fix two H\"{o}lder exponents $0<\beta<\gamma $ and some margin parameter $0< \alpha \le \max\{1, \frac{1}{\gamma}\} $. Then, there exists some $T_0$ such that for any horizon length $T \ge T_0$ and any admissible policy $\pi\in\Pi$ that achieves rate-optimal regret $\cO\left(T^{\zeta(\gamma, \alpha,d)}\right)$ over $\cP(\gamma, \alpha, d)$, there exists a constant $ \underline{C}>0$ independent of $T$ such that the following holds:
	\begin{enumerate}
		\item 	(At most Lipschitz smooth) If $0<\beta<\gamma \le1$, then
\color{black}
\vspace{-.1cm}
		\[
		\sup\limits_{\sfP \in \cP(\beta, \alpha, d)}\mathcal{R}^\pi(\sfP;T)
		\ge
		\underline{C}T^{1-\frac{\beta+d}{(\alpha+1)(2\beta +d  -\alpha\beta)}} \l[T^{\zeta(\gamma, \alpha,d)}\r]^{-\frac{\alpha(\beta+d)}{(\alpha+1)(2\beta +d  -\alpha\beta)}};
		\vspace{-0.8cm}
		\]	\color{black}
		\item (At least Lipschitz smooth)
		If $\beta = 1 < \gamma$, then
		\vspace{-.1cm}
		\oldnewam{
		\[
		\sup\limits_{\sfP \in \cP(1, \alpha, d)}\mathcal{R}^\pi(\sfP;T)
		\ge
		\underline{C}T^{\frac{\alpha}{\alpha+1}} \l[T^{\zeta(\gamma, \alpha,d)}\r]^{-\frac{\alpha}{\alpha+1}}.
		\vspace{-0.9cm}
		\]
	}
	\end{enumerate}
\end{theorem}
\noindent Theorem~\ref{theorem-impossibility-adaptation-smoothness}
establishes a lower bound on the achievable performance over a class of problems as a function of the performance over another class of problems with smoother payoff functions. This lower bound depends on the smoothness parameters of the two considered classes of payoff functions as well as the margin parameter $\alpha$. As the examples below illustrate, Theorem~\ref{theorem-impossibility-adaptation-smoothness} implies that there exist pairs of smoothness parameters for which adaptivity is impossible without further assumptions.

\begin{example}[At most Lipschitz smooth]
Part \textit{1} of Theorem~\ref{theorem-impossibility-adaptation-smoothness} can be simplified as follows:
	\vspace{-.1cm}
\begin{equation*}\label{eq:impossibility-simplified}
\sup\limits_{\sfP \in \cP(\beta, \alpha, d)}\mathcal{R}^\pi(\sfP;T)
\ge
\underline{C}^\prime T^{1-\frac{(\beta+d)  (2\gamma + d-\alpha\gamma)}{ (2\gamma+d)(2\beta+d-\alpha\beta)}},\vspace{-.1cm}
\end{equation*}
for some constant $\underline{C}^\prime>0$. Thus, if $\gamma=\frac{15}{100}$, $\beta=\frac{\gamma}{2}$, \oldnewam{$\alpha=\frac{1}{\gamma}$,} and $d=1$, the optimal regret rate with knowledge of smoothness over $\cP(\beta,\alpha,1)$ is \oldnewam{$\cO(T^{0.5})$}, while Part \textit{1} of Theorem \ref{theorem-impossibility-adaptation-smoothness} establishes a lower bound of order \oldnewam{$\Omega\l( T^{0.6183} \r)$} if the policy $\pi$ achieves rate optimal performance over $\cP(\gamma,\alpha,1)$.
\end{example}
\begin{example}[At least Lipschitz smooth]
	Part \textit{2} of Theorem \ref{theorem-impossibility-adaptation-smoothness} can be simplified as follows:
		\vspace{-.1cm}
	\begin{equation*}\label{eq:impossibility-simplified-Lipschitz-smoother}
	\sup\limits_{\sfP \in \cP(1, \alpha, d)}\mathcal{R}^\pi(\sfP;T)
	\ge
	\underline{C}^\prime T^{1-\frac{2\gamma +  d - \gamma \alpha }{2\gamma+d}},
		\vspace{-.1cm}
	\end{equation*}
for some constant $\underline{C}^\prime>0$. Thus, if $\gamma>1$, $\alpha=1$, and $d=1$, the optimal regret rate with knowledge of smoothness over $\cP(1,1,1)$ is $\cO(T^{\frac{1}{3}})$, while Part \textit{2} of Theorem \ref{theorem-impossibility-adaptation-smoothness} establishes a lower bound of order  if the policy $\pi$ achieves rate-optimal performance over $\cP(\gamma,1,1)$. Since $\frac{\gamma}{2\gamma + 1}> \frac{1}{3}$ for any $\gamma>1$, no policy can be simultaneously rate-optimal over both  $\cP(1,1,1)$ and $\cP(\gamma,1,1)$, for $\gamma>1$.
	\end{example}

\color{black}
\noindent Figures~\ref{fig-impossibility-proof-idea} and~\ref{fig-impossibility-proof-idea-at-least-Lipschitz} respectively depict types of at-least-Lipschitz-smooth and at-most-Lipschitz-smooth payoff functions under which the loss specified in Examples~2 and 3 is incurred. The instances depicted in these figures will be further discussed in \S\ref{subsec:prooflowerbound}.

\color{black}
We note that Theorem \ref{theorem-impossibility-adaptation-smoothness}
rules out adaptivity for some, but not necessarily all, pairs of smoothness parameters
$0 < \beta < \gamma$. Understanding whether there exist some pairs for which adaptivity
is possible---and, more broadly, providing a comprehensive characterization of adaptive
rates across mixtures of H\"{o}lder classes---would be of considerable interest.
In the present paper, however, we leave these questions for future work. In the next sections we turn our focus to payoff functions that are self-similar and show that---in this case---there exist policies that are smoothness adaptive with considerable generality.

\begin{remark} \label{remark-full-ffeedback}
		The general impossibility of adapting to unknown smoothness that is established in this section is a consequence of the partial feedback structure detailed in \S2 (often referred to as the bandit-feedback setting), where in each period an obsevation is collected only on the action that is selected in that period. By contrast, in Appendix~\ref{appendix-full-feedback} we show that if in each period the agent has access to reward observations from all the actions (often referred to as the full-feedback setting), then it is possible to adapt to payoff smoothness in the sense of Definition~\ref{subsec:smoothness-adaptive}.
	\end{remark}

\color{black}
\subsection{Key Ideas in the Proof of Theorem~\ref{theorem-impossibility-adaptation-smoothness}}\label{subsec:prooflowerbound}
\color{black}
The proof of Theorem \ref{theorem-impossibility-adaptation-smoothness} adapts to our framework ideas of identifying a worst-case nature ``strategy," while devising a novel construction of instances to reduce the problem to one of hypothesis testing. The proof of the theorem is deferred to the appendix, together with the proofs of all subsequent results. We next illustrate the key ideas of the proof.

\medskip
\noindent
\textbf{At-Most-Lipschitz-Smooth Payoffs.} We next detail the key ideas of the proof of Part 1 of the theorem for the case of $\alpha=\frac{1}{\gamma}$ and $d=1$; the construction of the worst-case instance in this setting is depicted in Figure~\ref{fig-impossibility-proof-idea}.

\begin{figure}[h]
	\centering
	\begin{subfigure}[t]{0.49\textwidth}
		\raisebox{-\height}{\includegraphics[width=\textwidth]{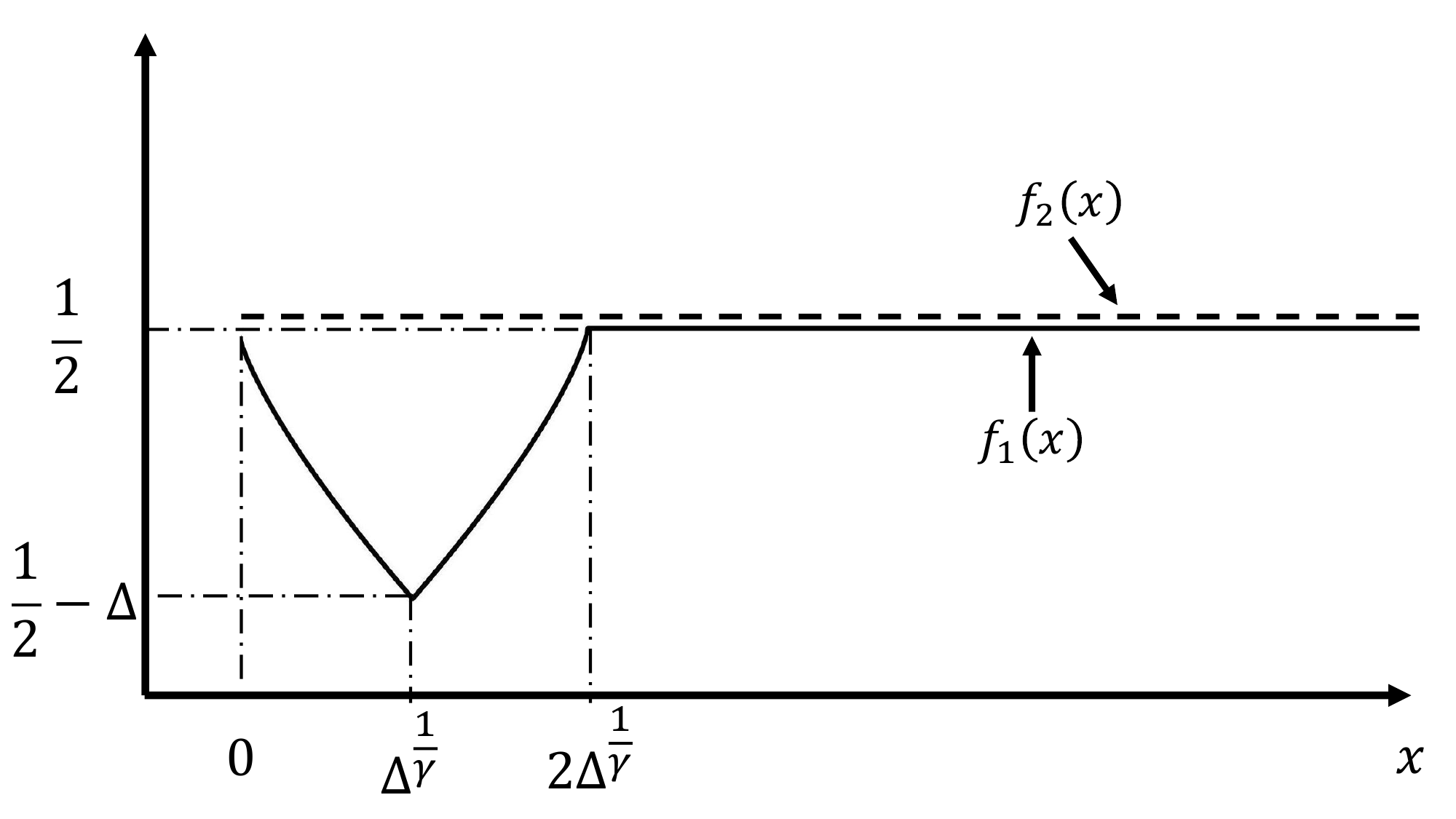}}
	\end{subfigure}
	\hfill
	\begin{subfigure}[t]{0.49\textwidth}
		\raisebox{-\height}{\includegraphics[width=\textwidth]{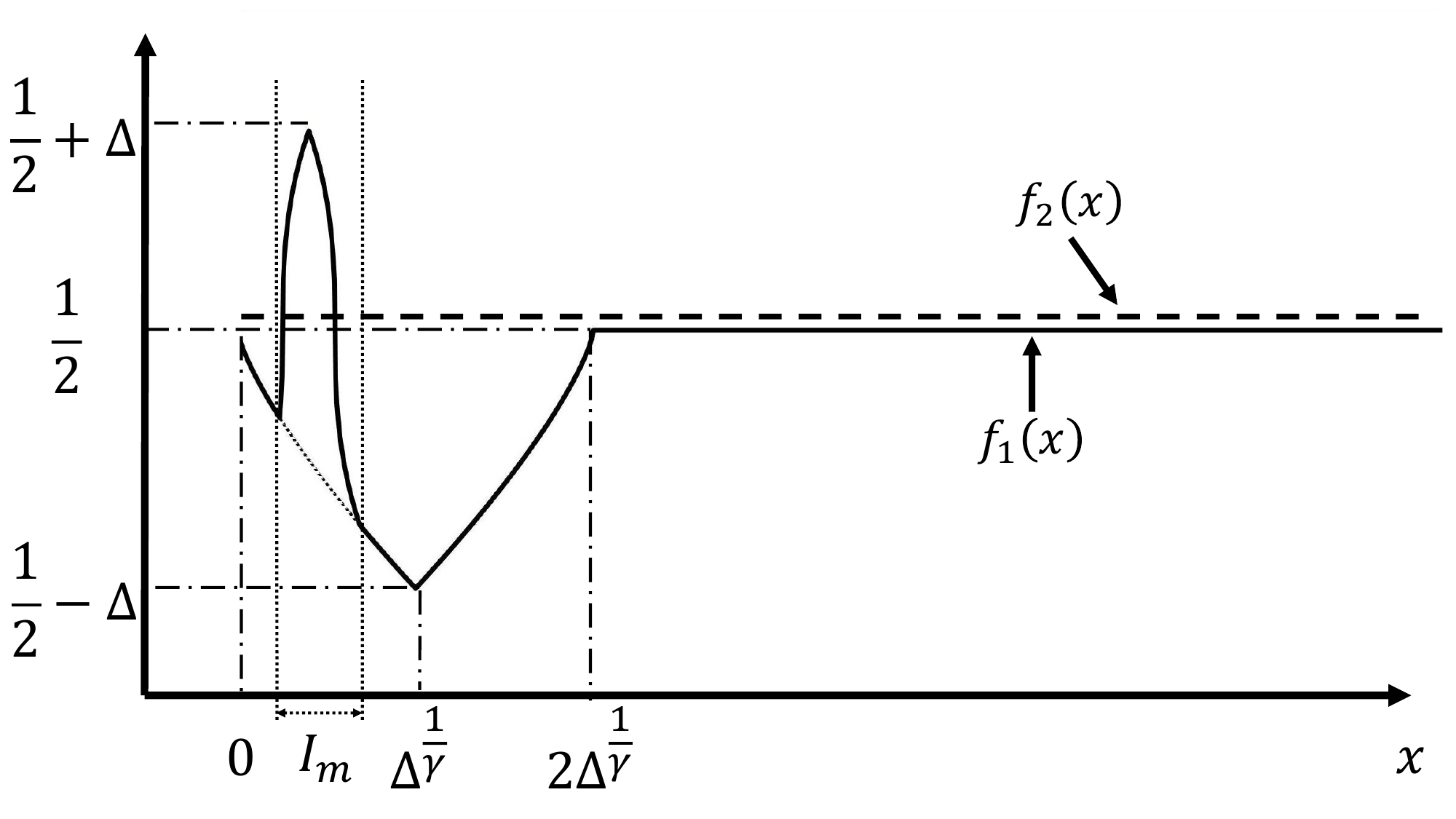}}
	\end{subfigure}
	\caption{\footnotesize Description of the worst-case instance constructed in the proof of Theorem \ref{theorem-impossibility-adaptation-smoothness} \oldnewam{for the at-most-Lipschitz-smooth case.} \textit{Left:} The payoffs of the nominal problem in $\cP(\gamma, \alpha,1)$; \textit{Right:} The payoffs of one of the alternative problems in~$\cP(\beta, \alpha,1)$.}\vspace{-0.2cm}
	\label{fig-impossibility-proof-idea}
\end{figure}

Fix a parameter $\Delta \le \frac{1}{4}$. First, consider a nominal problem instance in $\cP(\gamma, \alpha, 1)$ such that the first action's payoff function is $\frac{1}{2}$ for every covariate except for the interval $[0,2\Delta^{\frac{1}{\gamma}}]$, where it has a ``downward bump" and reaches its minimum, $\frac{1}{2}-\Delta$, and the second action's payoff function is $\frac{1}{2}$ everywhere. Furthermore, for each $1\le m\le M\coloneqq \lfloor \Delta^{\frac{1}{\gamma}-\frac{1}{\beta}}\rfloor$, consider a problem instance in $\cP(\beta, \alpha, 1)$ such that the payoff functions are equal to the aforementioned payoff functions everywhere except for the interval $I_m\coloneqq [2(m-1)\Delta^{\frac{1}{\beta}}, 2m\Delta^{\frac{1}{\beta}}]$, where the first action's payoff function has an ``upward bump" and reaches its maximum, $\frac{1}{2}+\Delta$, as depicted in Figure \ref{fig-impossibility-proof-idea}. That is, for the problem $m$, the first action is optimal over some segment of $I_m$ \oldnewam{with a gap of at least $\frac{\Delta}{2}$.} To meet its performance guarantees over $\cP(\gamma,\alpha,1)$, in at least one of the intervals $I_m$, the number of times $\pi$ selects action $1$ must be ``small." We denote one such interval by $I_{m*}$. Using this observation along with the fact that one can differentiate between the nominal problem described above and problem instance $m^\ast$ based only on the outcomes of action 1 in the interval $I_{m*}$, one may show that any admissible policy cannot distinguish between these two problem instances with strictly positive probability. This causes such a policy not to select action $1$ almost half of the times in which the realized covariates belong to the interval $I_m^\ast$. Interval $I_m^\ast$ contains a segment over which the first action is optimal for the problem $m^\ast$, which guarantees the regret bound stated in the theorem for a carefully selected value for the parameter $\Delta$.

\color{black}
\medskip
\noindent
\textbf{At-Least-Lipschitz-Smooth Payoffs.} The proof of Part 2 of the theorem follows a similar line of reasoning. We next detail the key ideas for the case of $d=1$; the construction of the worst-case instance in this setting is depicted in Figure~\ref{fig-impossibility-proof-idea-at-least-Lipschitz}.

\begin{figure}[h]
	\centering
	\begin{subfigure}[t]{0.49\textwidth}
		\raisebox{-\height}{\includegraphics[width=\textwidth]{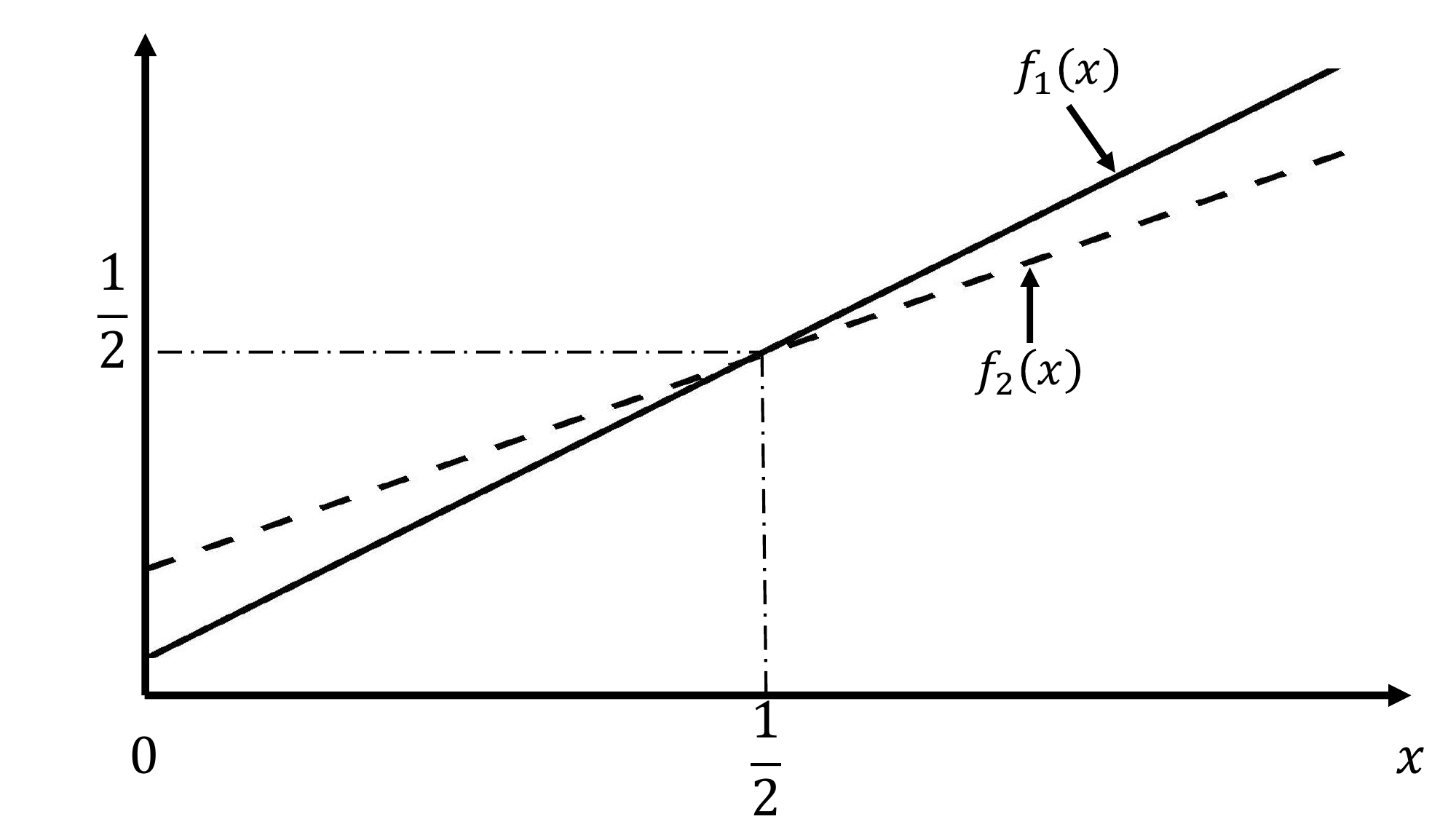}}
	\end{subfigure}
	\hfill
	\begin{subfigure}[t]{0.49\textwidth}
		\raisebox{-\height}{\includegraphics[width=\textwidth]{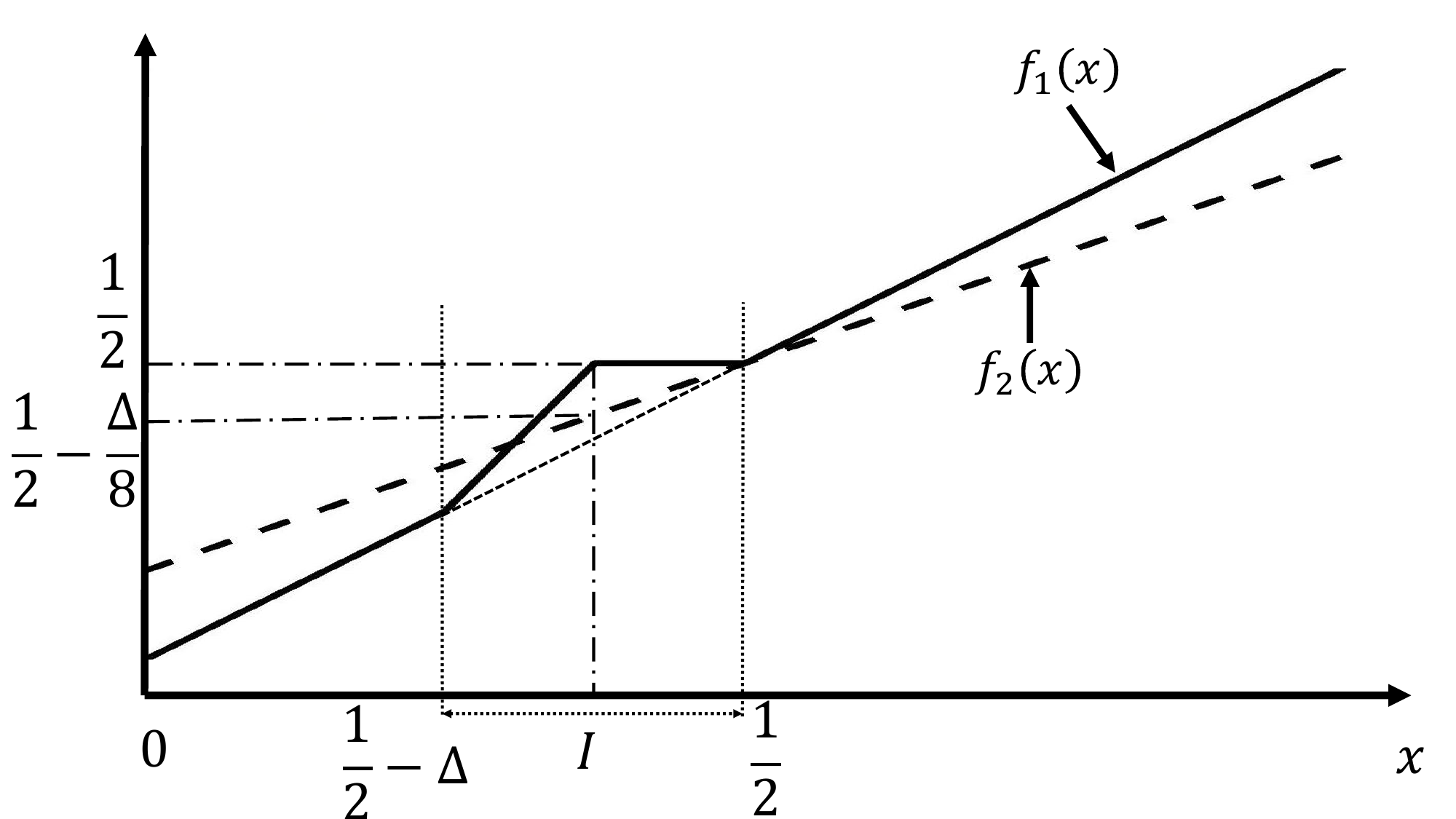}}
	\end{subfigure}
	\caption{\footnotesize \oldnewam{Description of the worst-case instance constructed in the proof of Theorem \ref{theorem-impossibility-adaptation-smoothness} for the at-least-Lipschitz-smooth case. \textit{Left:} The payoffs of the nominal problem in $\cP(\gamma, \alpha,1)$; \textit{Right:} The payoffs of the alternative problem in $\cP(\beta, \alpha,1)$.}}\vspace{-0.0cm}
	\label{fig-impossibility-proof-idea-at-least-Lipschitz}
\end{figure}

First, consider a nominal problem instance in $\cP(\gamma, \alpha, 1)$ consisting of linear payoff functions with a crossing point at $\l(\frac{1}{2}, \frac{1}{2}\r)$ such that the first action is strictly suboptimal for $x<\frac{1}{2}$. Furthermore, consider a problem instance in $\cP(\beta, \alpha, 1)$ consisting of payoff functions that are identical to the aforementioned payoffs everywhere except for the interval $I \coloneqq [\frac{1}{2} - \Delta, \frac{1}{2}]$, where the first action's payoff has a triangular ``upward bump" and reaches $\frac{1}{2}$. 
That is, the first action is optimal over some segment of $I$ with a gap of order~$\Delta$. The rest of the arguments in the proof are similar to those of Part 1.
\color{black}

\vspace{-0.0cm}
\section{Self-Similar Payoffs}\label{sec:SScondition}\vspace{-0.0cm}

In this section we first adapt in \S\ref{subsec:SScondition} a \textit{self-similarity} condition that appears in the literature on non-parametric confidence bands (e.g., \citealt{picard2000adaptive} and \citealt{gine2010confidence}), and then show in \S\ref{subsec:ssminimax} that the assumption that payoff functions are self-similar does not reduce the minimax regret complexity of the problem at hand. Later on, in \S\ref{sec:global}, we will show that self-similarity makes it possible to guarantee rate optimality without prior knowledge of the payoff smoothness, and devise a general policy for achieving smoothness-adaptive performance.

\vspace{-0.0cm}
\subsection{The Self-Similarity Condition}\label{subsec:SScondition}\vspace{-0.0cm}

Before introducing the self-similarity condition we first advance some relevant notation. For a given function $f(\cdot)$ and non-negative integers $l$ and $p$, define $\bm{\Gamma}_{\textcolor{black}{h}}^pf(\cdot; U)$ to be the $L_2(\bm{\mathrm{P}}_X)$-projection of the function $f(\cdot)$ to the class of polynomial functions of degree at most $p$ over the {hypercube} $U$. Formally, for any $x\in U$ we define\vspace{-0.1cm}
\begin{equation}\label{eq-def-polynomial-projection}
\bm{\Gamma}_{\textcolor{black}{h}}^pf(x;U)
\coloneqq
g(x), \qquad \text{s.t.} \qquad g = \arg\min\limits_{q\in \mathrm{Poly}(p)} \int_{U} \l| f(u)-q(u) \r|^2 K\l(\frac{x-u}{h}\r) p_X(u \mid U) du,\vspace{-0.1cm}
\end{equation}
where we use kernel $K(\cdot)=\Indlr{\|\cdot\|_\infty\le 1}$ and bandwidth $h$, and $\mathrm{Poly}(p)$ is the class of polynomials of degree at most $p$. Next, we formalize the notion of self-similarity, using the projection $\bm{\Gamma}_{\textcolor{black}{h}}^p f$. For integers $l\ge0$ and \oldnewam{$q$}, let $\cB_l^{\oldnewam{q}} \coloneqq  \l\{ \sfB_m, \, m=1,\dots,\oldnewam{q}^{ld} \r\}$ be a re-indexed collection of the hypercubes:\vspace{-0.1cm}
\[
\sfB_m = \sfB_{\sfm} \coloneqq \l\{ x \in [0,1]^d:\, \frac{\sfm_i-1}{\oldnewam{q}^l} \le x_i \le \frac{\sfm_i}{\oldnewam{q}^l},\; i\in\{1,\dots,d\} \r\},\vspace{-0.1cm}
\]
for $\sfm=(\sfm_1,\dots,\sfm_d)$ with $\sfm_i \in \{1,\dots,\oldnewam{q}^l\}$.

\begin{definition}[Self-similar payoffs]\label{def:global-ss}\color{black} For some $\beta\in [\underline{\beta}, \bar{\beta}]$, and finite constants $l_0 \ge 0$ and $b>0$, define the class of self-similar sets of payoffs, $ \cF^{ss}(\beta,b,l_0)$, to be the collection of the sets of payoffs $\{f_k\}_{k\in\cK}$ such that $f_k \in \cH(\beta)$, $k\in\cK,$ and for which for any integers $l \ge l_0$, $q>1$, and $\lfloor \beta \rfloor \le p \le \lfloor \ubeta \rfloor$, one has
	\[
	\max_{\sfB \in \cB_l^{\oldnewam{q}}}\max_{k\in\cK}\sup_{x \in\sfB}\l|
	\bm{\Gamma}_{\textcolor{black}{q^{-l}}}^pf_k(x;\sfB) - f_k(x)
	\r|
	\ge
	b \textcolor{black}{q}^{-l\beta}.
	\]
\end{definition}

\noindent
The self-similarity condition complements the H\"{o}lder smoothness condition (Assumption~\ref{assumption-Holder-smoothness}) in the following sense. On the one hand, H\"{o}lder smoothness implies an upper bound on the estimation bias of payoff functions at every point (\emph{estimation bias} refers to the absolute difference between the value of a function and the expected value of its estimator, using, e.g., local polynomial regression). On the other hand, the self-similarity condition effectively implies a global lower bound on the estimation bias. 
More precisely, self-similarity implies that for a set of $\beta$-smooth payoffs, the estimation bias is guaranteed to be at least of order $h^{\beta}$ for bandwidth $h>0$. This provides an opportunity to estimate the smoothness of payoff functions by ``comparing" estimation variance and bias, which is the essence of Lepski's method \citep{lepski1997optimal}. In the conventional version of Lepski's method, one adjusts the estimators to identify the bandwidth that balances estimation bias and stochastic error. The relationship between this bandwidth and the true smoothness enables one to estimate the smoothness. In the absence of direct access to the estimation bias, Lepski's method provides a general approach for constructing a proxy for it through the absolute difference of estimators with different bandwidths. In \S\ref{subsec:smoothness-estimation}, we advance a variant of Lepski's method that is tailored to the dynamic nature of the problem at hand, where we keep the bandwidth fixed and instead adjust the estimators through the number of samples they deploy.

\color{black}
\Copy{ss-intuitive-exp}
One may view the self-similarity condition as ensuring that the regularity of payoff functions is similar on small and large scales. To achieve smoothness adaptivity, one needs to know the regularity of the payoffs on small scales. 
If payoffs are self-similar, 
one may infer such small-scale regularity from the structure of payoffs on larger scales. For further discussion on the self-similarity condition see \cite{bull2012honest}. 
The next example illustrates a set of self-similar payoff functions.

\begin{example}[Self-similar payoffs]\label{exp:self-similar} Fix some $\beta < \ubeta < 1 $ and assume that covariates are one-dimensional. Assume that $f_k\in \cH(\beta)$ for each $k\in\cK$. If $f_1(x) = x^\beta$ for any $x \in [0,\Delta^{\frac{1}{\beta}}]$ and some $\Delta>0$, then the set of payoffs $\{f_k\}_{k=1}^K$ is self-similar with $b \coloneqq \frac{1}{\beta+1}$ and $l_0 \coloneqq -\frac{1}{\beta}\log_2 \Delta$. That is, for any $q>1$, $l\ge l_0 $,  and $p=0$, one has\vspace{-.1cm}
	\begin{equation*}
		\max_{\sfB \in \cB_l^{\textcolor{black}{q}}}\max_{k\in\cK}\sup_{x \in\sfB}\l|
		\bm{\Gamma}_{\textcolor{black}{q^{-l}}}^pf_k(x;\sfB) - f_k(x)
		\r|
		\ge
		\l|
		\frac{1}{\textcolor{black}{q}^{-l}} \int_0^{\textcolor{black}{q}^{-l}}f_1(x) dx - f_1(0)
		\r|
		=
		\frac{1}{\beta + 1} \textcolor{black}{q}^{-l \beta}
		=
		b \textcolor{black}{q}^{-l\beta}.\vspace{-.3cm}
	\end{equation*}
\end{example}
\begin{figure}[h]
	\centering
		\raisebox{-\height}{\includegraphics[width=.5\textwidth]{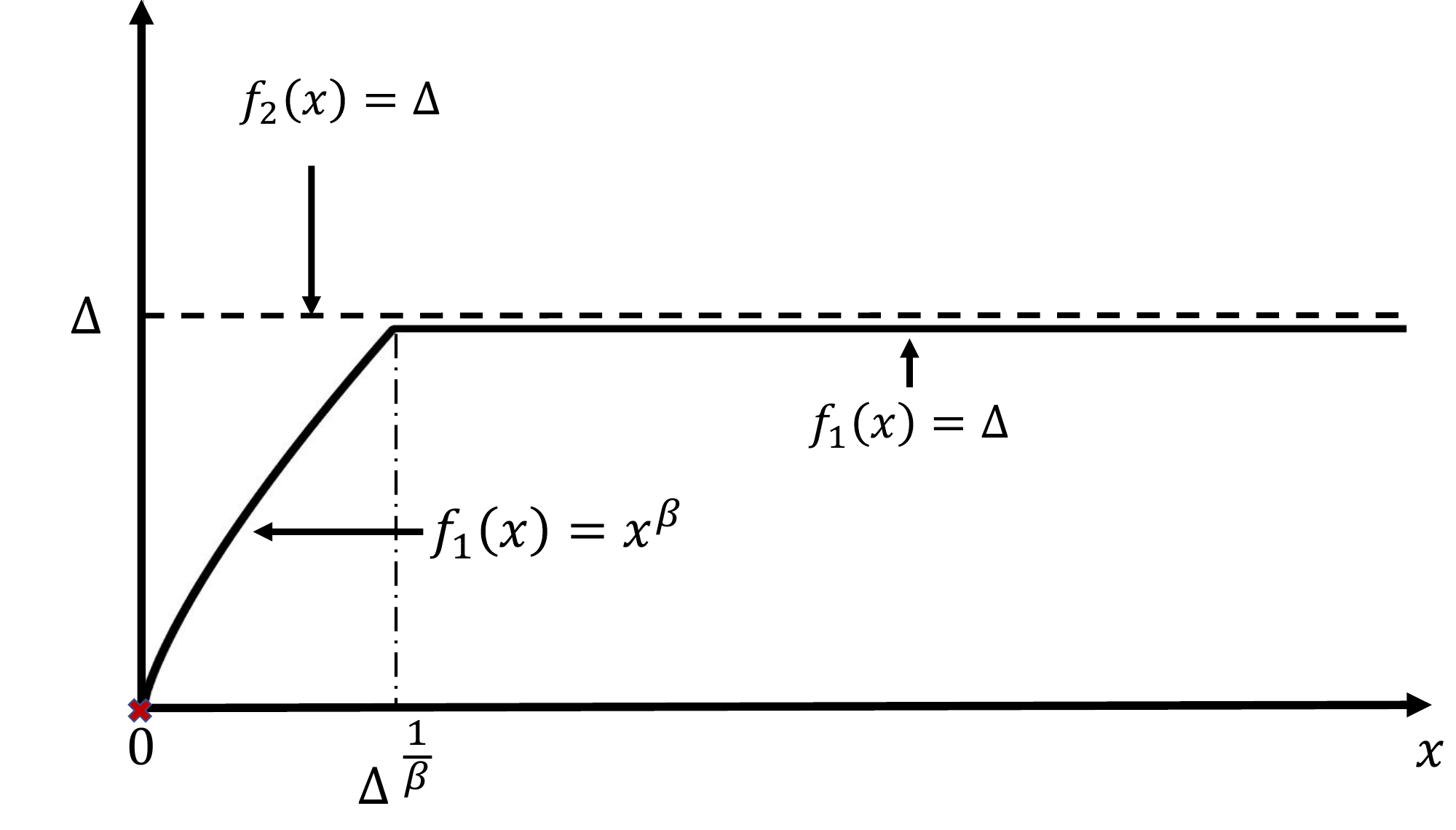}}
	\caption{\footnotesize \oldnewam{A set of payoff functions that can either satisfy or not satisfy the conditions in Example~\ref{exp:self-similar}, 
depending on the value of the parameter $\Delta$ relative to the self-similarity constants $l_0$ and $b$. For a fixed $l_0$, if either the parameter $b$ is sufficiently large, that is, $b >  \frac{1}{\beta + 1}$, or the parameter $\Delta$ is sufficiently small, that is, $ -\frac{1}{\beta}\log_2 \Delta \ge l_0$, then the set of payoffs is not self-similar. 
}}\vspace{-0.0cm}
	\label{fig:ss}
\end{figure}
\noindent
Figure \ref{fig:ss} depicts a set of payoff functions that can either satisfy or not satisfy the conditions in Example~\ref{exp:self-similar}, that is, self-similarity with $b \coloneqq \frac{1}{\beta+1}$ and $l_0 \coloneqq -\frac{1}{\beta}\log_2 \Delta$, depending on the value of the parameter $\Delta$ relative to the self-similarity constants $l_0$ and $b$. Note that for a fixed $l_0$, if the parameter $b$ is sufficiently large, that is, $b >  \frac{1}{\beta + 1}$, or the parameter $\Delta$ is sufficiently small, that is, $ -\frac{1}{\beta}\log_2 \Delta \ge l_0$, then the set of payoffs depicted in Figure~\ref{fig:ss} is not self-similar anymore, as one may not be able to detect the roughness of $f_1$ at $x=0$ based on its neighborhood in a coarse (``large" scale) partitioning $\cB^q_l$ with $l = l_0$. Nevertheless, we note that even when the payoffs depicted in Figure~\ref{fig:ss} are not self-similar anymore, they are still $\beta$-smooth for any $\Delta>0$.

\color{black}

We next provide the self-similarity assumption, followed by a formulation of the class of problem instances with self-similar payoff functions, which is a subset of the more general class from Definition~\ref{def:class-prob}.

\begin{assumption}[Self-similar payoffs]\label{assumption-global-self-similarity}
	\textcolor{black}{$\{f_k\}_{k\in\cK} \in \cF^{ss}(\beta, b, l_0)$ for some finite constants $l_0\ge 0$ and $b>0$.}
\end{assumption}

\begin{definition}[Class of problems with self-similar payoffs]
	\textcolor{black}{For any  $\beta \ge 0$, $\alpha\ge 0$,  and finite constants $l_0\ge 0$ and $b>0$, we define by $\cP^{\mathrm{ss}}(\beta, \alpha, d, b, l_0) \coloneqq \l\{ \sfP \in \cP(\beta, \alpha, d): \{f_k\}_{k\in\cK} \in \cF^{ss}(\beta, b, l_0)\r\}$ the class of problems with self-similar sets of payoffs.}
\end{definition}

\subsection{Minimax Complexity with Self-Similar Payoffs}\label{subsec:ssminimax}

While Example~\ref{exp:self-similar} illustrates a set of particularly simple payoff functions, we note that the class of self-similar payoffs is quite general and includes many different payoff structures. In fact, we next show that the minimax complexity of the dynamic optimization problem at hand is not reduced when the self-similarity condition is introduced.

We establish this result by constructing regret lower bounds that are of the same order as in~\eqref{eq:opt_rate}. To do so, we design worst-case instances consisting of payoff functions that satisfy Assumption~\ref{assumption-global-self-similarity}. More precisely, we show that worst-case instances developed in \cite{rigollet2010nonparametric} for the case of $\beta\le 1$ and in \cite{hu2019smooth} for the case of $\beta \ge 1$ can essentially be constructed using self-similar payoffs.

For consistency with the setting in \cite{hu2019smooth} that allows a more general structure for the support of the covariate distribution in the case of $\beta \ge 1$, we denote by \oldnewam{$\tilde{\cP}^{\mathrm{ss}}( \beta, \alpha, d, b, l_0)$} the class of problems with self-similar payoffs where the covariate density $p_X$ has a compact support $\cX\subseteq [0,1]^d$; 
for further details see Appendix \ref{app:reg-lower-ss}.

\begin{theorem}[Self-similarity assumption does not reduce minimax complexity]\label{theo-reg-lower-ss}
	Fix some non-integer H\"{o}lder exponent \oldnewam{$\beta > 0$}, some margin parameter $\alpha>0$ such that $\alpha  \le \max\{1,\frac{1}{\beta}\}$ and $\alpha \beta \le d$\oldnewam{, and some finite constants $l_0\ge0$ and $b>0$}. Then, there exist $T_0, \underline{C}>0$ such that for any horizon length $T \ge T_0$ and any admissible policy $\pi\in\Pi$, the following lower bounds on the regret hold:
	\begin{enumerate}
		\item (At most Lipschitz smooth) If $\beta\le1$, then
	$
		\sup_{\sfP \in \textcolor{black}{\cP^{\mathrm{ss}}( \beta, \alpha, d, b, l_0)}}\mathcal{R}^\pi(\sfP;T) \ge
		\underline{C} T^{1-\frac{\beta(\alpha+1)}{2\beta + d}}.
	$
	\item (At least Lipschitz smooth) If $\beta\ge1$, then\vspace{-0.1cm}
	$
	\sup_{\sfP \in \textcolor{black}{\tilde{\cP}^{\mathrm{ss}}( \beta, \alpha, d, b, l_0)}}\mathcal{R}^\pi(\sfP;T) \ge
	\underline{C} T^{1-\frac{\beta(\alpha+1)}{2\beta + d}}.
	$

	\end{enumerate}	
\end{theorem}

\medskip
\noindent 
Theorem \ref{theo-reg-lower-ss} establishes that requiring payoff functions to be self-similar does not reduce the minimax (regret) complexity, and therefore implies that the minimax complexity of the problem under self-similar payoffs (Assumption \ref{assumption-global-self-similarity}) is still as stated in~\eqref{eq:opt_rate}. Nevertheless, in the next section we establish that under self-similar payoffs one may design policies that are smoothness adaptive, and essentially guarantee the minimax regret rate without prior information on the smoothness of payoff functions.

\section{Adaptivity to Smoothness}\label{sec:global}
In this section, we first detail in \S\ref{sec:global-main-results} the main result of the section, establishing that under self-similar payoffs one may guarantee smoothness-adaptive performance. This result is based on providing a \textit{Smoothness-Adaptive Contextual Bandits} (\texttt{SACB}) policy, and establishing that this policy is smoothness adaptive. In \S\ref{sec:policy} we then provide a detailed description of the \texttt{SACB} policy and discuss its key components.

\subsection{Smoothness-Adaptive Performance with Self-Similar Payoffs}\label{sec:global-main-results}\vspace{-0.1cm}

We next detail the main results of this section. We show that the \textit{Smoothness-Adaptive Contextual Bandits} (\texttt{SACB}) policy (that is detailed in \S\ref{sec:policy}; see Algorithm \ref{algorithm-GSE}) is smoothness adaptive under Assumption~\ref{assumption-global-self-similarity}.

The key idea of the \texttt{SACB} policy lies in observing that the local polynomial regression estimation of any function $f(\cdot)$ cannot largely deviate from the projection $\bm{\Gamma}_{\textcolor{black}{h}}^p f$ with high probability. That is, Assumption~\ref{assumption-global-self-similarity} is key in establishing that for a set of H\"{o}lder-smooth payoff functions, the estimation bias is not only bounded from above, but also cannot shrink fast (see further discussion and analysis in Appendix~\ref{appendix-local-poly}). This suggests an opportunity to estimate the smoothness of the payoff functions by appropriately examining the estimation bias against its variance over the unit hypercube.

The \texttt{SACB} policy adaptively integrates a smoothness estimation sub-routine with some collection of non-adaptive policies $\{\pi_0(\beta_0)\}_{\beta_0 \in [\lbeta , \ubeta]}$ that are rate optimal under accurate tuning of the smoothness parameter. The estimation sub-routine of \texttt{SACB} consists of three steps: $(i)$ collecting samples over the covariate space; $(ii)$ estimating the payoff functions; and~$(iii)$~conducting a hypothesis test. After the estimation sub-routine is terminated, the produced estimate $\hat{\beta}_{\texttt{SACB}}$ is used to choose the corresponding rate-optimal non-adaptive policy $\pi_0^{(\hat{\beta}_{\texttt{SACB}})}$. The following result characterizes the quality of the smoothness estimation of the \texttt{SACB} policy.

\begin{theorem}[Smoothness estimation with self-similar payoffs]\label{theorem-GSE-smoothness-accuracy}
	Suppose that Assumption \ref{assumption-Holder-smoothness} holds for some $L>0$ and $\beta \in [\lbeta, \ubeta]$, and Assumption~\ref{assumption-global-self-similarity} holds for some $b>0$ and $l_0 \ge 0$. Then,
there exists $T_0>0$ independent of $T$ such that for any horizon length $T \ge T_0$, the \texttt{SACB} policy detailed in Algorithm \ref{algorithm-GSE}, with tuning parameter $\gamma>0$
\color{black}
and under-smoothing coefficient $\upsilon=\frac{d}{\lbeta}+4$,
\color{black}
computes an estimate of $\beta$, denoted by $\hat{\beta}_{\texttt{SACB}}$, by time step $t=\l\lceil\frac{4}{\underline{\rho}} \l( \log T \r)^{\frac{2d}{\underline{\beta}}+4}T^{\frac{(\underline{\beta}+d-1)}{(2\overline{ \beta}+d)}}\r\rceil$ with probability at least $1-2 T^{\constvar[SACB-time-1]}
\exp\l( -\constvar[SACB-time-2] T^{\constvar[SACB-time-3]} \r)$, such that the following bound holds:\vspace{0.00cm}
	\[
	\rPlr{\hat{\beta}_{\texttt{SACB}} \in \left[\beta- \frac{3(2\overline{ \beta} + d)^2\log_{q}\log T}{(\underline{\beta}+d-1)\log_{q} T}   ,\beta\right]}
	\ge
	1 - \constvar[GSE-smoothness-accuracy1] \l( \log T \r)^{\frac{d}{\underline{\beta}}} T^{ -\gamma ^2 \constvar[GSE-smoothness-accuracy2] + \constvar[GSE-smoothness-accuracy3] },\vspace{-0.05cm}
	\]
	where the constants $\constref{SACB-time-1}, \constref{SACB-time-2}, \constref{SACB-time-3}, \constref{GSE-smoothness-accuracy1}, \constref{GSE-smoothness-accuracy2},$ and $ \constref{GSE-smoothness-accuracy3}$ depend only on $\lbeta, \ubeta , b, L,  \underline{\rho}, \bar{\rho}$, and $d$.
\end{theorem}

\smallskip
\noindent The proof of Theorem~\ref{theorem-GSE-smoothness-accuracy} follows from Propositions \ref{proposition-GSE-lower-bound-round} and \ref{proposition-GSE-upper-bound-round}, which will be advanced in \S\ref{subsec:smoothness-estimation} for analyzing the performance of the smoothness estimation sub-routine in the \texttt{SACB} policy.
Theorem~\ref{theorem-GSE-smoothness-accuracy} implies that the error of the smoothness estimate grows linearly with the covariate dimension and decays as a function of the time horizon at a rate of $\frac{\log \log T}{\log T}$. This characterization of the smoothness estimation is leveraged in the next theorem to establish that the \texttt{SACB}, when coupled with appropriate off-the-shelf non-adaptive policies, guarantees the optimal regret rate up to poly-logarithmic terms, and smoothness-adaptive performance as stated in Definition \ref{def:adaptivity}.\smallskip

\begin{theorem}[Smoothness-adaptive policy with self-similar payoffs]\label{theorem-GSE-ABSE-regret-bound}
	Let $\pi$ be the \texttt{SACB} policy detailed in Algorithm~\ref{algorithm-GSE}, and let $\{\pi_0(\beta_0)\}_{\beta_0 \in [\lbeta , \ubeta]}$ be a set of non-adaptive policies such that if initialized with the true smoothness parameter, for any $\lbeta \le \beta_0 \le \ubeta$, $\alpha \le \frac{1}{\min\{1,\beta_0 \}}$, and $T\ge 1$, it satisfies the following upper bound on the regret:\vspace{-0.05cm}
	\[
	\sup_{\sfP \in \cP(\beta_0, \alpha, d)}\mathcal{R}^{\pi_0(\beta_0)}(\sfP;T)
	\le
	\bar{C}_0 \l( \log T \r)^{\iota_0(\beta_0, \alpha, d)} T^{\zeta(\beta_0, \alpha,d)},\vspace{0.00cm}
	\]
	for some $\iota_0(\beta_0, \alpha, d)$ and a constant $\bar{C}_0>0$ that is independent of $T$, where the function $\zeta(\beta_0, \alpha,d)$ is as given in (\ref{eq:opt_rate}). \color{black}
	Then, there exist $\gamma_0$ and $\bar{C}>0$, such that for any problem instance $\sfP\in \cP^{\mathrm{ss}}(\beta, \alpha, d, b, l_0)$ with $\lbeta \le \beta \le \ubeta$, $\alpha \le \frac{1}{\min\{1,\beta \}}$ and finite constants $l_0\ge 0$ and $b>0$, any tuning parameter $\gamma \ge \gamma_0$, any horizon length $T$,
and any under-smoothing coefficient $\upsilon=\frac{d}{\lbeta}+4$, one has
\color{black}
\vspace{0.00cm}
	\[
\mathcal{R}^\pi(\sfP;T)
	\le
	\bar{C}
	\l(\log T\r)^{\frac{3d(\alpha+1)(2\overline{ \beta} + d)^2}{(2\beta+d)(\beta+d)(\underline{\beta}+d-1)} + \iota_0\l(\beta- \frac{3(2\overline{ \beta} + d)^2\log_{\textcolor{black}{q}}\log T}{(\underline{\beta}+d-1)\log_{\textcolor{black}{q}} T} , \alpha, d\r)}  T^{\zeta(\beta, \alpha,d)}.\vspace{0.05cm}
\]
\end{theorem}

\noindent The proof of Theorem \ref{theorem-GSE-ABSE-regret-bound} follows from observing that, with high probability, the number of time periods that are required to generate the smoothness estimate (and hence the regret that is incurred throughout the smoothness estimation process) is ``small" relative to the optimal regret rate, and from plugging the lower confidence bound established in Theorem~\ref{theorem-GSE-smoothness-accuracy} for the smoothness estimate into the regret rate of the non-adaptive policy $\pi_0$.

When the policy $\pi_0$ that is deployed in the \texttt{SACB} policy is rate optimal in the sense that $\iota_0(\beta, \alpha, d)=0$, then the resulting \texttt{SACB} policy is smoothness adaptive according to Definition~\ref{def:adaptivity}. More precisely, the adaptation cost for the \texttt{SACB} policy is poly-logarithmic in the horizon length with the degree $\frac{3d(\alpha+1)(2\overline{ \beta} + d)^2}{(2\beta+d)(\beta+d)(\underline{\beta}+d-1)}$, which is bounded from above for any dimension $d$, and hence can be replaced by some function $\iota(\beta, \lbeta, \ubeta, \alpha)$ independent of $d$ as required in Definition~\ref{def:adaptivity}. We next demonstrate this for the cases of at-most-Lipschitz-smooth and at-least-Lipschitz-smooth payoffs.

\subsubsection{Rate optimality with at-most-Lipschitz-Smooth Payoffs}
When the estimated smoothness in the \texttt{SACB} policy is less than 1, that is, $\hat{\beta}_{\texttt{SACB}}\le 1$, one may deploy the \textit{Adaptively Binned Successive Elimination} (\verb|ABSE|) policy from \cite{perchet2013multi} as the non-adaptive input policy $\pi_0$ to guarantee rate-optimal performance without prior information on the smoothness of the payoff functions. This is formalized by the following corollary.

\begin{corollary}[Rate optimality with at-most-Lipschitz-Smooth Payoffs]\label{corllary:global-adapt-beta<1}
	Consider the setting in Theorem \ref{theorem-GSE-ABSE-regret-bound}, and suppose that $\pi_0(\beta_0) = \texttt{ABSE}( \min(1,\beta_0))$. Then, one has
	\begin{equation}\label{eq:up-bound-beta<1}
	\mathcal{R}^\pi(\sfP;T)
	\le
	\bar{C} T^{\zeta(\beta, \alpha,d)} \l(\log T\r)^{\frac{3d(\alpha+1)(2\overline{ \beta} + d)^2}{(2\beta+d)(\beta+d)(\underline{\beta}+d-1)}}
	\qquad\forall \beta \in [\lbeta, 1].
	\end{equation}
\end{corollary}

\noindent
The \verb|ABSE| policy from \cite{perchet2013multi} relies on the knowledge of $\beta$, and achieves the rate-optimal regret of order $T^{\zeta(\beta, \alpha,d)}$ for any problem instance with $0<\beta\le1$. The \texttt{SACB} policy resulting from deploying \verb|ABSE| as an input policy when $\hat{\beta}_{\texttt{SACB}}\le 1$ is \emph{smoothness-adaptive} in the regime of smooth non-differentiable payoff functions with the adaptation penalty $\l(\log T\r)^{\frac{3d(\alpha+1)(2\overline{ \beta} + d)^2}{(2\beta+d)(\beta+d)(\underline{\beta}+d-1)}}$. 

\subsubsection{Rate optimality with at-least-Lipschitz-Smooth Payoffs}
When the estimated smoothness in the \texttt{SACB} policy is greater than 1, that is, $\hat{\beta}_{\texttt{SACB}}> 1$, one may deploy the \texttt{SmoothBandit} policy from \cite{hu2019smooth} as the non-adaptive input policy $\pi_0$. The \texttt{SmoothBandit} policy relies on the following additional assumption on the regularity of decision regions.

\begin{assumption}[Regularity] \label{assumption-regularity}
	Let $\cQ_k\coloneqq \l\{x\in[0,1]^d: (-1)^{k-1}(f_1(x)- f_2(x) ) \ge 0\r\}, k\in \cK$,  be the optimal decision regions. Then, each $\cQ_k$ is a non-empty $(c_0,r_0)$-regular set, where a Lebesgue measurable set $\cS$ is said to be $(c_0,r_0)$-regular if for all $x\in\cS$, one has
	$
	\lambda\l[ \cS \cap \mathrm{Ball}_2(x,r) \r]
	\ge
	c_0
	\lambda\l[\mathrm{Ball}_2(x,r) \r],
	$
	where $\mathrm{Ball}_2(x,r)$ is the Euclidean ball of radius $r$ centered around $x$ and $\lambda[\cdot]$ denotes the Lebesgue measure.
\end{assumption}

\noindent Under Assumption \ref{assumption-regularity}, the resulting \texttt{SACB} policy guarantees rate-optimal performance without prior information on the smoothness of the payoff functions. This is formalized by the following corollary.

\begin{corollary}[Rate optimality with at-least-Lipschitz-smooth payoffs]\label{corllary:global-adapt-beta>1}
Consider the setting in Theorem \ref{theorem-GSE-ABSE-regret-bound}, and suppose that $\pi_0(\beta_0) = \texttt{SmoothBandit}(\max(1,\beta_0))$ for $\beta_0\ge 1$. Then, if the decision regions associated with $\sfP$ satisfy the regularity condition in Assumption \ref{assumption-regularity}, one has
		\begin{equation}\label{eq:up-bound-beta>1}
		\mathcal{R}^\pi(\sfP;T)
		\le
		\bar{C}
		T^{\zeta(\beta, \alpha,d)} \l(\log T\r)^{\frac{3d(\alpha+1)(2\overline{ \beta} + d)^2}{(2\beta+d)(\beta+d)(\underline{\beta}+d-1)}+\frac{2\beta+d}{2\beta}}
		\qquad \forall \beta \in [1, \ubeta].
		\end{equation}
	
\end{corollary}

\noindent
The \texttt{SmoothBandit} policy relies on the knowledge of $\beta$ and achieves the near-optimal regret of order $\cO\l((\log T)^{\frac{2\beta+d}{2\beta}} T^{\zeta(\beta, \alpha,d)}\r)$ for any problem instance with $\beta\ge1$. The \texttt{SACB} policy, when paired with \texttt{SmoothBandit} as its non-adaptive input policy, guarantees near-optimality without prior knowledge of the smoothness in the regime of differentiable payoff functions, incurring the adaptation penalty~$\l(\log T\r)^{\frac{3d(\alpha+1)(2\overline{\beta} + d)^2}{(2\beta+d)(\beta+d)(\underline{\beta}+d-1)}}$. 

We note that the upper and lower bounds established in this regime with prior knowledge of the smoothness are separated by a factor $ \l(\log T\r)^{\frac{2\beta+d}{2\beta}}$ that is exponential in $d$. If the upper bound of \cite{hu2019smooth} is indeed optimal in the sense that the above factor cannot be removed, then Corollary~\ref{corllary:global-adapt-beta>1} establishes that the resulting \texttt{SACB} policy is smoothness adaptive in the sense of Definition \ref{def:adaptivity}. Otherwise, if another non-adaptive policy can be shown to eliminate the above factor and achieve the lower bound of order $\Omega\l(T^{\zeta(\beta, \alpha,d)}\r)$, then it could be deployed to construct a smoothness-adaptive \texttt{SACB} policy.

We conclude this subsection by noting that Corollaries \ref{corllary:global-adapt-beta<1} and \ref{corllary:global-adapt-beta>1} demonstrate that through the \texttt{SACB} policy one can achieve rate optimality without prior knowledge of the smoothness parameter $\beta$ in each of the two smoothness regimes that have been studied in the literature; that is, $\beta \leq 1$ in, e.g., \cite{perchet2013multi}, and $\beta \geq 1$ in \cite{hu2019smooth}. However, it is important to note that the \texttt{SACB} policy does not require prior knowledge of the regime in which the smoothness parameter lies in order to achieve rate optimality. This is formalized by the following remark.

\begin{remark}[Rate optimality with general smoothness]\label{remark:beta<1-and-beta>1}
Consider the setting in Theorem \ref{theorem-GSE-ABSE-regret-bound}, and suppose that $\pi_0(\beta_0) = \texttt{ABSE}(\beta_0)$ for $\beta_0\le 1$ and $\pi_0(\beta_0) = \texttt{SmoothBandit}(\beta_0)$ for $\beta_0 > 1$. Then, for $\beta\le 1$, one recovers the same regret bound as in \eqref{eq:up-bound-beta<1}, and for $\beta>1$, under Assumption \ref{assumption-regularity}, one recovers the same regret bound as in \eqref{eq:up-bound-beta>1}.
\end{remark}

\subsection{The SACB Policy}\label{sec:policy}

The \textit{Smoothness-Adaptive Contextual Bandits} (\texttt{SACB}) policy adaptively integrates a smoothness estimation sub-routine with an off-the-shelf non-adaptive policy that is rate optimal under prior knowledge of the smoothness. The smoothness estimation sub-routine consists of three consecutive steps: $(i)$ collecting samples in different regions of the covariate space; $(ii)$ estimating the payoff functions; and $(iii)$ examining a hypothesis test over the estimated functions. The policy repeats these steps until the smoothness estimation sub-routine is terminated. Then, the smoothness of the payoff functions is estimated based on the results of the hypothesis tests, and the estimate $\hat{\beta}_{\texttt{SACB}}$ is used as an input to an off-the-shelf non-adaptive policy that is designed to perform well under accurate tuning of the smoothness parmeter. We next formalize the \texttt{SACB} policy (see Algorithm 1), and discuss the estimation sub-routine in more detail in \S\ref{subsec:smoothness-estimation}.

\vspace{-.1cm}
\subsubsection{Smoothness Estimation under the SACB Policy}\label{subsec:smoothness-estimation}\vspace{-.15cm}
\noindent \textbf{Sampling.} In the \verb|SACB| policy, we consider the partition of the unit hypercube corresponding to $\cB_l$ with $l \leftarrow \left\lceil \frac{(\underline{\beta}+d-1)\log_{\oldnewam{q}} T}{(2\overline{ \beta} + d)^2}\right\rceil$. For each {bin} (hypercube) $\sfB\in \cB_l$, we collect samples for both actions in multiple rounds. Define the maximum round index as follows:
$
\bar{r}
\coloneqq
\lceil 2l\ubeta +\oldnewam{\upsilon}\log_{\textcolor{black}{q}} \log T
\rceil
$, \oldnewam{where $\upsilon$ is the under-smoothing parameter that will be introduced later.}
In every round~$r\in \{1,\dots,\bar{r}\}$, we collect $\textcolor{black}{q}^r$ samples for each action by alternating between them every time the covariate belongs to $\sfB$. If for some $\sfB\in \cB_l^{\textcolor{black}{q}}$ we reach $\bar{r}$ before the smoothness estimation sub-routine is terminated, we continue alternating between the arms every time the covariate belongs to $\sfB$. We denote by $T_{\texttt{SACB}}$ the time step at which the smoothness estimation sub-routine is terminated.\vspace{-.01cm}

\bigskip
	\begin{algorithm}[H]\footnotesize
		\caption{Smoothness-Adaptive Contextual Bandits (\texttt{SACB})}\label{algorithm-GSE}
		\textbf{Input:} Set of non-adaptive policies $\{\pi_0(\beta_0)\}_{\beta_0 \in [\lbeta , \ubeta]}$, horizon length $T$, minimum and maximum smoothness exponents $\underline{\beta}$ and $\ubeta$, a tuning parameter $\gamma$,
\color{black}
under-smoothing coefficient $\upsilon$, and a base for partitioning and advancing sampling counts $q$
\color{black}\\
		\textbf{Initialize: } $l \leftarrow \left\lceil \frac{(\underline{\beta}+d-1)\log_{\oldnewam{q}} T}{(2\overline{ \beta} + d)^2}\right\rceil$, $
			\bar{r}
			\leftarrow
			\lceil 2l\ubeta +\oldnewam{\upsilon}\log_{\oldnewam{q}} \log T
			\rceil
			$, and $\xi^{(\sfB)} \leftarrow 0 $, $N_k^{(\sfB)} \leftarrow 0 $  for all $\sfB \in \cB_l$ and $k\in\cK$\\
		\For{$t = 1, \dots$}{
			Determine the bin in which the current covariate is located: $\sfB \in \cB_l$ s.t. $ X_t\in \sfB$ \\	
			Alternate between the arms: $\pi_t \leftarrow 1 + \Indlr{N_{1}^{(\sfB)}>N_{2}^{(\sfB)}}$ \\	
			Update the counters: $N_{k}^{(\sfB)} \leftarrow N_{k}^{(\sfB)} + \Indlr{\pi_t = k}\, \forall k \in \cK$	\\
			\If{$N_{1}^{(\sfB)} + N_{2}^{(\sfB)} \ge 2 \times \textcolor{black}{q}^{r^{(\sfB)}}$ and $r^{(\sfB)} \le \bar{r}$}
			{
				\If{$\xi^{(\sfB)} = 0$ \textbf{\textsc{and}} $\sup_{k\in\cK, x \in \cM^{(\sfB)}} \l| \hat{f}_k^{(\sfB, r^{(\sfB)})}(x;j_1^{(\sfB)}) - \hat{f}_k^{(\sfB, r^{(\sfB)})}(x;j_{2}^{(\sfB)}) \r|  > \frac{ \gamma\l( \log T\r)^{\frac{d}{2\underline{\beta}} + \frac{1}{2}} }{\oldnewam{q}^{r^{(\sfB)}/2}}$ \tcc*[r]{see \eqref{eq:def-estimates}}} {
					Record $r^{(\sfB)}_{\mathrm{last}}$:  $r^{(\sfB)}_{\mathrm{last}} \leftarrow r^{(\sfB)}$; Raise the flag: $\xi^{(\sfB)} \leftarrow 1$
				}
				Update the sampling round index:
				$r^{(\sfB)} \leftarrow r^{(\sfB)} +1$; Reset the counters: $N_{k}^{(\sfB)} \leftarrow 0\; \forall k\in\cK$
			}		
			\If{\textbf{\textsc{[}} $\xi^{(\sfB^\prime)} = 1$ \textbf{\textsc{or}} $r^{(\sfB^\prime)} > \bar{r}$ \textbf{\textsc{]}} for all $\sfB^\prime \in \cB_l$}{
				Record $T_{\texttt{SACB}}:
				T_{\texttt{SACB}} \leftarrow t$\\
				\textbf{break}
			}
		}
		Estimate the smoothness: $\hat{\beta}_{\texttt{SACB}} \leftarrow \frac{1}{2l} \l[\min_{\sfB \in \cB_l} r^{(\sfB)}_{\mathrm{last}} - \oldnewam{\upsilon}  \log_{\textcolor{black}{q}} \log T\r] $
		\\
		Choose the corresponding non-adaptive policy $\pi_0 \leftarrow \pi_0(\min[\max[\lbeta, \hat{\beta}_{\texttt{SACB}}], \ubeta])$
		\\
		\For{$t=T_{\texttt{SACB}}+1,\dots,T$}{
			$\pi_t \leftarrow \pi_0\l(X_t \r)$}
	\end{algorithm}\normalsize

\medskip
\noindent \textbf{Estimation.} We briefly review the local polynomial regression method based on the analysis in \cite{audibert2007fast}; further analysis can be found in Appendix \ref{appendix-local-poly}. Let $\cD = \l\{ (X_i,Y_i) \r\}_{i=1}^n$ be a set of $n$ i.i.d. pairs $(X_i,Y_i) \in \cX\times \rR$, distributed according to a joint distribution~$P$. Denote by $\mu$ the marginal density of $X_i$'s and define the regression function $\eta(x)\coloneqq \rElr{Y\middle |  X=x}$. To estimate the value of the function $\eta$ at any point $x\in \cX$, we define the local polynomial regression method as follows.\vspace{-.1cm}
\begin{definition}[Local polynomial regression]\label{def-LPR}
	Fix a set of pairs $\cD = \l\{ (X_i,Y_i) \r\}_{i=1}^n$,a point $x\in \rR^d$,  a bandwidth $h>0$, an integer $p>0$, and the kernel function $K(\cdot)=\Indlr{\|\cdot\|_\infty\le 1}$.  Define by $\hat{\theta}_x(u; \cD, h, p) = \sum_{|s|\le p} \xi_s u^s$ a polynomial of degree~$p$ on $\rR^d$ that minimizes\vspace{-0.15cm}
	\begin{equation}\label{LPR-minimization-problem}
	\sum_{i=1}^n \l(Y_i -  \hat{\theta}_x(X_i-x; \cD, h, p)\r)^2 K\l( \frac{X_i - x}{h} \r).\vspace{-0.15cm}
	\end{equation}
The local 
estimator $\hat{\eta}^{\mathrm{LP}}(x; \cD, h, p)$ of the value $\eta(x)$ of the regression function $f(\cdot)$ at point $x$ is defined to be $\hat{\eta}^{\mathrm{LP}}(x; \cD, h, p) \coloneqq \hat{\theta}_x(0; \cD, h, p)$ if \eqref{LPR-minimization-problem} has a unique minimizer, and $\hat{\eta}^{\mathrm{LP}}(x; \cD, h, p) \coloneqq 0$ otherwise.
\end{definition}

\noindent Denote by $X_{k,1}^{(\sfB,r)}, X_{k,2}^{(\sfB,r)}, \dots$ and $Y_{k,1}^{(\sfB,r)}, Y_{k,2}^{(\sfB,r)}, \dots$ the successive covariates and outcomes when action $k$ is selected in $\sfB$ in round $r$, respectively. Denote by
$
\cD_k^{(\sfB,r)}
\coloneqq
\l\{
\l(X_{k,\tau}^{(\sfB,r)}, Y_{k,\tau}^{(\sfB,r)}\r)
\r\}_{\tau=1}^{ \textcolor{black}{q}^r}
$
the corresponding set of pairs. Define the two bandwidth exponents:
$
j^{(\sfB)}_1
\coloneqq
l$, and
$j^{(\sfB)}_2
\coloneqq
l + \lceil
\frac{1}{\underline{\beta}} \log_{\textcolor{black}{q}} \log T
\rceil.
$
Let $\tilde{l} \coloneqq \lceil \frac{\ubeta l}{\underline{\beta}} + \frac{ \log_{\textcolor{black}{q}} \log T}{\underline{\beta}}\rceil \vee \lceil (1+\ubeta)l + \log_{\textcolor{black}{q}} \log T\rceil$. For every bin $\sfB$ define the mesh points:\vspace{-0.1cm}
\begin{align*}
\cM^{(\sfB)}
&\coloneqq
\l\{
x = \l(\frac{m_1}{\textcolor{black}{q}^{\tilde{l}}},\dots, \frac{m_d}{\textcolor{black}{q}^{\tilde{l}}}\r): x \in \sfB, m_i \in \{1,\dots,\textcolor{black}{q}^{\tilde{l}}\} \text{ for } i\in\{1,\dots,d\}
\r\}.\vspace{-0.1cm}
\end{align*}
For every mesh point $x \in \cM^{(\sfB)}$, we form two separate estimates of the payoff functions, using local polynomial regression of degree $\lfloor \ubeta \rfloor$:\vspace{-0.1cm}
\begin{equation}\label{eq:def-estimates}
\hat{f}_k^{(\sfB,r)}(x;j)
\coloneqq
\hat{\eta}^{\mathrm{LP}}(x; \cD_k^{(\sfB,r)}, \textcolor{black}{q}^{-j}, \lfloor \ubeta \rfloor),
j\in\{j^{(\sfB)}_1,j^{(\sfB)}_2\}.\vspace{-0.1cm}
\end{equation}

\noindent \textbf{Hypothesis Test.}
At the end of each sampling round $r$ in bin $\sfB$, we check whether the difference between the estimations using the two bandwidth exponents $j^{(\sfB)}_1$ and $j^{(\sfB)}_2$ exceeds a predetermined threshold. Formally, for a tuning parameter $\gamma$, we check whether the following holds:\vspace{-0.1cm}
\begin{dmath}\label{eq-hypothesis-test-GSE}
	\sup_{k \in \cK, x \in \cM^{(\sfB)}} \l| \hat{f}_k^{(\sfB, r)}(x;j^{(\sfB)}_1) - \hat{f}_k^{(\sfB, r)}(x;j^{(\sfB)}_2) \r|
	\ge \frac{ \gamma\l( \log T\r)^{\frac{d}{2\underline{\beta}} + \frac{1}{2}} }{\textcolor{black}{q}^{\frac{r}{2}}}.
\end{dmath}
The left-hand side of \eqref{eq-hypothesis-test-GSE} is driven by two terms: the estimation bias of $\hat{f}_k^{(\sfB, r)}(x;j^{(\sfB)}_1)$, which is potentially larger due to a larger bandwidth; and the standard deviation of $\hat{f}_k^{(\sfB, r)}(x;j^{(\sfB)}_2)$, which is potentially larger since, on average, it is based on fewer samples. The right-hand side of \eqref{eq-hypothesis-test-GSE}, however, is proportional to the standard deviation of the estimate $\hat{f}_k^{(\sfB, r)}(x;j^{(\sfB)}_2)$. That is, by examining \eqref{eq-hypothesis-test-GSE}, we are detecting the number of samples that are required for the estimation bias of $\hat{f}_k^{(\sfB, r)}(x;j^{(\sfB)}_1)$ to dominate the standard deviation of $\hat{f}_k^{(\sfB, r)}(x;j^{(\sfB)}_2)$, which, as we will see, is dependent on the smoothness of the payoff functions. This dependence allows one to infer the smoothness of payoff functions with good precision with high probability. Denote by $r^{(\sfB)}_{\mathrm{last}}$ the smallest round index for which \eqref{eq-hypothesis-test-GSE} holds in {bin} $\sfB$ (upon this event, we set the flag $\xi^{(\sfB)}=1$). If \eqref{eq-hypothesis-test-GSE} never holds in $\sfB$, we simply set $r^{(\sfB)}_{\mathrm{last}} = \bar{r}$.

The quantity $r^{(\sfB)}_{\mathrm{last}}$ closely relates to the smoothness of the payoff functions. In what follows, we show that $\min_{\sfB \in \cB_l}r^{(\sfB)}_{\mathrm{last}}\approx 2 l \beta$ with high probability; this relation stems from $\textcolor{black}{q}^{r^{(\sfB)}_{\mathrm{last}}}$ essentially being the minimal number of samples required for the bias and standard deviation to be balanced for hypercube $\sfB$ under our procedure (in the sense of equation \ref{eq-hypothesis-test-GSE}).

We next develop high-probability bounds for $r^{(\sfB)}_{\mathrm{last}}$; following the above connection, these bounds are used for establishing the smoothness estimate in \eqref{eq:smoothness-est}, as well as in Theorem~\ref{theorem-GSE-smoothness-accuracy}. The next proposition provides a high-probability lower bound for $r^{(\sfB)}_{\mathrm{last}}$ for all the {bins} $\sfB \in \cB_l$.
\begin{prop}[High-probability lower bound for $r^{(\sfB)}_{\mathrm{last}}$] \label{proposition-GSE-lower-bound-round}
	Suppose that Assumption \ref{assumption-Holder-smoothness} holds for some $L>0$ and $\beta \in [\lbeta, \ubeta]$. Then, there exist constants $\underline{C}_r$, $\constvar[GSE-lower-bound-round1], \constvar[GSE-lower-bound-round2]$, and $\constvar[GSE-lower-bound-round3]$ such that for all $T\ge 1$,\vspace{-0.15cm}
	\[
	r^{(\sfB)}_{\mathrm{last}} <
	\underline{C}_r + 2l\beta + \left(\frac{d}{\underline{\beta}} + 1\right) \log_{\textcolor{black}{q}} \log T\vspace{-0.15cm}
	\]
	for some $\sfB\in \cB_l$, with probability less than $\constref{GSE-lower-bound-round1}  \l( \log T \r)^{\frac{d}{\underline{\beta}}} T^{ -\gamma ^2 \constref{GSE-lower-bound-round2} + \constref{GSE-lower-bound-round3} },
	$
	where the constants $\constref{GSE-lower-bound-round1}, \constref{GSE-lower-bound-round2},$ and $ \constref{GSE-lower-bound-round3}$ depend only on $\lbeta, \ubeta ,L,  \underline{\rho}, \bar{\rho}$, and $d$, and $\underline{C}_r$ depends only $\lbeta, \ubeta ,L,  \underline{\rho},$ and $ \bar{\rho}$.
	\end{prop}

\noindent
The proof of Proposition~\ref{proposition-GSE-lower-bound-round} is based on the discussion provided after \eqref{eq-hypothesis-test-GSE}. Since the payoff functions belong to $\cH(\beta,L)$, their estimation bias is bounded in each {bin} $\sfB\in\cB_l$. This implies that when the number of samples is ``small," the left-hand side of \eqref{eq-hypothesis-test-GSE} is dominated by the standard deviation of $\hat{f}_k^{(\sfB, r)}(x;j^{(\sfB)}_2)$, which is proportional to the right-hand side of \eqref{eq-hypothesis-test-GSE},  with high probability. The next result complements Proposition~\ref{proposition-GSE-lower-bound-round} by providing a high-probability upper bound for $\min_{\sfB \in \cB_l} r^{(\sfB)}_{\mathrm{last}}$.
\begin{prop}[High-probability upper bound for $\min_{\sfB \in \cB_l} r^{(\sfB)}_{\mathrm{last}}$] \label{proposition-GSE-upper-bound-round}
Suppose that Assumption \ref{assumption-Holder-smoothness} holds for some $L>0$ and $\beta \in [\lbeta, \ubeta]$, and that Assumption \ref{assumption-global-self-similarity} holds for some $b>0$ and $l_0 \ge 0 $. Then, there exist some $\sfB \in \cB_l$ and some constants $\overline{C}_r$, $\constvar[GSE-upper-bound-round1],$ and $\constvar[GSE-upper-bound-round2]$ such that for all $T\ge 1$,\vspace{-0.15cm}
	\[
	r^{(\sfB)}_{\mathrm{last}} >
	\overline{C}_r + 2l\beta +\left(\frac{d}{\underline{\beta}} + 3\right) \log_{\textcolor{black}{q}} \log T\vspace{-0.15cm}
	\]
	with probability less than $\constref{GSE-upper-bound-round1} T^{ -\gamma ^2 \constref{GSE-upper-bound-round2}},
	$
where the constants $\constref{GSE-upper-bound-round1}$ and $ \constref{GSE-upper-bound-round2}$ depend only on $\lbeta, \ubeta , L,b,  \underline{\rho}, \bar{\rho}$, and $d$, and $\overline{C}_r$ depends only on $\lbeta, \ubeta ,L, b,  \underline{\rho},$ and $ \bar{\rho}$.\vspace{-0.00cm}
\end{prop}

\noindent
The proof of Proposition~\ref{proposition-GSE-upper-bound-round} is again based on the discussion provided after \eqref{eq-hypothesis-test-GSE}. Since the set of payoff functions is self-similar, the estimation bias of the estimate $\hat{f}_k^{(\sfB, r)}(x;j^{(\sfB)}_1)$ remains ``large" in at least one of the {bins} $\sfB\in\cB_l$ and for one of the arms, which implies that for that specific {bin} and arm, if the number of samples is ``large" enough, then the left-hand side of \eqref{eq-hypothesis-test-GSE} is dominated by the aforementioned bias and eventually exceeds the right-hand side of \eqref{eq-hypothesis-test-GSE} with high probability.

Based on Proposition \ref{proposition-GSE-lower-bound-round} and \ref{proposition-GSE-upper-bound-round}, we estimate the smoothness of the problem as follows:\vspace{-0.15cm}
\begin{equation}\label{eq:smoothness-est}
\hat{\beta}_{\texttt{SACB}} = \frac{1}{2l}\l[\min_{\sfB \in \cB_l} r^{(\sfB)}_{\mathrm{last}} -  \textcolor{black}{\upsilon} \log_{\textcolor{black}{q}} \log T\r].\vspace{-0.15cm}
\end{equation}

\noindent Note that in order to avoid costly estimation errors this estimate is designed to be less than $\beta$ with high probability \oldnewam{when $\upsilon=\left(\frac{2d}{\underline{\beta}} + 4\right)$}, which is commonly referred to as ``under-smoothing" in the construction of confidence intervals; see, e.g., \cite{bickel1973some}, \cite{hall1992effect}, \cite{picard2000adaptive}, and \cite{gine2010confidence}.

We conclude this section with a discussion on the low sample complexity of our smoothness estimation sub-routine relative to the optimal regret rates. In order to achieve rate optimality one is required to estimate the smoothness $\beta$ with precision of order $\frac{1}{\log T}$. Broadly speaking, our proposed estimation sub-routine $(i)$ collects $n$ independent samples from payoff functions that are H\"{o}lder-smooth and self-similar; $(ii)$ partitions the unit hypercube into hypercubes of side-length $h$; and $(iii)$ estimates the payoff functions in each hypercube, using local polynomial regression. The resulting estimation bias is of order~$h^\beta$. However, using our proposed sub-routine, one may evaluate the estimation bias as $c h^\beta$ for some finite $c$. This results in an estimate of $\beta$ of the form $\frac{\log(c h^\beta)}{\log h} = \beta + \frac{c}{\log h}$. Hence, to achieve precision of order $\frac{1}{\log T}$, it suffices to have $h=T^{p}$ for some $p$. In order for the estimation sub-routine to perform well, one requires the estimation bias (which is of order $h^\beta$) and the estimation standard deviation (which is of order $\frac{1}{\sqrt{n h ^d}}$) to be balanced; that is, $n$ should be of order $T^q$ for some $q$. Finally, one may set $q$ to be arbitrarily small such that $n$ is not large relative to the optimal regret rate.

\vspace{-0.1cm}
\color{black}
\section{Numerical Analysis}\label{sec-numerics}\vspace{-0.1cm}
We simulate the performance of the policies \texttt{SACB}, \texttt{ABSE}($\beta$) that is initiated by the correct smoothness parameter $\beta$, and \texttt{ABSE}($\tilde \beta$) that is initiated by some misspecified smoothness parameter $\tilde \beta$. The simulation code can be found at \url{https://github.com/armmn/Smoothness-Adaptive-Contextual-Bandits}, and a summary of results in various settings is provided in Appendix~\ref{appendix:numerics}.

\medskip
\noindent 
\textbf{Setup.} We consider here a setting with one-dimensional covariates with support in the segment $[0,1]$, and rewards that are Gaussian with standard deviation $\sigma=0.05$. (While for simplicity our model assumes binary rewards, we note that theoretical results that are similar to the ones established in this paper could be obtained for other sub-Gaussian reward distributions.) In order to demonstrate the cost of smoothness misspecification (and, correspondingly, the value of adaptation), we use payoff functions that are inspired by those that were used in the analysis of Example~\ref{exp:cost-smoothness-misspecification}. Let\vspace{-0.1cm}
$
\phi(x)
=
	\Indlr{|x|\le1}(1-|x|)^\beta
$.
For some parameters $C,L_1,M,$ and $m$, the payoff functions are defined as follows:\vspace{-0.1cm}
\begin{align*}
	f_1(x)
	&=
	\begin{cases}
		\frac{1  + L_1(\frac{1}{2})^\beta - L_1x^\beta}{2} & \text{ if } 0 \le  x \le \frac{1}{2};\\
		\frac{1}{2} + \sum\limits_{j=1}^{m} (-1)^j(2M)^{-\beta} C \phi\l( 2M[(2-2x)-a_j] \r)  & \text{ if } \frac{1}{2} <  x \le 1;
	\end{cases}
	\\ & &\\
	\bigskip f_2(x)
	&=
	\begin{cases}
		\frac{1  + L_1(\frac{1}{2})^\beta - L_1x^\beta}{2} & \text{ if } 0 \le  x \le \frac{1}{2};\\
		\frac{1}{2}   & \text{ if } \frac{1}{2} <  x \le 1,
	\end{cases}
\end{align*}
where $a_j = \frac{j + \frac{1}{2}}{M}$ for each $j=1,\ldots,m$. We have considered two different settings in terms of the parameter selection: Settings I and II that are inspired by the problem design in Parts 1 and 2 of Example~\ref{exp:cost-smoothness-misspecification}, respectively, and are detailed in Table \ref{table-num-values}.

\newcolumntype{C}[1]{%
	>{\vbox to 5ex\bgroup\vfill\centering}%
	p{#1}%
	<{\egroup}}
\begin{table}[H]
		
	\centering          
	\small   
	\begin{tabular}{|l|l|l|l|l|l|l|l|}
		\hline
		\textbf{Setting} & \boldmath{$M$} & \boldmath{$\alpha$} & \boldmath{$m$} & \boldmath{$\tau$} &\boldmath{$L_1$} & \boldmath{$C$}
		\\
		\hline
		\textbf{I} & $\frac{1}{16} \l \lfloor \frac{1}{2c_0} \l( \frac{2 \log 2 }{T} \r)^{\frac{-\tau}{2\tau+1}} \r\rfloor ^{\frac{1}{\beta}}$ & 0.01 & $\lfloor  M^{1-\alpha \beta}\rfloor$ & 0.8 & 1 & 1
		 \\
		\hline
		\textbf{II}  & $2^{\lceil \frac{\log_2\l( T /2\log 2\r)}{\tau+1 }\rceil}/4$ & $\frac{1}{\beta}$ & $\lfloor  M^{1-\alpha \beta}\rfloor$ & 0.6 & 1 & 50
		\\
		\hline
	\end{tabular}
	\caption{\footnotesize Numerical values of parameters}\vspace{-0.0cm}
\label{table-num-values}
\end{table}
\vspace{-.2cm}
The \texttt{ABSE} policy is tuned by two parameters $c_0 = 2$ and $\gamma_{\texttt{ABSE}} = 2$, and the $\texttt{SACB}$ policy is tuned by $\gamma_{\texttt{SACB}} = 0.145$, $q=1.1$, $\upsilon=0.325$, $\lbeta=0.4$, and $\ubeta = 1$.

\medskip
\noindent 
\textbf{Results.} Plots comparing the performance of the aforementioned policies for the horizon length $T =~2\times 10^6$ in Settings I, and II  appear in Figure~\ref{fig:sim_fixed_T} where $\beta = 0.9$ and  $\beta = 0.5$, respectively, and where $\tilde \beta \in \{0.4,0.45,\dots,1\}$ in both cases.  We note that the results are consistent across different smoothness values, horizon lengths, and payoff structures; for a summary of results for additional values of $\beta$ and $T$, as well as payoff functions that are generated randomly according to fractional Brownain motion and Brownian bridge models see Appendix~\ref{appendix:numerics}.\vspace{-0.3cm}
\begin{figure}[h]
	\centering
	\begin{subfigure}[t]{0.44\textwidth}
		\raisebox{-\height}{\includegraphics[width=\textwidth]{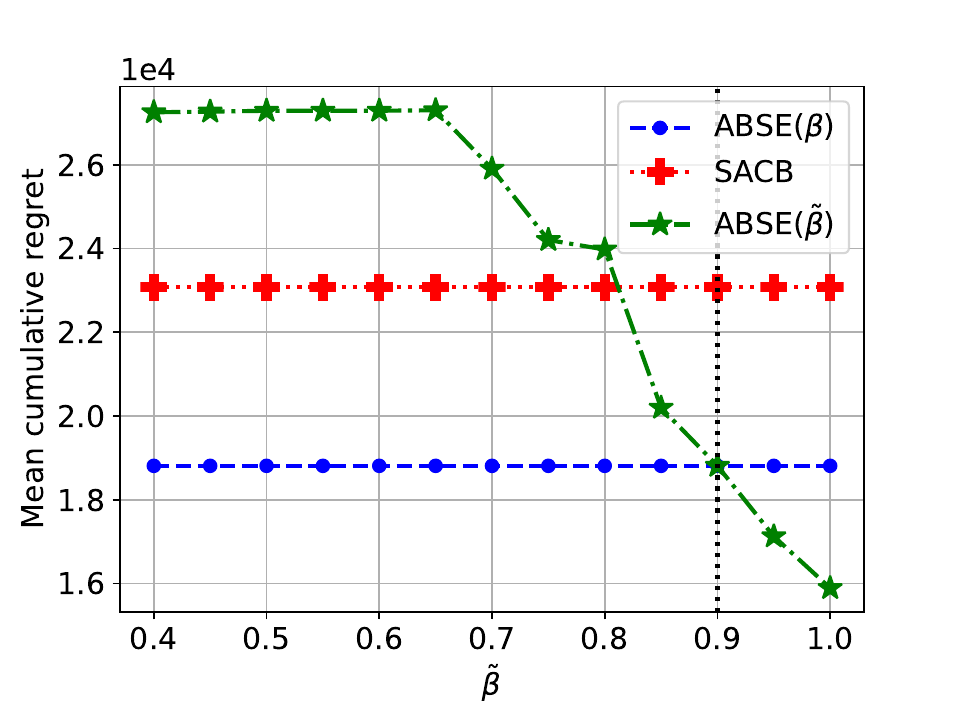}}
	\end{subfigure}
	\begin{subfigure}[t]{0.44\textwidth}		\raisebox{-\height}{\includegraphics[width=\textwidth]{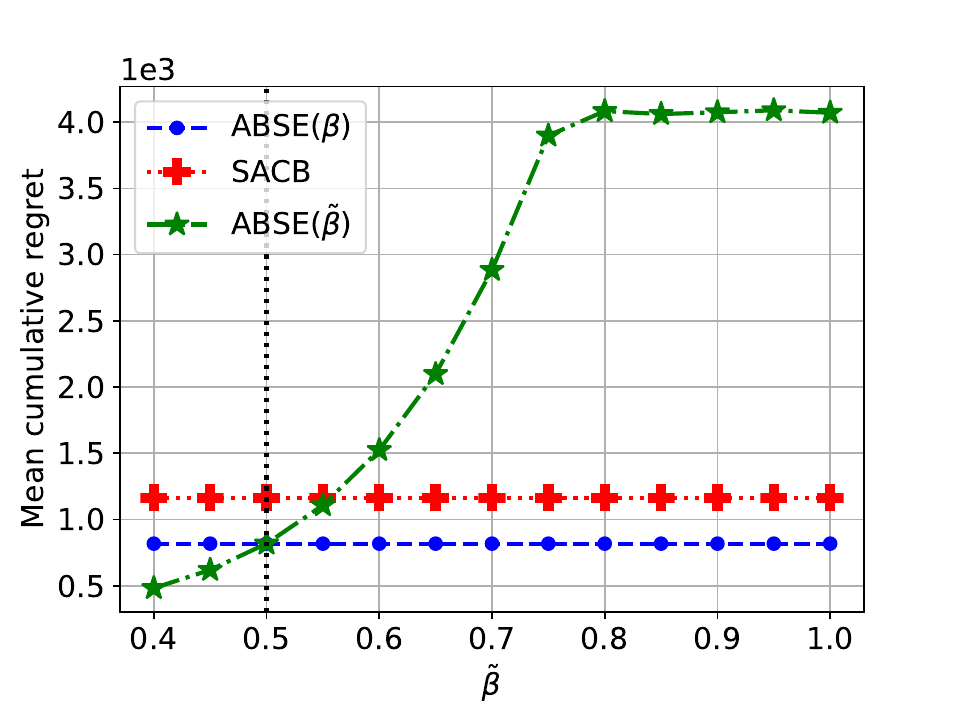}}
	\end{subfigure}
	\vspace{-0.2cm}
	\caption{\footnotesize \color{black}Average cumulative regret for horizon length $T=2\times 10^6$. The value of smoothness parameter $\beta$ is denoted using the vertical dashed line. \textit{Left:} Setting I with $\beta = 0.9$; \textit{Right:} Setting II with $\beta = 0.5$. The maximum 95\% confidence interval widths for \texttt{ABSE}($\beta$), \texttt{SACB}, and \texttt{ABSE}($\tilde{\beta}$) are $0.003 \times 10^4$, $0.057\times 10^4$, and $0.004\times 10^4$ in Setting I, and $0.005\times 10^3$, $0.122\times 10^3$, and $0.024\times 10^3$ in Setting II, respectively. }\label{fig:sim_fixed_T}\vspace{-0.1cm}
\end{figure}

Each point in the plots corresponds to the cumulative regret of the respective policy, averaged over $40$ iterations. In Setting I (Setting II), as $\tilde \beta$ gets smaller (larger) compared to $\beta$, the cost of smoothness misspecification incurred by the misspecified \texttt{ABSE}($\tilde{\beta}$) policy increases and dominates the cost of adaptation incurred by the \texttt{SACB} policy (relative to the performance achieved by \texttt{ABSE}($\beta$), tuned by the accurate smoothness parameter $\beta$). As one may expect, we observe that the value of adaptation relative to deploying a misspecified \texttt{ABSE} policy increases with the horizon length; see Appendix~\ref{appendix:numerics}.
However, as $\tilde{\beta}$ gets closer to $\beta$ the cost of smoothness misspecification decreases and eventually, when the smoothness misspecification is sufficiently small, the cost associated with it becomes smaller than the cost of adaptation. These results therefore demonstrate that the value captured by adapting to the smoothness of the payoff functions might be particularly large for long decision horizons, and when there is a risk of significant misspecification of the smoothness parameter.

Note that, in some cases (including the ones demonstrated in Figure~\ref{fig:sim_fixed_T}), the performance of \texttt{ABSE}($\tilde{\beta}$) with a misspecified value of $\tilde{\beta}$ may be slightly better than the performance of \texttt{ABSE}($\beta$), tuned by the true smoothness parameter $\beta$. This is not in contrast with the analysis provided here and in prior work: recall that \cite{perchet2013multi} only established the \emph{minimax} optimality of \texttt{ABSE}($\beta$) tuned by the true smoothness $\beta$ over the class of problems with payoff functions of smoothness $\beta$. In particular, \texttt{ABSE}($\beta$) is not guaranteed to minimize regret under every function in that class.

\color{black}

\section{Concluding Remarks}

\textbf{Summary and Implications.} In this paper, we studied the problem of designing algorithms that adapt to unknown smoothness of payoff functions in a non-parametric contextual MAB setting. First, we showed that, in general, it is impossible to achieve rate-optimal performance simultaneously over different classes of payoff functions in the following sense: there exist some pairs of smoothness parameters $(\gamma, \beta)$ such that no policy can simultaneously attain optimal regret rates over the problems $ \cP(\gamma,\alpha, d)$ and $ \cP(\beta,\alpha, d)$. This implies that, in general, one may incur a non-trivial adaptation cost when the smoothness of payoffs is misspecified or a priori unknown.

We overcame the impossibility of adaptation by leveraging a self-similarity condition that, as we established, does not reduce the minimax complexity of the problem. We devised a general policy based on: $(i)$ inferring the smoothness of the payoff functions from observations that are collected throughout the decision-making process; and $(ii)$ using effective non-adaptive policies like off-the-shelf input policies. We showed that this approach allows one to guarantee the best regret rate that is achievable given the underlying smoothness exponent $\beta$ that characterizes the problem instance, without requiring prior knowledge of that smoothness. Our policy is \emph{smoothness adaptive}, in the sense of achieving that rate up to a multiplicative term that is poly-logarithmic in the horizon length and a multiplicative constant that may depend on other problem parameters. We demonstrate our approach by leveraging non-adaptive policies designed for payoff functions that are at most Lipschitz smooth and at least Lipschitz smooth to guarantee rate optimality without prior information on the underlying payoff smoothness.

\color{black}
The SACB policy demonstrates how Lepski's approach, which was originally designed for learning smoothness in static settings, can be appropriately adapted for obtaining performance guarantees in a dynamic operational setting. Particularly, our analysis shows that such a new variant of this approach, together with leveraging the structure of effective non-adaptive policies, leads to smoothness adaptivity and rate optimality in the non-parametric contextual MAB problem without prior knowledge of the smoothness of payoff functions.

\color{black}
\medskip\noindent
\textbf{Avenues for Future Research.} Our study presents several new research directions. One open question is whether our impossibility statement holds for any pair of H\"{o}lder exponents $(\gamma, \beta)$. More precisely, it is left to understand whether there is any policy that can achieve rate-optimal performance simultaneously over two different problem instances characterized by different smoothness parameters, without additional assumptions such as that of self-similarity. If not, it would be desirable to extend our analysis in order to establish impossibility of adaptation for any pair of H\"{o}lder exponents $(\gamma, \beta)$.

Another path is to study how tight the lower bound provided in Theorem \ref{theorem-impossibility-adaptation-smoothness} is. In other words, for pairs of smoothness parameters $(\gamma, \beta)$ over which impossibility of adaptation is established, can one design a MAB policy that achieves rate-optimal performance over problem instances characterized by $\gamma$, and incurs the regret rate provided in Theorem \ref{theorem-impossibility-adaptation-smoothness} for problem instances characterized by $\beta$?

Another interesting question is whether there exists any assumption weaker than that of self-similarity that allows for designing smoothness-adaptive policies. If not, a natural direction would be to study the adaptation cost that one has to incur with respect to the self-similarity constant $b$ in Definition~\ref{def:global-ss}.

\oldnewam{
Last but not least, one may design a unified approach for adapting to the smoothness of payoffs and the intrinsic dimension of covariates. In fact, after our work, \cite{li2020dimension} proposed a method for adapting to the intrinsic dimension of covariates in a non-parametric contextual MAB setting. It would be interesting to design a variant of the \texttt{SACB} policy that, when paired with a base policy that adapts to the intrinsic dimension of covariates, achieves adaptivity over both smoothness and intrinsic dimension with a small adaptation cost that does not grow fast with the covariate space dimension.
}

\vspace{0.05cm}
	\footnotesize
\setstretch{0.8}
\bibliographystyle{chicago}
\bibliography{references}
\bibliographystyle{informs2014}
\normalsize
\setstretch{1.41}
\newpage
\appendix

\section{Proofs of main results}\label{appendix: proof of main results}\vspace{-0.1cm}
	\subsection{Proof of Part \text{1} of Theorem \ref{theorem-impossibility-adaptation-smoothness}}
	\label{section-proof-impossibility-part1}\vspace{-0.1cm}
	
	In this section, we describe the proof of Part \text{1} of Theorem \ref{theorem-impossibility-adaptation-smoothness}. The proof follows the next steps. In Step 1, we discuss some notations and definitions including the definition of inferior sampling rate. In Step 2, we leverage a result from \cite{rigollet2010nonparametric} to connect regret and inferior sampling rate, which enables one to simplify analysis by focusing on the inferior sampling rate throughout the proof.
	
	In Step 3, which is a key step of the proof, we reduce the problem at hand to a hypothesis testing problem by introducing a novel construction of a set of problem instances. This set consists of a nominal problem instance with smoothness parameter $\gamma$ and some other problem instances with smoothness parameter $\beta$, each of which differs from the nominal one only over a specific region of the covariate space. These problem instances are designed to connect between the amount of exploration and the ability to identify the correct smoothness parameter. This construction is designed for showing that if a policy achieves rate-optimal performance for smooth problems, it is likely to underexplore in ``rougher" problems, and hence not be able to differentiate between the two.
	
	In Step 4, we verify that the aforementioned problem instances satisfy the margin condition. In Step~5, we show that with high probability, the number of covariates that belong to the regions mentioned in Step~3 grow linearly with respect to the time horizon and the volume of the regions. In Steps 6 and 7, we show that since the policy is rate optimal for $\gamma$-smooth problems, it cannot distinguish between the nominal problem and at least one of the $\beta$-smooth problems. In Step 8, we lower bound the inferior sampling rate due to not being able to identify the correct smoothness parameter (some high-level ideas in Steps 7 and 8 are adopted from the proof of Theorem 3 in \cite{locatelli2018adaptivity}). In Step 9, we revert back the lower bound on inferior sampling rate to a lower bound on regret.

	\paragraph{Step 1 (Preliminaries).}
	For any policy $\pi$ and decision horizon length $T$, let $\cS^\pi(\sfP; T)$ be the \textit{inferior sampling rate} defined as\vspace{-0.15cm}
	\begin{equation}\label{eq:inferior-sampling-def}
		\cS^\pi(\sfP; T) \coloneqq \mathbb{E}^\pi \left[ \sum\limits_{t=1}^T \Indlr{ f_{\pi^\ast_t}(X_t) \neq f_{\pi_t}(X_t)} \right].
	\end{equation}
	Fix a covariate distribution $\bm{\mathrm{P}}_X$. For any policy $\pi$ and function $f:[0,1]^d \rightarrow [0,1]$, denote by $\cS^\pi(f;T)$ the inferior sampling rate of $\pi$ when $\bm{\mathrm{P}}_X$ is the covariate distribution, $\rElr{Y_{1,t} \;\middle|\; X_t} = f(X_t)$, and $\rElr{Y_{2,t} \;\middle|\; X_t} = \frac{1}{2}$.  Notably, the oracle policy $\pi_f^\ast$ is given by $\pi_f^\ast(x) = 2-\Indlr{f(x) \ge \frac{1}{2}}$. We further denote by $\rP_{\pi,f}$ and $\rE_{\pi,f}$ the corresponding probability and expectation. Finally, for any H\"{o}lder exponent $\beta>0$ and margin parameter $\alpha>0$, define: \vspace{-0.1cm}
	\[
	\cR^\pi_{\beta,\alpha}(T) \coloneqq \sup_{\sfP \in \cP(\beta, \alpha, d)}\mathcal{R}^\pi(\sfP;T);
	\qquad
	\cS^\pi_{\beta,\alpha}(T) \coloneqq \sup_{\sfP \in \cP(\beta, \alpha, d)}\mathcal{S}^\pi(\sfP;T).\vspace{-0.1cm}
	\]
	Fix $T\ge1$, two H\"{o}lder exponents $0<\beta < \gamma \le 1 $, a margin parameter $0\le \alpha \le \frac{1}{\gamma}$, a~positive Lipschitz constant $L$, and positive constants $\underline{\rho},\bar{\rho}$ such that $\bm{\mathrm{P}}_X$ satisfies Assumption~\ref{assumption-covar-dist} with parameters $\underline{\rho},\bar{\rho}$.

	\paragraph{Step 2 (From regret to inferior sampling rate).}
	The following lemma implies that it suffices to first analyze inferior sampling rate and then, revert the result back to regret.
	\begin{lemma}[\protect{\citealt[Lemma 3.1]{rigollet2010nonparametric}}]
		\label{lemma-regret-inferior-sampling-rate}	For any $\alpha >0 $ under the margin condition in Assumption \ref{assumption-margin}, one has\vspace{-0.1cm}
		\[
		\cS^\pi(\sfP; T)
		\le
		C_{sr} T^{\frac{1}{\alpha+1}} [\cR^\pi(\sfP; T)]^{\frac{\alpha}{\alpha+1}},
		\]
		for any policy $\pi$ and some positive constant $C_{sr}$.
	\end{lemma}
	\noindent By Lemma \ref{lemma-regret-inferior-sampling-rate}, we have $\cS^\pi_{\beta,\alpha}(T)
	\le
	C_{sr} T^{\frac{1}{\alpha+1}} \l[\cR^\pi_{\beta,\alpha}(T)\r]^{\frac{\alpha}{\alpha+1}}$.
	Note that when $\pi$ is rate-optimal over $\cP(\gamma,\alpha,d)$, Lemma~\ref{lemma-regret-inferior-sampling-rate} implies that for some constants $C_r,C_s>0$, one has:\vspace{-0.15cm}
	\[
	\cR^\pi_{\gamma,\alpha}(T)
	\le
	C_r T^{1-\frac{\gamma(1+\alpha)}{2\gamma+d}}
	\eqqcolon
	\cR^\ast_{\gamma,\alpha}(T);
	\qquad
	\cS^\pi_{\gamma,\alpha}(T)
	\le
	C_s T^{1-\frac{\gamma\alpha}{2\gamma+d}}
	\eqqcolon
	\cS^\ast_{\gamma,\alpha}(T).\vspace{-0.15cm}
	\]

	\paragraph{Step 3 (Constructing problem instances).}
	In this step we reduce our problem to a hypothesis testing problem. In order to do so, we first construct some problem instances. Defining $M\coloneqq \lceil \Delta^{\alpha-\frac{d}{\beta}} \rceil$ and $C_{\phi} \coloneqq
	\frac{L}{2^{2+2\beta}}$, fix the parameter~$\Delta>0$ such that\vspace{-0.15cm}
	\[
	\frac{64C_{\phi}^2\Delta^2\cS^\ast_{\gamma,\alpha}(T)}{3M}
	=
	\frac{1}{2}.\vspace{-0.15cm}
	\]
	This selection of $\Delta$ implies that for large enough $T$ one has
	$
	C_{\phi}\Delta \le \frac{1}{4}.
	$
	For any $0<\kappa \le 1$, define the functions $\tilde{\psi}_{\kappa}$ and $\hat{\psi}_{\kappa}$ as follows:\vspace{-0.2cm}
	\begin{align*}
		\tilde{\psi}_{\kappa}(x)
		\coloneqq
		\begin{cases}
			\l| 1-\|x\|_\infty \r|^\kappa & \text{if } 0\le \|x\|_\infty \le 1;\\
			0 & \text{o.w.};
		\end{cases}
		\qquad
		\hat{\psi}_{\kappa}(x)
		\coloneqq
		\begin{cases}
			\l| 1-\|x\|_\infty \r|^\kappa & \text{if } 0\le \|x\|_\infty \le 1;\\
			-\l| \|x\|_\infty -1  \r|^\kappa & \text{if } 1\le \|x\|_\infty \le 2;\\
			-1 & \text{o.w.}
		\end{cases}
	\end{align*}
	Note that $\tilde{\psi}_{\kappa}\in \cH_{\rR^d}(\kappa,1)$ and $\hat{\psi}_{\kappa}\in \cH_{\rR^d}(\kappa,2)$. The following two lemmas (proved in Appendix \ref{app:auxiliary}) are the main tools to analyze the smoothness of the payoff functions that we construct in this step.
	\begin{lemma}[Scaling and smoothness]\label{lemma-multip-presrves-smoothness}
		Suppose $f\in \cH_{\rR^d}(\beta, L)$ for some  $0<\beta\le 1$ and $L>0$, and define the function $g$ such that $g(x) = C^{-\beta} f(Cx)$ for all $x\in \rR^d$ and some $C>0$. Then, $g \in   \cH_{\rR^d}(\beta, L)$.
	\end{lemma}
	
	\begin{lemma}[Min/Max and smoothness]\label{lemma-max-presrves-smoothness}
		Suppose $f,g\in \cH_{\cX}(\beta, L)$ for some $\cX \subseteq \rR^d$,  $0<\beta\le 1$ and $L>0$, and define the functions $h_1 \coloneqq \max(f,g)$ and $h_2 \coloneqq \min(f,g)$. Then, $h_1, h_2\in   \cH_{\cX}(\beta, L)$.
	\end{lemma}
	
	\noindent Define a hypercube $H_0 \coloneqq [0, 2\Delta^{\frac{\alpha}{d}}]^d$ with a center $a_0\coloneqq (\Delta^{\frac{\alpha}{d}}, \Delta^{\frac{\alpha}{d}},\dots, \Delta^{\frac{\alpha}{d}}) \in \rR^d $. Define the function\vspace{-0.15cm}
	\[
	\phi_0(x) \coloneqq \frac{1}{2} - C_{\phi} \cdot \min\l\{\Delta, \Delta^{\frac{\alpha\gamma}{d}} \cdot \tilde{\psi}_{\gamma}\l(  \Delta^{-\frac{\alpha}{d}} [x-a_0]\r) \r\}.\vspace{-0.15cm}
	\]
	By Lemmas \ref{lemma-multip-presrves-smoothness} and \ref{lemma-max-presrves-smoothness}, $\phi_0 \in \cH(\gamma,L)$ since $C_{\phi}\le L$. Consider the grid $G$ that partitions the hypercube $H_0$ into $M $ disjoint hypercubes $\l( H_m \r)_{m\in \{1,\dots,M\}}$. Let $a_m\in \rR^d, m\in \{1,\dots,M\}$, be the center of the hypercube $H_m$. Let $\tilde{H}_m$ be the hypercube of side-length \oldnewam{$l \coloneqq \frac{\Delta^{\frac{\alpha}{d}}}{4M^{\frac{1}{d}}}$} centered around $a_m$. Note that $\tilde{H}_m \subset H_m$ and that the side-length of $H_m$ is \oldnewam{$4l$}. Define the functions $\phi_m, m\in\{1,\dots,M\},$ as follows:\vspace{-0.15cm}
	\[
	\phi_m(x)
	\coloneqq
	\max\l\{ \phi_0(x), \frac{1}{2}+ C_{\phi} \cdot \Delta \cdot \hat{\psi}_{\beta}\l( 2l^{-1}[x-a_m] \r) \r \}.\vspace{-0.15cm}
	\]
	Since $C_{\phi}= \frac{L}{2^{2+2\beta}}$ and $\Delta \le 2^{\beta+1} l^\beta$ for large enough $T$, by Lemmas \ref{lemma-multip-presrves-smoothness} and \ref{lemma-max-presrves-smoothness} one has that $\phi_m \in \cH(\beta,L)$ for $1\le m \le M$.

	\paragraph{Step 4 (Verifying the margin condition).} By examining different cases of parametric values, we verify that the margin condition is satisfied with parameters $\alpha$ and $C_0 \coloneqq 2^d  3d\overline{\rho} C_{\phi}^{-\alpha}$ when $f_1=\phi_m$ and $f_2=\frac{1}{2}$ for all $0\le m \le M$.
		\begin{itemize}
		\item For $m=0$ and $\delta \le C_{\phi} \Delta$, one has\vspace{-0.2cm}
		\begin{align}\label{eq:margin-help1}
			\bm{\mathrm{P}}_X\l\{0 < |\phi_0(X) - \frac{1}{2}| \le \delta\r\}
			&\le
			\bar{\rho}
			\int_{H_0}
			\Indlr{C_{\phi}\Delta^{\frac{\alpha\gamma}{d}} \cdot \tilde{\psi}_{\gamma}\l(  \Delta^{-\frac{\alpha}{d}} [x-a_0]\r) \le \delta} dx
			\nonumber
			\\
			&\le
			2^d\bar{\rho}\Delta^{\alpha}
			\int_{[0,1]^d}
			\Indlr{ \tilde{\psi}_{\gamma}\l(  x\r) \le \delta C_{\phi}^{-1} \Delta^{-\frac{\alpha\gamma}{d}}} dx
			\nonumber
			\\
			&\le
			2^d\bar{\rho}\Delta^{\alpha}
			\l[1-
			\int_{[0,1]^d}
			\Indlr{ \|x\|_\infty \le 1-  \delta^{\frac{1}{\gamma}} C_{\phi}^{-\frac{1}{\gamma}} \Delta^{-\frac{\alpha}{d}}} dx
			\r]
			\nonumber
			\\
			&\le
			2^d\bar{\rho}\Delta^{\alpha}
			\l[1-
			\l(1-  \delta^{\frac{1}{\gamma}} C_{\phi}^{-\frac{1}{\gamma}} \Delta^{-\frac{\alpha}{d}}\r)^d
			\r]
			\nonumber
			\\
			&\le
			2^d\bar{\rho}\Delta^{\alpha}
			\l[d \delta^{\frac{1}{\gamma}} C_{\phi}^{-\frac{1}{\gamma}} \Delta^{-\frac{\alpha}{d}}
			\r]
			\overset{(a)}{\le }
			2^dd\bar{\rho}C_{\phi}^{-\alpha}\Delta^{\frac{1}{\gamma}-\frac{\alpha}{d}} \delta^{\alpha}
			\le
			2^dd\bar{\rho}C_{\phi}^{-\alpha} \delta^{\alpha},
		\end{align}
		where (a) holds since $\alpha\le\frac{1}{\gamma}$.\vspace{-0.1cm}
		
		\item For $m=0$ and $\delta > C_{\phi} \Delta$, one has\vspace{-0.2cm}
		\begin{align*}
			\bm{\mathrm{P}}_X\l\{0 < |\phi_0(X) - \frac{1}{2}| \le \delta\r\}
			\le
			2^d \bar{\rho} \Delta^{\alpha}
			\le
			2^d \bar{\rho} C_{\phi}^{-\alpha}\delta^{\alpha}.\vspace{-0.2cm}
		\end{align*}
		\item For $1 \le m \le M$ and $\delta \le C_{\phi} \Delta$, one has\vspace{-0.2cm}
		\begin{align}\label{eq:margin-help2}
			\bm{\mathrm{P}}_X\l\{0 < |\phi_m(X) - \frac{1}{2}| \le \delta\r\}
			&\le
			\bm{\mathrm{P}}_X\l\{0 < C_{\phi} \cdot \Delta \cdot  \l| 1- 2l^{-1}\|x-a_m \|_\infty \r|^\beta  \le \delta,\, X\in \tilde{H}_m \r\}
			\nonumber
			\\&\quad+
			\bm{\mathrm{P}}_X\l\{0<C_{\phi} \cdot \Delta \cdot \l| 2l^{-1}\|x-a_m \|_\infty - 1\r|^\beta \le \delta, X \in H_m\setminus\tilde{H}_m \r\}
			\nonumber
			\\&\quad+
			\bm{\mathrm{P}}_X\l\{0 < |\phi_0(X) - \frac{1}{2}| \le \delta\r\}.
		\end{align}
		Next, we analyze each term separately. One has
		\begin{align}\label{eq:margin-help3}
			\bm{\mathrm{P}}_X\l\{0 < C_{\phi} \cdot \Delta \cdot  \l| 1- 2l^{-1}\|x-a_m \|_\infty \r|^\beta \le \delta,\, X\in \tilde{H}_m \r\}
			\nonumber
			&
			\\
			&	\hspace{-4cm}\le
			\bar{\rho}
			\int_{\tilde{H}_m}
			\Indlr{C_{\phi} \cdot \Delta \cdot  \l| 1- 2l^{-1}\|x-a_m \|_\infty \r|^\beta \le \delta} dx
			\nonumber
			\\
			&	\hspace{-4cm}\le
			\bar{\rho}
			\int_{\tilde{H}_m}
			\Indlr{\|x-a_m \|_\infty \ge \frac{l}{2}\l(1 - C_{\phi}^{-\frac{1}{\beta}}  \Delta^{-\frac{1}{\beta}}\delta^{\frac{1}{\beta}}\r)} dx
			\nonumber
			\\
			&\hspace{-4cm}=
			\bar{\rho}2^{-d}l^d\l[1-\l(1 - C_{\phi}^{-\frac{1}{\beta}}  \Delta^{-\frac{1}{\beta}}\delta^{\frac{1}{\beta}}\r)^d\r]
			\nonumber
			\\
			&\hspace{-4cm}\overset{(a)}{\le}
			\bar{\rho}2^{-d}l^d
			\l[d \delta^{\frac{1}{\beta}} C_{\phi}^{-\frac{1}{\beta}} \Delta^{-\frac{1}{\beta}}
			\r]
			\overset{(b)}{\le }
			d\bar{\rho}2^{-2d}C_{\phi}^{-\alpha} \delta^{\alpha},
		\end{align}
		where (a) follows from the inequality $(1-x)^r\ge 1-rx$ for $0\le x \le 1, r\ge 1$, and (b) holds by $\alpha\le\frac{1}{\gamma}$. Similarly,\vspace{-0.2cm}
		\begin{align}\label{eq:margin-help4}
			\bm{\mathrm{P}}_X\l\{0<C_{\phi} \Delta  \l| 2l^{-1}\|x-a_m \|_\infty - 1\r|^\beta \le \delta, X \in H_m\setminus\tilde{H}_m \r\}
			&
			\nonumber
			\\
			&
			\hspace{-4cm}\le
			\bar{\rho}
			\int_{H_m\setminus\tilde{H}_m}
			\Indlr{C_{\phi} \Delta  \l| 2l^{-1}\|x-a_m \|_\infty - 1\r|^\beta \le \delta} dx
			\nonumber
			\\
			& \hspace{-4cm} \le
			\bar{\rho}
			\int_{H_m\setminus\tilde{H}_m}
			\Indlr{   \|x-a_m \|_\infty \le \frac{l}{2}\l(1 + C_{\phi}^{-\frac{1}{\beta}}  \Delta^{-\frac{1}{\beta}}\delta^{\frac{1}{\beta}}\r)} dx
			\nonumber
			\\
			&\hspace{-4cm}=
			\bar{\rho}2^{-d}l^d\l[\l(1 + C_{\phi}^{-\frac{1}{\beta}}  \Delta^{-\frac{1}{\beta}}\delta^{\frac{1}{\beta}}\r)^d-1\r]
			\nonumber
			\\
			&\hspace{-4cm}\overset{(a)}{\le}
			\bar{\rho}l^d C_{\phi}^{-\frac{1}{\beta}}  \Delta^{-\frac{1}{\beta}}\delta^{\frac{1}{\beta}}
			\overset{(b)}{\le }
			\bar{\rho} C_{\phi}^{-\alpha} \delta^{\alpha},
		\end{align}
		where (a) follows from the inequality $(1+x)^r\le 2^rx+1$ for $0\le x \le 1, r\ge 1$, and (b) holds by $\alpha\le\frac{1}{\gamma}$. Putting together \eqref{eq:margin-help1}, \eqref{eq:margin-help2}, \eqref{eq:margin-help3}, and \eqref{eq:margin-help4}, yields for $\delta \le C_{\phi} \Delta$:\vspace{-0.2cm}
		\[
		\bm{\mathrm{P}}_X\l\{0 < |\phi_m(X) - \frac{1}{2}| \le \delta\r\}
		\le
		C_0 \delta^\alpha.\vspace{-0.2cm}
		\]
		\item The case
		$1 \le m \le M$ and $\delta > C_{\phi} \Delta$ can be analyzed similar to the case $m =0$ and $\delta > C_{\phi} \Delta$.
	\end{itemize}
	
	\paragraph{Step 5 (Desirable event).} For $m \in \{1,\dots,M\}$, define
	$
	Q_{m}\coloneqq \sum_{t=1}^T  \Indlr{X_t \in \tilde{H}_m} \eqqcolon  \sum_{t=1}^T Z_{m,t}
	$
	to be the number of times periods at which the realized covariates belong to the hypercube $\tilde{H}_m$. Define
	$
	\cA\coloneqq \l\{\exists m \in \{1,\dots,M\}:  Q_{m} < \frac{\underline{\rho}}{2}Tl^d \text{ or } Q_{m} > \frac{3\bar{\rho}}{2}Tl^d\r\}
	$
	to be the event where $Q_{m}$ is less than $\frac{\underline{\rho}}{2}Tl^d$ or larger than $\frac{3\bar{\rho}}{2}Tl^d$ for at least one value of $m\in\{1,\dots,M\}$. Note that\vspace{-0.15cm}
	\[
	\rPlr{\cA}
	\le
	\sum_{m=1}^{M}
	\rPlr{Q_{m} < \frac{\underline{\rho}}{2}Tl^d} + \rPlr{Q_{m} < \frac{3\bar{\rho}}{2}Tl^d}.\vspace{-0.1cm}
	\]
	In order to bound each of the summands on the right hand side of the above inequality, one may apply Bernstein's inequality in the following lemma \ref{lemma-Bernstein-inequality} to $Q_m$:
	\begin{lemma}[Bernstein inequality]\label{lemma-Bernstein-inequality}
		Let $X_1,\dots,X_n$ be random variables with range $|X_i|\le M$ and
		$
		\sum\limits_{t=1}^{n}\mathrm{Var}\left[X_t\,\vert\,X_{t-1},\dots,X_1\right] = \sigma^2.
		$
		Let $S_n=X_1+\dots+X_n.$ Then for all $a\ge0$\vspace{-0.3cm}
		\[
		\mathbb{P}\{S_n\ge \mathbb{E}[S_n] +a\}
		\le \exp\left(-\frac{a^2/2}{\sigma^2+Ma/3} \right).\vspace{-0.15cm}
		\]
	\end{lemma}
	\noindent Note that since $\underline{\rho} l^d T\le \rE Z_{m,t} \le \bar{\rho}Tl^d$, $|Z_{m,t}|\le 1$, and $\Var Z_{m,t} \le \rE Z_{m,t}^2 \le 2\bar{\rho}l^d$, one obtains:\vspace{-0.2cm}
	\begin{align*}
		\rPlr{\cA}
		&\le
		M\exp\l(-c_1Tl^d/5\r)
		\\
		&
		\overset{(a)}{\le}
		c_2 [\cS^\ast_{\gamma,\alpha}]^{\frac{\alpha\beta -d}{2\beta +d  -\alpha\beta}}  \exp\l( -c_3 T[\cS^\ast_{\gamma,\alpha}(T)]^{-\frac{d}{2\beta +d  -\alpha\beta}} \r)
		\\
		&
		\overset{}{\le}
		c_2 T^{\frac{\alpha\beta -d}{2\beta +d  -\alpha\beta}}
		\exp\l( -c_4 T^{\frac{2\beta(2\gamma+d-\alpha\gamma)+\alpha d (\gamma-\beta)}{(2\gamma+d)(2\beta + d -\alpha\beta)}} \r)
		\overset{(b)}{\le}
		c_5 T^{-3},
	\end{align*}
	for large enough $T$ and constants $c_1,c_2,c_3,c_4,c_5>0$, where (a) follows from the definition of $M$ and $l$, and (b) holds by $\frac{2\beta(2\gamma+d-\alpha\gamma)+\alpha d (\gamma-\beta)}{(2\gamma+1)(2\beta + d -\alpha\beta)} > 0 $ for $\alpha \le \frac{1}{\gamma}$.
	For any problem instance $\sfP$ and horizon length $T$, denote the inferior sampling rate of $\pi$ when the event $\cA$ does not occur by\vspace{-0.15cm}
	\[
	\bar{\cS}^\pi(\sfP; T) \coloneqq \mathbb{E}^\pi \left[ \sum\limits_{t=1}^T \Indlr{ f_{\pi^\ast_t}(X_t) \neq f_{\pi_t}(X_t)} \;\middle |\; \bar{\cA} \; \right].\vspace{-0.15cm}
	\]
	Define $\bar{\cS}^\pi_{\gamma,\alpha}(T) \coloneqq \sup\limits_{\sfP \in \cP(\gamma, \alpha, d)}\bar{\mathcal{S}}^\pi(\sfP;T) $. Note that\vspace{-0.15cm}
	\begin{equation*}\label{eq-inf-sampling-rate-A-fails}
		\l( 1- \rPlr{\cA} \r)\bar{\cS}^\pi(\sfP; T)
		\le
		\cS^\pi(\sfP; T)
		\le
		\bar{\cS}^\pi(\sfP; T) + T \rPlr{\cA},\vspace{-0.2cm}
	\end{equation*}
	which implies that\vspace{-0.05cm}
	\begin{equation}\label{eq-relationship-S-and-S-bar}
		\l| \bar{\cS}^\pi_{\gamma,\alpha}(T)
		-
		\cS^\pi_{\gamma,\alpha}(T)\r|
		\le c_4T^{-2}.\vspace{-0.05cm}
	\end{equation}
	For the rest of the proof, all probabilities and expectations will be computed conditional on $\bar \cA$.
	
	\paragraph{Step 6 (Selecting a single problem with smoothness $\beta$).} Let $N_{m,T} \coloneqq \sum_{t=1}^T \Indlr{\pi_t = 1, X_t \in H_m}$ denote the number of times policy $\pi$ selects arm 1 when realized covariates belong to the hypercube~$H_m$. By definition, $\rE_{\pi, \phi_0}^T\l[ \sum_{m=1}^M N_{m,T}\; \middle| \; \bar{\cA}\; \r]  \le \bar{\cS}^\pi_{\gamma,\alpha}(T)$, implying that there exists some $ m^\ast \in \{1,\dots,M\}$ such that\vspace{-0.0cm}
	\[
	\rE_{\pi, \phi_0}^T\l[N_{m^\ast,T}\; \middle| \; \bar{\cA}\; \r]  \le \frac{\bar{\cS}^\pi_{\gamma,\alpha}(T)}{M}
	\le
	\frac{\cS^\pi_{\gamma,\alpha}(T)}{M} + c_4T^{-2},\vspace{-0.0cm}
	\]
	where the last inequality holds by \eqref{eq-relationship-S-and-S-bar}.\vspace{-0.1cm}

	\paragraph{Step 7 (Likelihood of distinguishing between different smoothness parameters).} We show that policy~$\pi$ cannot distinguish between $\phi_0$ and $\phi_{m^\ast}$ with a strictly positive probability. For any set of samples $\l\{ (\pi_t, X_t, Y_{\pi_t,t})\r\}_{t=1}^T$, define the log-likelihood ratio $L_{m,T}=L_{m,T}\l( \l\{ (\pi_t, X_t, Y_{\pi_t,t})\r\}_{t=1}^T \r)$ for~$m\in \{1,\dots,M\}$ as:\vspace{-0.1cm}
	\begin{align*}
		L_{m,T}
		&\coloneqq
		\sum_{t=1}^T \log \l(  \frac{\rP_{\pi, \phi_0}\l\{ Y_{\pi_t,t} \; \middle | \; \pi_t, X_t\r\}}{\rP_{\pi, \phi_m}\l\{ Y_{\pi_t,t} \; \middle | \; \pi_t, X_t\r\}} \r)
		\\
		&\le
		\sum_{t=1}^T \Indlr{\pi_t = 1, X_t \in H_{m}} \cdot \l[Y_{\pi_t,t} \log \l(  \frac{\phi_0(X_t)}{\phi_m(X_t)} \r) +  (1-Y_{\pi_t,t}) \log \l(  \frac{(1-\phi_0(X_t))}{(1-\phi_m(X_t))}\r)
		\r]
		\\
		&\le
		\sum_{t=1}^T \Indlr{\pi_t = 1, X_t \in H_{m}} \cdot \l[Y_{\pi_t,t}   \frac{(\phi_0(X_t)-\phi_m(X_t))}{\phi_m(X_t)}  +  (1-Y_{\pi_t,t}) \frac{(\phi_m(X_t)-\phi_0(X_t))}{(1-\phi_m(X_t))}
		\r]
		\\
		&=
		\sum_{t=1}^T \Indlr{\pi_t = 1, X_t \in H_{m}} \cdot  \frac{(Y_{\pi_t,t}-\phi_m(X_t) )(\phi_0(X_t)-\phi_m(X_t))}{\phi_m(X_t)(1-\phi_m(X_t))},
	\end{align*}
	where the last inequality follows from $\log(1+x) \le x$ for all $x>0$. By taking expectations of the above inequality and conditioning on the event $\bar{\cA}$ for $m=m^\ast$, one obtains: \vspace{-0.1cm}
	\begin{align}\label{eq-expected-log-likelihood-upperbound}
		\rE_{\pi,\phi_0}\l[L_{m^\ast,T}\; \middle| \; \bar{\cA}\; \r]
		&\le
		\rE_{\pi,\phi_0}\l[\sum_{t=1}^T \Indlr{\pi_t = 1, X_t \in H_{m^\ast}} \cdot  \frac{(Y_{\pi_t,t}-\phi_{m^\ast}(X_t) )(\phi_0(X_t)-\phi_{m^\ast}(X_t))}{\phi_{m^\ast}(X_t)(1-\phi_{m^\ast}(X_t))} \; \middle| \; \bar{\cA}\; \r]
		\nonumber
		\\
		&=
		\rE_{\pi,\phi_0}\l[\sum_{t=1}^T \rElr{ \Indlr{\pi_t = 1, X_t \in H_{m^\ast}} \cdot  \frac{(Y_{\pi_t,t}-\phi_{m^\ast}(X_t) )(\phi_0(X_t)-\phi_{m^\ast}(X_t))}{\phi_{m^\ast}(X_t)(1-\phi_{m^\ast}(X_t))}\; \middle |\; X_t} \; \middle| \; \bar{\cA}\; \r]
		\nonumber
		\\
		&=
		\rE_{\pi,\phi_0}\l[\sum_{t=1}^T \Indlr{\pi_t = 1, X_t \in H_{m^\ast}} \cdot  \frac{(\phi_0(X_t)-\phi_{m^\ast}(X_t))^2}{\phi_{m^\ast}(X_t)(1-\phi_{m^\ast}(X_t))}\; \middle| \; \bar{\cA}\; \r]
		\nonumber
		\\
		&\overset{(a)}{\le}
		\frac{64C_{\phi}^2\Delta^2}{3}
		\rE_{\pi,\phi_0}\l[\sum_{t=1}^T \Indlr{\pi_t = 1, X_t \in H_{m^\ast}}\; \middle| \; \bar{\cA}\; \r]
		\nonumber
		\\
		& = \frac{64C_{\phi}^2\Delta^2}{3}
		\rE_{\pi,\phi_0}\l[N_{m^\ast,T}\; \middle| \; \bar{\cA}\; \r]
		\overset{(b)}{\le} \frac{64C_{\phi}^2\Delta^2\cS^\ast_{\gamma,\alpha}(T)}{3M} + c_4T^{-2}
		\overset{(c)}{\le}
		1,
	\end{align}
	for large enough $T$, where: (a) follows from $C_{\phi}\Delta\le \frac{1}{4}$; (b) follows from the definition of $m^\ast$; and (c) holds by the definition of $\Delta$.

	\paragraph{Step 8 (Lower bound on \oldnewam{regret as a function of inferior sampling rate}).}
	Let $\tilde{N}_{m,T} \coloneqq \sum_{t=1}^T \Indlr{\pi_t = 1, X_t \in \tilde{H}_m}$
	denote the number of times policy $\pi$ selects arm 1 when realized covariates belong to the hypercube $\tilde{H}_m$. We next use two lemmas in order to show that with a strictly positive probability one has $\tilde{N}_{m^\ast,T} < \frac{\underline{\rho}Tl^d}{2}$ conditional on the event $\bar{\cA}$ under problem $m^\ast$, implying that $\pi$ selects an inferior arm at least $\frac{\underline{\rho}Tl^d}{2}$ times. The first lemma is a simple variation of Lemma 2.6 in \cite{Tsybakov2008introduction} and is proved for completeness in Appendix \ref{app:auxiliary}; the second lemma is a straightforward extension of Lemma 19 in \cite{kaufmann2016complexity}.
	\begin{lemma}[Hypothesis testing error probability]\label{Tsybakov2008introduction-lemma-2.6}
		Let $\rho_0,\rho_1$ be two probability distributions supported on $\cX$, with $\rho_0$ absolutely continuous with respect to $\rho_1$. Then, for any measurable function~$\Psi:\cX\rightarrow\{0, 1\}$:\vspace{-0.1cm}
		\begin{equation*}
			\mathbb{P}_{\rho_0} \{\Psi(X) = 1\} + \mathbb{P}_{\rho_1} \{\Psi(X) = 0\} \ge \frac{1}{2} \exp(-\mathrm{KL}(\rho_0,\rho_1)).
		\end{equation*}
	\end{lemma}
	
	\begin{lemma}[Log-likelihood ratio and historical events]\label{lemma-relationship-log-likelihood-events}
		For any event\\ $\cE \in \cF_t^-=\sigma\l(
		\pi_1,X_1,Y_{\pi_1,1}, \dots , \pi_{T},X_{T},Y_{\pi_{T},T}\r)$ and an arbitrary event $\cA$, one has\vspace{-0.1cm}
		\[
		\rE_{\pi,\phi_0}\l[L_{m,T}\; \middle| \; \cE, \cA  \r]
		\ge
		\log\l( \frac{\rP_{\pi,\phi_0}\l\{ \cE \; \middle| \;  \cA \r\} }{\rP_{\pi,\phi_m}\l\{ \cE \; \middle| \;  \cA \r\}}
		\r).
		\]
	\end{lemma}
	
	\noindent Denote by $\rho_0$ and $\rho_{m}$ the distributions of $\tilde{N}_{m,T}$ under the problems 0 and $m$ conditional on the event~$\bar{\cA}$. Define the test function $\Psi(x)=\Indlr{x \ge \frac{\underline{\rho}Tl^d}{2}}$. With this selection of $\rho_0$, $\rho_m$, and $\Psi$, Lemma~\ref{Tsybakov2008introduction-lemma-2.6} yields:\vspace{-0.2cm}
	\[
	\rP_{\pi,\phi_0}\l\{\tilde{N}_{m^\ast,T} \ge \frac{\underline{\rho}Tl^d}{2}\; \middle| \; \bar{\cA}\; \r\}
	+
	\rP_{\pi,\phi_{m^\ast}}\l\{\tilde{N}_{m^\ast,T} < \frac{\underline{\rho}Tl^d}{2}\; \middle| \; \bar{\cA}\; \r\}
	\ge
	\frac{1}{2} \exp(-\mathrm{KL}(\rho_0,\rho_{m^\ast})).
	\]
	To establish a lower bound on the right hand side of the above inequality, we note that:\vspace{-0.15cm}
	\begin{align*}
		\rE_{\pi,\phi_0}\l[L_{m^\ast,T}\; \middle| \; \bar{\cA}\; \r]
		&=
		\sum_{s=1}^T
		\rE_{\pi,\phi_0}\l[L_{m^\ast,T}\; \middle| \; \bar{\cA},  \tilde{N}_{m^\ast,T}=s\r]\rP_{\pi,\phi_0}\l\{ \tilde{N}_{m^\ast,T}=s \; \middle| \;  \bar{\cA}\; \r\}
		\\
		&\ge
		\sum_{s=1}^T
		\log\l( \frac{\rP_{\pi,\phi_0}\l\{ \tilde{N}_{m^\ast,T}=s \; \middle| \;  \bar{\cA}\; \r\} }{\rP_{\pi,\phi_{m^\ast}}\l\{ \tilde{N}_{m^\ast,T}=s \; \middle| \;  \bar{\cA}\; \r\}}
		\r)\rP_{\pi,\phi_0}\l\{ \tilde{N}_{m^\ast,T}=s \; \middle| \;  \bar{\cA}\; \r\}
		=
		\mathrm{KL}(\rho_0,\rho_{m^\ast}),
	\end{align*}
	where the inequality follows from Lemma \ref{lemma-relationship-log-likelihood-events}. The last two inequalities, along with \eqref{eq-expected-log-likelihood-upperbound}, yield\vspace{-0.1cm}
	\[
	\rP_{\pi,\phi_0}\l\{\tilde{N}_{m^\ast,T} \ge \frac{\underline{\rho}Tl^d}{2}\; \middle| \; \bar{\cA}\; \r\}
	+
	\rP_{\pi,\phi_{m^\ast}}\l\{\tilde{N}_{m^\ast,T} < \frac{\underline{\rho}Tl^d}{2}\; \middle| \; \bar{\cA}\; \r\}
	\ge
	\frac{1}{2} \exp(-1).\vspace{-0.1cm}
	\]
	Next, we show that $\rP_{\pi,\phi_0}\l\{\tilde{N}_{m^\ast,T} \ge \frac{\underline{\rho}Tl^d}{2}\; \middle| \; \bar{\cA}\; \r\}$ is small. We apply Markov's inequality to obtain:\vspace{-0.15cm}
	\begin{align*}
		\rP_{\pi,\phi_0}\l\{\tilde{N}_{m^\ast,T} \ge \frac{\underline{\rho}Tl^d}{2}\; \middle| \; \bar{\cA}\; \r\}
		&\le
		\frac{\rE_{\pi,\phi_0}\l[\tilde{N}_{m^\ast,T}  \; \middle| \; \bar{\cA}\; \r]}{\frac{\underline{\rho}Tl^d}{2}}
		\overset{(a)}{\le}
		\frac{\frac{\cS^\ast_{\gamma,\alpha}(T)}{M} + c_4T^{-2}}{\frac{\underline{\rho}Tl^d}{2}}
		\\
		&
		\overset{(b)}{\le}
		\frac{2^d [\cS^\ast_{\gamma,\alpha}]^{1+\frac{\alpha\beta}{2\beta +d  -\alpha\beta}} + c_4l^{-d}T^{-2}}{\frac{\underline{\rho}T}{2}}
		\overset{}{\le}
		c_5
		T^{\frac{\alpha d (\beta - \gamma)}{(2\beta+ d - \alpha\beta)(2\gamma+d)}}\overset{(c)}{\le}
		\frac{1}{4} \exp(-1),
	\end{align*}
	for large enough $T$ and some constant $c_5>0$, where: (a) follows from the definition of $m^\ast$ and \eqref{eq-relationship-S-and-S-bar}; (b) holds due to the definition of $l$ and $M$; and (c) holds due to the fact that $\frac{\alpha d (\beta - \gamma)}{(2\beta+ d - \alpha\beta)(2\gamma+d)}<0$ since $\alpha\le \frac{1}{\gamma}$. The last two displays yield that for large enough $T$, one has
	$
	\rP_{\pi,\phi_{m^\ast}}\l\{\tilde{N}_{m^\ast,T} < \frac{\underline{\rho}Tl^d}{2}\; \middle| \; \bar{\cA}\; \r\}
	\ge
	\frac{1}{4e}.
	$
	By definition, when event $\bar{\cA}$ holds, at least $\underline{\rho}Tl^d$ times realized covariates belong to the hypercube $\tilde{H}_{m^\ast}$ \oldnewam{where $f_1(x)\ge f_2(x)+\frac{\Delta}{2}$ for problem $m^\ast$}, that is, for some constant $c_6>0$, one has:\vspace{-0.1cm}
	\oldnewam{
		\begin{equation}\label{eq:inferior-sampling-rate-m*}
			\cR^\pi_{\beta,\alpha}(T)
			\ge
			\frac{\underline{\rho}Tl^d}{2}
			\cdot \frac{\Delta}{2}
			\cdot
			\rP_{\pi,\phi_{m^\ast}}\l\{\tilde{N}_{m^\ast,T} < \frac{\underline{\rho}Tl^d}{2}\; \middle| \; \bar{\cA}\; \r\}
			\rPlr{\bar{\cA}}
			\ge
			\frac{\underline{\rho}Tl^d}{32e}
			\ge
			c_6 T\l[ \cS^\ast_{\gamma,\alpha}(T) \r]^{-\frac{\beta+d}{2\beta +d  -\alpha\beta}}.\vspace{-0.1cm}
		\end{equation}
	}
	\oldnewam{The final result follows by noting that $\cS^\pi_{\gamma,\alpha}(T)
		\le
		\tilde{C}_{sr} T^{\frac{1}{\alpha+1}} \l[T^{\zeta(\gamma,\alpha,d)}\r]^{\frac{\alpha}{\alpha+1}}$. This concludes the proof.\hfill$\blacksquare$}

	\subsection{Proof of Part \text{2} of Theorem \ref{theorem-impossibility-adaptation-smoothness}}
	The proof follows similar lines of argument as in the proof of Part \text{1} of Theorem \ref{theorem-impossibility-adaptation-smoothness}.\vspace{-0.1cm}
	\paragraph{Step 1 (Preliminaries).}
	Fix time horizon length $T\ge1$, and some H\"{o}lder exponent $ \gamma > 1 $, some margin parameter $0\le \alpha \le 1$, some positive Lipschitz constants $L$, and some positive constants $\underline{\rho},\bar{\rho}$ such that $\bm{\mathrm{P}}_X$, the covariate distribution, satisfies Assumption \ref{assumption-covar-dist} with parameters $\underline{\rho},\bar{\rho}$.

	\paragraph{Step 2 (From regret to inferior sampling rate).}
	By Lemma \ref{lemma-regret-inferior-sampling-rate}, we have $\cS^\pi_{\beta,\alpha}(T)
	\le
	C_{sr} T^{\frac{1}{\alpha+1}} \l[\cR^\pi_{\beta,\alpha}(T)\r]^{\frac{\alpha}{\alpha+1}}$.
	Note that by the assumption that $\pi$ is rate-optimal over $\cP(\gamma,\alpha,d)$ and Lemma \ref{lemma-regret-inferior-sampling-rate}, one has\vspace{-0.1cm}
	\[
	\cR^\pi_{\gamma,\alpha}(T)
	\le
	C_r T^{1-\frac{\gamma(1+\alpha)}{2\beta+d}}
	\eqqcolon
	\cR^\ast_{\gamma,\alpha}(T),
	\qquad
	\cS^\pi_{\gamma,\alpha}(T)
	\le
	C_s T^{1-\frac{\gamma\alpha}{2\gamma+d}}
	\eqqcolon
	\cS^\ast_{\gamma,\alpha}(T),\vspace{-0.1cm}
	\]
	for some constants $C_r,C_s>0$.

	\paragraph{Step 3 (Constructing problem instances).}
	We will reduce our problem to a hypothesis testing problem. To do so, we  construct some problem instances first. Define the parameter $\Delta>0$ such that\vspace{-0.1cm}
	\[
	\frac{64C_{\phi}^2\Delta^2\cS^\ast_{\gamma,\alpha}(T)}{3}
	=
	\frac{1}{2},\vspace{-0.2cm}
	\]
	where we define $C_{\phi} \coloneqq \frac{L}{2^{2\beta}}$. Note that the definition of $\Delta$ implies that for large enough $T$, one has $C_{\phi}\Delta \le \frac{1}{4}$. Define the function:\vspace{-0.1cm}
	\[
	\phi_0(x) \coloneqq \frac{1}{2} - C_{\phi} \cdot (\frac{1}{2} - x_1).\vspace{-0.1cm}
	\]
	Note that $\phi_0 \in \cH(\gamma,L)$ since $C_{\phi}\le L$. Define the hypercube $H \coloneqq [\frac{1}{2} - \Delta, \frac{1}{2}]\times [0,1]^{d-1}$, with a center $a_0\coloneqq (\frac{1- \Delta}{2}, \frac{1}{2},\dots, \frac{1}{2}) \in \rR^d  $, and the function:\vspace{-0.2cm}
	\[
	\phi_1(x)
	\coloneqq
	\phi_0(x)
	+
	2 C_{\phi} \cdot \Delta \cdot \tilde{\psi}\l( 2\Delta^{-1}[x-a_0] \r).\vspace{-0.2cm}
	\]
	For any $0<\kappa \le 1$, define the functions $\tilde{\psi}$:\vspace{-0.1cm}
	\begin{align*}
	\tilde{\psi}(x)
	&\coloneqq
	\begin{cases}
	\l| 1-|x_1| \r| & \text{if } |x_1| \le 1\\
	0 & \text{o.w.}
	\end{cases}.
	\end{align*}
	Note that by Lemmas \ref{lemma-multip-presrves-smoothness} and \ref{lemma-max-presrves-smoothness}, $\phi_0 \in \cH(1,L)$ since $C_{\phi}\le \frac{L}{2^{2\beta}}$.\vspace{-0.1cm}

\paragraph{Step 4 (Verifying the margin condition).}
	We verify that the margin condition is satisfied with parameters $\alpha$ and $ C_0 \coloneqq \frac{5\bar{\rho}}{2C_{\phi}}$ when $f_1=\phi_m$ and $f_2=\frac{1}{2}$ for all $0\le m \le 1$.\vspace{-0.1cm}
	\begin{itemize}
		\item For $m=0$ and $0< \delta \le 1$, one has\vspace{-0.1cm}
		\begin{equation*}
		\bm{\mathrm{P}}_X\l\{0 < |\phi_0(X) - \frac{1}{2}| \le \delta\r\}
		\le
		\frac{2\bar{\rho}\delta}{C_{\phi}}
		\le
		\frac{2\bar{\rho}\delta^\alpha}{C_{\phi}}.\vspace{-0.15cm}
		\end{equation*}
		\item For $m = 1$ and $0< \delta \le 1$, one has\vspace{-0.1cm}
		\begin{equation*}
		\bm{\mathrm{P}}_X\l\{0 < |\phi_1(X) - \frac{1}{2}| \le \delta\r\}
		\le
		\frac{5\bar{\rho}\delta}{2C_{\phi}}
		\le
		\frac{5\bar{\rho}\delta^\alpha}{2C_{\phi}}.\vspace{-0.1cm}
		\end{equation*}
	\end{itemize}

\paragraph{Step 5 (Desirable event).} Note that for $x\in \tilde H \coloneqq [\frac{1}{2} - \frac{7\Delta}{10}, \frac{1}{2} - \frac{\Delta}{6}]\times [0,1]^{d-1}$, the first arm is optimal \oldnewam{by a gap of at least $\frac{\Delta}{2}$} under the problem $m=1$. Define $Q\coloneqq \sum_{t=1}^T  \Indlr{X_t \in \tilde{H}} \eqqcolon  \sum_{t=1}^T Z_{t}$ to be the number of times covariates fall into the the hypercube $\tilde{H}_m$ during the entire time horizon. Define the event\vspace{-0.2cm}
\oldnewam{\[
	\cA\coloneqq \l\{Q < \frac{2}{15}\underline{\rho}T\Delta\r\}\vspace{-0.2cm}
	\]
}
	to be the event on which the number of covariates that have fallen into the hypercube $\tilde{H}$ is less than \oldnewam{$\frac{2}{15}\underline{\rho}T\Delta$}.
	In order to bound $	\rPlr{\cA}$, one can apply Bernstein's inequality in Lemma \ref{lemma-Bernstein-inequality} to $Q$ by noting that \oldnewam{$\rE Z_{t} \ge \frac{2}{15}\underline{\rho}\Delta$}, $|Z_{t}|\le 1$, and \oldnewam{$\Var Z_{t} \le \rE Z_{t}^2 \le \frac{2}{15}\underline{\rho}\Delta$} to obtain\vspace{-0.2cm}
	\begin{equation*}
		\rPlr{\cA}
		\le
		\oldnewam{\exp\l(-\frac{2}{15}\underline{\rho}T\Delta/5\r)}
		\overset{(a)}{\le}
		\exp\l( -c_1 T[\cS^\ast_{\gamma,\alpha}(T)]^{-\frac{1}{2}} \r)
		\overset{}{\le}
		\exp\l( -c_1 T^{\frac{(1+\frac{\alpha}{2})\gamma + \frac{d}{2}}{(2\gamma+d)}} \r)
		\overset{}{\le}
		c_2 T^{-3},\vspace{-0.2cm}
	\end{equation*}
	for large enough $T$ and constants $c_1,c_2>0$, where (a) follows from the definition of $\Delta$.
	
\noindent For any problem instance $\sfP$ and time horizon length $T$, denote by\vspace{-0.25cm}
	\[
	\bar{\cS}^\pi(\sfP; T) \coloneqq \mathbb{E}^\pi \left[ \sum\limits_{t=1}^T \Indlr{ f_{\pi^\ast_t}(X_t) \neq f_{\pi_t}(X_t)} \;\middle |\; \bar{\cA} \; \right]\vspace{-0.25cm}
	\]
	the inferior sampling rate of $\pi$ when the event $\cA$ fails, and let $\bar{\cS}^\pi_{\gamma,\alpha}(T) \coloneqq \sup\limits_{\sfP \in \cP(\gamma, \alpha, d)}\bar{\mathcal{S}}^\pi(\sfP;T) $. Note that\vspace{-0.25cm}
	\[
	\l( 1- \rPlr{\cA} \r)\bar{\cS}^\pi(\sfP; T)
	\le
	\cS^\pi(\sfP; T)
	\le
	\bar{\cS}^\pi(\sfP; T) + T \rPlr{\cA},\vspace{-0.3cm}
	\]
	which implies\vspace{-0.35cm}
	\begin{equation}\label{eq-relationship-S-and-S-bar-Lipschitz-smoother}
	\l| \bar{\cS}^\pi_{\gamma,\alpha}(T)
	-
	\cS^\pi_{\gamma,\alpha}(T)\r|
	\le c_2T^{-2}.\vspace{-0.15cm}
	\end{equation}
	For the rest of the proof probabilities and expectations will be computed conditional on the event~$\bar \cA$.\vspace{-0.2cm}

	\paragraph{Step 6 (Likelihood of distinguishing different smoothness parameters).} In this step, we will show that policy $\pi$ cannot distinguish between $\phi_0$ and $\phi_{1}$ with a strictly positive probability.
	For any set of samples $\l\{ (\pi_t, X_t, Y_{\pi_t,t})\r\}_{t=1}^T$, define the log-likelihood ratio $L_{T}=L_{T}\l( \l\{ (\pi_t, X_t, Y_{\pi_t,t})\r\}_{t=1}^T \r)$ as: \vspace{-0.35cm}
	\begin{align*}
	L_{T}
	&\coloneqq
	\sum_{t=1}^T \log \l(  \frac{\rP_{\pi, \phi_0}\l\{ Y_{\pi_t,t} \; \middle | \; \pi_t, X_t\r\}}{\rP_{\pi, \phi_1}\l\{ Y_{\pi_t,t} \; \middle | \; \pi_t, X_t\r\}} \r)
	\\
	&\le
	\sum_{t=1}^T \Indlr{\pi_t = 1, X_t \in H} \cdot \l[Y_{\pi_t,t} \log \l(  \frac{\phi_0(X_t)}{\phi_1(X_t)} \r) +  (1-Y_{\pi_t,t}) \log \l(  \frac{(1-\phi_0(X_t))}{(1-\phi_1(X_t))}\r)
	\r]
	\\
	&\le
	\sum_{t=1}^T \Indlr{\pi_t = 1, X_t \in H} \cdot \l[Y_{\pi_t,t}   \frac{(\phi_0(X_t)-\phi_1(X_t))}{\phi_m(X_t)}  +  (1-Y_{\pi_t,t}) \frac{(\phi_1(X_t)-\phi_0(X_t))}{(1-\phi_1(X_t))}
	\r]
	\\
	&=
	\sum_{t=1}^T \Indlr{\pi_t = 1, X_t \in H} \cdot  \frac{(Y_{\pi_t,t}-\phi_1(X_t) )(\phi_0(X_t)-\phi_1(X_t))}{\phi_1(X_t)(1-\phi_1(X_t))},
	\end{align*}
	where the last inequality follows from $\log(1+x) \le x$ for all $x>0$. Taking expectation conditional on the event $\bar{\cA}$, one obtains:\vspace{-0.35cm}
	\begin{align}\label{eq-expected-log-likelihood-upperbound-Lipschitz-smoother}
	\rE_{\pi,\phi_0}\l[L_{T}\; \middle| \; \bar{\cA}\; \r]
	&\le
	\rE_{\pi,\phi_0}\l[\sum_{t=1}^T \Indlr{\pi_t = 1, X_t \in \tilde{H}} \cdot  \frac{(Y_{\pi_t,t}-\phi_{1}(X_t) )(\phi_0(X_t)-\phi_{1}(X_t))}{\phi_{1}(X_t)(1-\phi_{1}(X_t))} \; \middle| \; \bar{\cA}\; \r]
	\nonumber
	\\
	&\le
	\rE_{\pi,\phi_0}\l[\sum_{t=1}^T \rElr{ \Indlr{\pi_t = 1, X_t \in \tilde{H}_{1}} \cdot  \frac{(Y_{\pi_t,t}-\phi_{1}(X_t) )(\phi_0(X_t)-\phi_{1}(X_t))}{\phi_{1}(X_t)(1-\phi_{1}(X_t))}\; \middle |\; X_t} \; \middle| \; \bar{\cA}\; \r]
	\nonumber
	\\
	&=
	\rE_{\pi,\phi_0}\l[\sum_{t=1}^T \Indlr{\pi_t = 1, X_t \in \tilde{H}} \cdot  \frac{(\phi_0(X_t)-\phi_{1}(X_t))^2}{\phi_{1}(X_t)(1-\phi_{1}(X_t))}\; \middle| \; \bar{\cA}\; \r]
	\nonumber
	\\
	&\overset{(a)}{\le}
	\frac{64C_{\phi}^2\Delta^2}{3}
	\rE_{\pi,\phi_0}\l[\sum_{t=1}^T \Indlr{\pi_t = 1, X_t \in \tilde{H}}\; \middle| \; \bar{\cA}\; \r]
	\nonumber
	\\
	& \overset{}{\le}
	 \frac{64C_{\phi}^2\Delta^2\cS^\ast_{\gamma,\alpha}(T)}{3} + c_2T^{-2}
	\overset{(b)}{\le}
	1,
	\end{align}
	for large enough $T$, where (a) follows from $C_{\phi}\Delta\le \frac{1}{4}$, and (b) follows from the definition of $\Delta$.
	
\paragraph{Step 7 (Lower bound on \oldnewam{regret as a function of inferior sampling rate}).}
Let\\ $\tilde{N}_{T} \coloneqq \sum_{t=1}^T \Indlr{\pi_t = 1, X_t \in \tilde{H}}$ be the number of times policy $\pi$ pulls arm 1 when covariates fall into the hypercube $\tilde{H}$. We will  show that with a strictly positive probability one has \oldnewam{$\tilde{N}_{T} < \frac{\underline{\rho}T\Delta}{15}$} conditional on the event $\bar{\cA}$. This will imply that policy $\pi$ makes at least \oldnewam{$\frac{\underline{\rho}T\Delta}{15}$} number of mistakes under the problem $m=1$, \oldnewam{where each mistake is associated with at least $\frac{\Delta}{2}$ instantaneous regret}.

Denote by $\rho_0$ and $\rho_{1}$ the distribution of $\tilde{N}_{T}$ under the problems $m=0$ and $m=1$ conditional on the event $\bar{\cA}$. Define the test function \oldnewam{$\Psi(x)=\Indlr{x \ge\frac{\underline{\rho}T\Delta}{15}}$}. With this choice of $\rho_0$, $\rho_1$, and $\Psi$, one can apply Lemma \ref{Tsybakov2008introduction-lemma-2.6} to obtain\vspace{-0.1cm}
\oldnewam{
	\[
	\rP_{\pi,\phi_0}\l\{\tilde{N}_{T} \ge \frac{\underline{\rho}T\Delta}{15}\; \middle| \; \bar{\cA}\; \r\}
	+
	\rP_{\pi,\phi_{1}}\l\{\tilde{N}_{T} < \frac{\underline{\rho}T\Delta}{15}\; \middle| \; \bar{\cA}\; \r\}
	\ge
	\frac{1}{2} \exp(-\mathrm{KL}(\rho_0,\rho_{1})).\vspace{-0.1cm}
	\]
}
	In order to lower bound the right hand side of this inequality, we note that\vspace{-0.2cm}
	\begin{align*}
	\rE_{\pi,\phi_0}\l[L_{T}\; \middle| \; \bar{\cA}\; \r]
	&=
	\sum_{s=1}^T
	\rE_{\pi,\phi_0}\l[L_{T}\; \middle| \; \bar{\cA},  \tilde{N}_{T}=s\r]\rP_{\pi,\phi_0}\l\{ \tilde{N}_{T}=s \; \middle| \;  \bar{\cA}\; \r\}
	\\
	&\ge
	\sum_{s=1}^T
	\log\l( \frac{\rP_{\pi,\phi_0}\l\{ \tilde{N}_{T}=s \; \middle| \;  \bar{\cA}\; \r\} }{\rP_{\pi,\phi_{1}}\l\{ \tilde{N}_{T}=s \; \middle| \;  \bar{\cA}\; \r\}}
	\r)\rP_{\pi,\phi_0}\l\{ \tilde{N}_{T}=s \; \middle| \;  \bar{\cA}\; \r\}
	=
	\mathrm{KL}(\rho_0,\rho_{1}),
	\end{align*}
	where the inequality follows from Lemma \ref{lemma-relationship-log-likelihood-events}. The last two displays along with \eqref{eq-expected-log-likelihood-upperbound-Lipschitz-smoother} yield\vspace{-0.1cm}
	\oldnewam{
	\[
	\rP_{\pi,\phi_0}\l\{\tilde{N}_{T} \ge \frac{\underline{\rho}T\Delta}{15}\; \middle| \; \bar{\cA}\; \r\}
	+
	\rP_{\pi,\phi_{1}}\l\{\tilde{N}_{T} < \frac{\underline{\rho}T\Delta}{15}\; \middle| \; \bar{\cA}\; \r\}
	\ge
	\frac{1}{2} \exp(-1).\vspace{-0.1cm}
	\]
}
	To show that \oldnewam{$\rP_{\pi,\phi_0}\l\{\tilde{N}_{T} \ge \frac{\underline{\rho}T\Delta}{15}\; \middle| \; \bar{\cA}\; \r\}$} is small, we apply Markov's inequality:\vspace{-0.2cm}
	\oldnewam{
	\begin{align*}
	\rP_{\pi,\phi_0}\l\{\tilde{N}_{T} \ge \frac{\underline{\rho}T\Delta}{15}\; \middle| \; \bar{\cA}\; \r\}
	&\le
	\frac{\rE_{\pi,\phi_0}\l[\tilde{N}_{T}  \; \middle| \; \bar{\cA}\; \r]}{\frac{\underline{\rho}T\Delta}{15}}
	\overset{(a)}{\le}
	\frac{\cS^\ast_{\gamma,\alpha}(T) + c_4T^{-2}}{\frac{\underline{\rho}T\Delta}{15}}
	\\
	&\overset{}{\le}
	\frac{c_3[\cS^\ast_{\gamma,\alpha}]^{\frac{1}{2}} + c_2l^{-d}T^{-2}}{\underline{\rho}T}
	\overset{}{\le}
	c_4
	T^{-\frac{1}{2}}
	\overset{}{\le}
	\frac{1}{4} \exp(-1)
	\end{align*}
}
\noindent 	for large enough $T$ and some constant $c_3,c_4>0$, where (a) follows from \eqref{eq-relationship-S-and-S-bar-Lipschitz-smoother}. The last two displays yield that for large enough $T$, one has\vspace{-0.1cm}
	\oldnewam{
	\[
	\rP_{\pi,\phi_{1}}\l\{\tilde{N}_{T} < \frac{\underline{\rho}T\Delta}{15}\; \middle| \; \bar{\cA}\; \r\}
	\ge
	\frac{1}{4e}.\vspace{-0.1cm}
	\]
}
	Note that by definition, when the event $\bar{\cA}$ holds, at least $\frac{2\underline{\rho}T\Delta}{3}$ number of covariates fall into the hypercube $\tilde{H}$, that is,
	\oldnewam{
	\begin{equation}\label{eq:inferior-sampling-rate-m*-Lipschitz-smoother}
\cR^\pi_{\beta,\alpha}(T)
	\ge
	\frac{\underline{\rho}T\Delta}{15} \cdot \frac{\Delta}{2} \cdot
	\rP_{\pi,\phi_{1}}\l\{\tilde{N}_{T} < \frac{\underline{\rho}T\Delta}{15} \; \middle| \; \bar{\cA}\; \r\}
	\rPlr{\bar{\cA}}
	\ge
	\frac{\underline{\rho}T\Delta^2}{240e}
	\ge
	c_5 T\l[ \cS^\ast_{\gamma,\alpha}(T) \r]^{-1},
	\end{equation}
}
	for some constant $c_5>0$. \oldnewam{The final result follows by noting that $\cS^\pi_{\gamma,\alpha}(T)
		\le
		\tilde{C}_{sr} T^{\frac{1}{\alpha+1}} \l[T^{\zeta(\gamma,\alpha,d)}\r]^{\frac{\alpha}{\alpha+1}}$. This concludes the proof.\hfill$\blacksquare$}

	\subsection{Proof of Theorem \ref{theo-reg-lower-ss} }\label{app:reg-lower-ss}
	The following lemma characterizes a general class of self-similar payoff functions for any non-integer smoothness parameter $\beta \in [\lbeta, \ubeta]$.
	\begin{lemma}\label{lemma:ss-construction-non-integer}
		Fix dimension $d$, some positive non-integer $\beta$ and some $\ubeta\ge \beta$. Consider some set of payoff functions $\{f_k\}_k$ such that $f_k\in \cH(\beta)$, $k\in \cK$.  Suppose $f_1(x) = a + bx_1^\beta$ for $x_1 \in [0,c]$ where $a,b$ and $0\le c\le 1$ are some constants. Then, the set of payoff functions $\l\{ f_k \r\}_k$ is self-similar as in Definition~\ref{def:global-ss} with some finite constants $l_0\ge 0$ and $b>0$.
	\end{lemma}
	\begin{proof}
		It suffices to show that for any non-negative integer $p$, one has
		\[
		\max_{\sfB \in \cB_l^{\oldnewam{q}}}\max_{k\in\cK}\sup_{x \in\sfB}\l|
		\bm{\Gamma}_{\textcolor{black}{q^{-l}}}^pf_k(x;\sfB) - f_k(x)
		\r|
		\ge
		b^\prime \textcolor{black}{q}^{-l\beta},
		\]
		for any $l\ge l_0=\lceil\log \frac{1}{c} \rceil $ and some $b^\prime>0$. Fix some $l>l_0$. Let $\sfB_0 \coloneqq [0, \textcolor{black}{q}^{-l}] ^ d$. One has
		\begin{equation}\label{eq:ss-construction-non-integer1}
		\max_{\sfB \in \cB_l^{\oldnewam{q}}}\max_{k\in\cK}\sup_{x \in\sfB}\l|
		\bm{\Gamma}_{\textcolor{black}{q^{-l}}}^pf_k(x;\sfB) - f_k(x)
		\r|
		\ge
		\l|
		\bm{\Gamma}_{\textcolor{black}{q^{-l}}}^pf_1(0;\sfB_0) - f_1(0)
		\r|
		=
		b\l|
		\bm{\Gamma}_{\textcolor{black}{q^{-l}}}^pg(0;\sfB_0)
		\r|,
		\end{equation}
		where $g(x) = x^\beta$.
		By Part \text{1} of Lemma \ref{lemma-L(P)-projection-closed-form}, one has
		$
		\bm{\Gamma}_{\textcolor{black}{q^{-l}}}^pg(0;\sfB_0) = e_1^\top
		B^{-1}
		W
		$, where
		\[
		e_1
		=
		\l( \Indlr{s=0} \r)_{s \in \{0,1,\dots,p\}},
		\quad
		B
		=
		\l( \frac{1}{s_1 + s_2 + 1} \r)_{s_1, s_2 \in \{0,1,\dots,p\}},
		\quad
		W=
		\l( \frac{\textcolor{black}{q}^{-l\beta}}{s+\beta+1} \r)_{s \in \{0,1,\dots,p\}}.
		\]
		By Cramer's rule for linear matrix equations, one has
		\begin{equation}\label{eq:ss-construction-non-integer2}
		\bm{\Gamma}_{\textcolor{black}{q^{-l}}}^pg(0;\sfB_0)
		=
		\frac{\det(B_1)}{\det (B)} \textcolor{black}{q}^{-l\beta},
		\end{equation}
		where
		\[
		B_1
		=
		\begin{pmatrix}
		\frac{1}{\beta + 1} & \frac{1}{2} & \frac{1}{3} & \dots & \frac{1}{p+1}\\
		\frac{1}{\beta + 2} & \frac{1}{3} & \frac{1}{4} & \dots & \frac{1}{p+2}\\
		\frac{1}{\beta + 3} & \frac{1}{4} & \frac{1}{5} & \dots & \frac{1}{p+3}\\
		\vdots & \vdots & \vdots & \dots & \vdots \\
		\frac{1}{\beta + p + 1} & \frac{1}{p+2} & \frac{1}{p+3} & \dots & \frac{1}{2p+1}
		\end{pmatrix}.
		\]
		Note that one can rewrite both matrices $B$ and $B_1$ as follows
		\begin{align*}
		B&=\l( \frac{1}{u_i+w_j} \r)_{1\le i,j\le p+1}, \quad u_i = i, w_j=j-1;
		\\
		B_1&=\l( \frac{1}{u_i^\prime+w_j^\prime} \r)_{1\le i,j\le p+1}, \quad u_i^\prime = i, w_j^\prime=\beta \Indlr{j=1} + (j -1)\Indlr{j>1}.
		\end{align*}
		The next theorem shows that the determinants of $B$ and $B_1$ are non-zero.
		\begin{theorem}[Cauchy double alternant determinant]
			For any set of indeterminates $\{u_i\}_{1\le i \le n}$ and $\{v_j\}_{1\le j \le n}$ such that $u_i + v_j \neq 0, \; \forall i,j \in \{1,\dots, n \}$, one has
			$$
			\det ~\left(\frac{1}{u_i+w_j}\right)_{1\le i,j\le n} = \frac{\prod_{1\leq i<j\leq n}(u_i-u_j)(w_i-w_j)}{\prod_{1\leq i\neq j\leq n}(u_i+w_j)}.
			$$
		\end{theorem}
		Hence, putting together \ref{eq:ss-construction-non-integer1} and \ref{eq:ss-construction-non-integer2} yields that for any integer $l>l_0$,
		\[
		\max_{\sfB \in \cB_l^{\oldnewam{q}}}\max_{k\in\cK}\sup_{x \in\sfB}\l|
		\bm{\Gamma}_{\textcolor{black}{q^{-l}}}^pf_k(x;\sfB) - f_k(x)
		\r|
		\ge
		b \frac{\det(B_1)}{\det (B)} \textcolor{black}{q}^{-l\beta}.
		\]
		This concludes the proof.
	\end{proof}
	
\noindent Using Lemma \ref{lemma:ss-construction-non-integer}, one can adjust the lower bound arguments in \cite{rigollet2010nonparametric} and \cite{hu2019smooth} in order to establish the same lower bounds for optimal regret when payoff functions are self-similar. We provide here the proof of the second part of the theorem; the proof of the first part is very similar, except for using Theorem 4.1 in \cite{rigollet2010nonparametric} instead of Theorem 3 in \cite{hu2019smooth}. First, we define the class of problems of interest.
		
		\begin{definition}
			For any $\beta\ge0$ and $\alpha\ge 0$, we denote by $\tilde \cP(\beta, \alpha, d) = \tilde\cP(\beta, L, \alpha, C_0, \underline{\rho}, \bar{\rho})$ the class of problems $\sfP=\l(\bm{\mathrm{P}}_X,\bm{\mathrm{P}}^{(1)}_{Y|X},\bm{\mathrm{P}}^{(2)}_{Y|X}\r)$ that satisfy Assumption~\ref{assumption-Holder-smoothness} for $\beta$ and $L>0$, Assumption~\ref{assumption-margin} for $\alpha$ and some $C_0>0$, and the following assumption regarding covariate distribution:  the covariate density $p_X$ has a compact support $\cX\subseteq [0,1]^d$ and $\underline{\rho} \le p_X(x) \le \urho$ for some $\urho\ge \underline{\rho} > 0$ and $x\in~\cX$. Furthermore, for any $\beta\in~[\lbeta, \ubeta]$,  $\alpha\ge 0$, $l_0\ge 0$, and $b>0$, we define by $\cP^{\mathrm{ss}}(\beta, \alpha, d, b, l_0) \subseteq \l\{ \sfP \in \cP(\beta, \alpha, d): \{f_k\}_{k\in\cK} \in \cF^{ss}(\beta, b, l_0)\r\}$ the corresponding class of problems with self-similar payoffs.
		\end{definition}
		
\noindent In their Theorem 3, \cite{hu2019smooth}  construct a problem instance $\sfP^\ast \in \tilde{\cP}(\beta, \alpha, d)$ such that
		$
		\mathcal{R}^\pi(\sfP^\ast;T) \ge~C T^{1-\frac{\beta(\alpha+1)}{2\beta + d}}
		$ for some constant $C>0$. Let $\{f^\ast_k\}_k$ be the set of payoff functions of  $\sfP^\ast$. Define the set of payoff functions~$\{f^{\ast \ast }_k\}_k$ such that
		\[
		f^{\ast \ast }_k(x)
		\coloneqq
		\begin{cases}
		\frac{1 + L_1x_1^\beta}{2} & \text{ if } 0 \le  x_1 \le \frac{1}{8},\\
		\frac{1+ L_1u(x_1) x_1^\beta}{2}  & \text{ if } \frac{1}{8} \le  x_1 \le \frac{1}{4},\\
		\frac{1}{2}  & \text{ if } \frac{1}{4} \le  x_1 \le \frac{1}{2},\\
		L_1f^{\ast }_k(g(x)) & \text{ if } \frac{1}{2} \le  x_1 \le 1,
		\end{cases}
		\]
		where $L_1>0$ is some constant and we define
		\begin{align*}
		g(x)\coloneqq \begin{pmatrix}
		2x_1 -1\\
		x_2
		\\
		x_3
		\\
		\vdots \\
		x_d
		\end{pmatrix},
		\quad
		u(x_1)
		\coloneqq
		\frac{
			\int_{x_1}^{\frac{1}{4}}
			\exp\l( \frac{-1}{\l|s - \frac{1}{8}\r|\l|s - \frac{1}{4}\r|} \r) ds
		}
		{
			\int_{\frac{1}{8}}^{\frac{1}{4}}
			\exp\l( \frac{-1}{\l|s - \frac{1}{8}\r|\l|s - \frac{1}{4}\r|} \r) ds
		}.
		\end{align*}
		Now, we show that $f^{\ast\ast}_k\in \cH(\beta)$, $k\in \cK$. Note that $u(x_1)$ is infinitely differentiable over $[\frac{1}{8}, \frac{1}{4}]$ and $x_1^\beta \in \cH(\beta)$. Hence, by the following lemma,  $u(x_1) x_1^\beta \in \cH_{[\frac{1}{8}, \frac{1}{4}]}(\beta)$.
		\begin{lemma}\label{lemma:product-Holder-smooth-funcs}
			Suppose $f,g\in \cH_{\cX}(\beta, L)$ for some $\cX \subseteq [0,1]$,  $\beta>0$, and $L>0$, and define the function $h \coloneqq f \cdot g$ as the product of $f$ and $g$. Then, $h \in \cH(\beta, L^\prime)$ for some $L^\prime > 0$.
		\end{lemma}
\noindent Furthermore, any derivative of $f^{\ast\ast}_k$ up to degree $\lfloor \beta \rfloor$ exists for $x_1\in\{\frac{1}{8}, \frac{1}{4}, \frac{1}{2}\}$. Hence, $f^{\ast\ast}_k \in \cH(\beta)$. One can also make $L_1>0$ small enough so that $f^{\ast\ast}_k \in \cH(\beta, L)$. Finally, by Lemma \ref{lemma:ss-construction-non-integer}, the set of payoff functions $\{f^{\ast \ast }_k\}_k$ is self-similar. Now, let $\sfP^{\ast \ast}$ be a problem instance that is the same as $\sfP^\ast$ except for its payoff functions that are $\{f^{\ast \ast }_k\}_k$. One can perform a similar analysis as in the proof of Theorem 3 in \cite{hu2019smooth} in order to show that $
		\mathcal{R}^\pi(\sfP^{\ast\ast};T) \ge
		C T^{1-\frac{\beta(\alpha+1)}{2\beta + d}}.
		$
		This concludes the proof.
$\blacksquare$

	\subsection{Proof of Proposition \ref{proposition-GSE-lower-bound-round}}
	Let $\tilde{r} \coloneqq \lfloor 2\log_{\textcolor{black}{q}}(\frac{\gamma}{2\constref{LPR-converge-to-true-value3}})+ 2l\beta +(\frac{d}{\underline{\beta}} + 1) \log_{\textcolor{black}{q}} \log T \rfloor$ where the constant $\constref{LPR-converge-to-true-value3}$ was introduced in  Proposition \ref{prop-LPR-converge-to-true-value}. We will prove the result by bounding the following probability
	\begin{align}\label{eq-GSE-r-lower-bound1}
	&\rPlr{
		\exists   r \le \tilde{r}:  \sup_{k\in \mathcal{K}, x\in \cM^{(\sfB)}} \l|
		\hat{f}_k^{(\sfB, r)}(x;j_1^{(\sfB)}) - \hat{f}_k^{(\sfB, r)}(x;j_{2}^{(\sfB)})
		\r|
		\ge
		\frac{\gamma \l( \log T\r)^{\frac{d}{2\underline{\beta}} + \frac{1}{2}} }{\textcolor{black}{q}^{r/2}}
	}
	\nonumber \\
	&\le
	\sum_{r \in [\tilde{r}]}
	\sum_{k\in \cK }
	\sum_{ x\in \cM^{(\sfB)}}\rPlr{
		\l|
		\hat{f}_k^{(\sfB, r)}(x;j_1^{(\sfB)}) - \hat{f}_k^{(\sfB, r)}(x;j_{2}^{(\sfB)})
		\r|
		\ge
		\frac{\gamma \l( \log T\r)^{\frac{d}{2\underline{\beta}} + \frac{1}{2}} }{\textcolor{black}{q}^{r/2}}
	}.
	\end{align}
	Note that by the triangle inequality,
	\begin{dmath*}
		\l|
		\hat{f}_k^{(\sfB, r)}(x;j_1^{(\sfB)}) - \hat{f}_k^{(\sfB, r)}(x;j_{2}^{(\sfB)})
		\r|
		\le
		\l|
		f_k(x) - \hat{f}_k^{(\sfB, r)}(x;j_1^{(\sfB)})
		\r|
		+
		\l|
		f_k(x) - \hat{f}_k^{(\sfB, r)}(x;j_{2}^{(\sfB)})
		\r|.
	\end{dmath*}
	That is,
	\begin{align}\label{eq-GSE-r-lower-bound2}
	&\rPlr{
		\l|
		\hat{f}_k^{(\sfB, r)}(x;j_1^{(\sfB)}) - \hat{f}_k^{(\sfB, r)}(x;j_{2}^{(\sfB)})
		\r|
		\ge
		\frac{\gamma \l( \log T\r)^{\frac{d}{2\underline{\beta}} + \frac{1}{2}} }{\textcolor{black}{q}^{r/2}}
	}
	\nonumber\\
	&\le
	\rPlr{
		\l|
		f_k(x) - \hat{f}_k^{(\sfB, r)}(x;j_1^{(\sfB)})
		\r|
		\ge
		\frac{\gamma \l( \log T\r)^{\frac{d}{2\underline{\beta}} + \frac{1}{2}} }{\textcolor{black}{q}^{1+r/2}}
	}
	\nonumber \\
	&+
	\rPlr{
		\l|
		f_k(x) - \hat{f}_k^{(\sfB, r)}(x;j_{2}^{(\sfB)})
		\r|
		\ge
		\frac{\gamma \l( \log T\r)^{\frac{d}{2\underline{\beta}} + \frac{1}{2}} }{\textcolor{black}{q}^{1+r/2}}
	}.
	\end{align}
	Note that since when $r \le \tilde{r}$ one has $\frac{\gamma \l( \log T\r)^{\frac{d}{2\underline{\beta}} + \frac{1}{2}} }{\textcolor{black}{q}^{1+r/2}} \ge \constref{LPR-converge-to-true-value3} \textcolor{black}{q}^{-\beta j_1^{\sfB}} \ge \constref{LPR-converge-to-true-value3} \textcolor{black}{q}^{-\beta j_{2}^{\sfB}}  $, one can apply Proposition \ref{prop-LPR-converge-to-true-value} to bound the two terms on the right hand side of above inequality. Namely, one can apply Proposition \ref{prop-LPR-converge-to-true-value} with $n=\textcolor{black}{q}^{r}$, $\underline{\mu} = \frac{\underline{\rho}}{\bar{\rho} \textcolor{black}{q}^{-dl}}$, $\bar{\mu} = \frac{\bar{\rho}}{\underline{\rho} \textcolor{black}{q}^{-dl}}$, $\delta=\frac{\gamma \l( \log T\r)^{\frac{d}{2\underline{\beta}} + \frac{1}{2}} }{\textcolor{black}{q}^{1+r/2}}$, and $h=\textcolor{black}{q}^{-j_1^{\sfB}}$ for the first term and $h=\textcolor{black}{q}^{-j_{2}^{\sfB}}$ for the second term to obtain
	\begin{align*}
	\rPlr{
		\l|
		f_k(x) - \hat{f}_k^{(\sfB, r)}(x;j_1^{(\sfB)})
		\r|
		\ge
		\frac{\gamma \l( \log T\r)^{\frac{d}{2\underline{\beta}} + \frac{1}{2}} }{\textcolor{black}{q}^{1+r/2}}
	}
	&\le
	\tilde{\constref{LPR-converge-to-true-value1}} T^{ -\gamma ^2 \tilde{\constref{LPR-converge-to-true-value2}} },
	\\
	\rPlr{
		\l|
		f_k(x) - \hat{f}_k^{(\sfB, r)}(x;j_{2}^{(\sfB)})
		\r|
		\ge
		\frac{\gamma \l( \log T\r)^{\frac{d}{2\underline{\beta}} + \frac{1}{2}} }{\textcolor{black}{q}^{1+r/2}}
	}
	&\le
	\tilde{\constref{LPR-converge-to-true-value1}} T^{ -\gamma ^2 \tilde{\constref{LPR-converge-to-true-value2}} },
	\end{align*}
	where the constants $\tilde{\constref{LPR-converge-to-true-value1}}, \tilde{\constref{LPR-converge-to-true-value2}}$ depend only on $L,  \underline{\rho}, \bar{\rho}$, and $d$. These two inequalities along with \eqref{eq-GSE-r-lower-bound1} and \eqref{eq-GSE-r-lower-bound2} imply
	\begin{align*}
	\rPlr{
		\exists   r  \le \tilde{r}:  \sup_{k\in \mathcal{K}, x\in \cM^{(\sfB)}} \l|
		\hat{f}_k^{(\sfB, r)}(x;j_1^{(\sfB)}) - \hat{f}_k^{(\sfB, r)}(x;j_{2}^{(\sfB)})
		\r|
		\ge
		\frac{\gamma \l( \log T\r)^{\frac{d}{2\underline{\beta}} + \frac{1}{2}} }{\textcolor{black}{q}^{r/2}}
	}
	&\le
	\textcolor{black}{q}^{1-ld}\l| \cM^{(\sfB)}\r|\tilde{r} \tilde{\constref{LPR-converge-to-true-value1}}  T^{ -\gamma ^2 \tilde{\constref{LPR-converge-to-true-value2}} }
	\\
	&\le
	\constref{GSE-lower-bound-round1}\textcolor{black}{q}^{-ld} \l( \log T \r)^{\frac{d}{\underline{\beta}}} T^{ -\gamma ^2 \constref{GSE-lower-bound-round2} + \constref{GSE-lower-bound-round3} },
	\end{align*}
	where the constants $\constref{GSE-lower-bound-round1}, \constref{GSE-lower-bound-round2}, \constref{GSE-lower-bound-round3}$ depend only on $\lbeta, \ubeta ,L,  \underline{\rho}, \bar{\rho}$, and $d$. The results follows by applying union bound over $\sfB\in \cB_l$.
	This concludes the proof. \hfill$\blacksquare$

	\subsection{Proof of Proposition \ref{proposition-GSE-upper-bound-round}}
	By Assumption \ref{assumption-global-self-similarity}, there exists at least one {bin} $\tilde{\sfB}\in \cB_l$, an arm $\tilde{k}\in \cK$, and a point $\hat{x}\in \tilde{\sfB}$ such that
	\begin{equation}\label{eq-GSE-r-upper-bound1}
	\l|
	\bm{\Gamma}_{j_1^{(\sfB)}}^{0}f_{{k}}(\hat{x};\tilde{\sfB}) - f_{{k}}(\hat{x})
	\r|
	=
	\l|
	\bm{\Gamma}_l^{\lfloor\ubeta\rfloor}f_{{k}}(\hat{x}) - f_{{k}}(\hat{x};\tilde{\sfB})
	\r|
	\ge
	b \textcolor{black}{q}^{-l\beta}.
	\end{equation}
	Let $\tilde{x} = \arg \min_{x \in \cM^{(\sfB)}} \|x-\hat{x} \|_\infty$ (if there is more than one minimizer we choose the one with the minimum $L_1$-norm). Note that $\|\tilde{x}-\hat{x} \|_\infty \le \textcolor{black}{q}^{-\tilde{l}}$, which along with the assumption $f_{\tilde{k}} \in \cH(\lbeta,L)$ implies that
	\begin{equation}\label{eq-GSE-r-upper-bound2}
	\l| f_{\tilde{k}}(\tilde{x}) - f_{\tilde{k}}(\hat{x}) \r|
	\le
	L \|\tilde{x}-\hat{x} \|_\infty^{\lbeta}
	\le
	L \textcolor{black}{q}^{-\tilde{l}\lbeta}
	\le
	\frac{L}{\log T}\textcolor{black}{q}^{-l\beta}.
	\end{equation}
	In addition, by Lemma \ref{lemma-L(P)-projection-closed-form}, one has
	\begin{equation}\label{eq-GSE-r-upper-bound3}
	\l|
	\bm{\Gamma}_{l}^{\lfloor\ubeta\rfloor}f_{{\tilde{k}}}(\hat{x}; \tilde{\sfB}) - \bm{\Gamma}_{l}^{\lfloor\ubeta\rfloor}f_{{\tilde{k}}}(\tilde{x}; \tilde{\sfB})
	\r|
	\le
	\kappa_0 \textcolor{black}{q}^l \|\hat{x}-\tilde{x}\|_\infty
	\le
	\kappa_0 \textcolor{black}{q}^{l-\tilde{l}}
	\le
	\frac{\kappa_0}{\log T}\textcolor{black}{q}^{-l\beta},
	\end{equation}
	where $\kappa_0$ was introduced in Lemma \ref{lemma-L(P)-projection-closed-form}. Let $\hat{r} \coloneqq  \lfloor2\log_{\textcolor{black}{q}}(\frac{4\gamma}{L \wedge \kappa_0})+ 2l\beta +(\frac{d}{\underline{\beta}} + 3) \log_{\textcolor{black}{q}} \log T \rfloor $. One has
	\begin{equation}\label{eq-GSE-r-upper-bound4}
	\rPlr{r^{(\tilde{\sfB})}_{\mathrm{last}} > \hat{r}}
	\le
	\rPlr{
		\l|
		\hat{f}_{\tilde{k}}^{(\tilde{\sfB}, \hat{r})}(\tilde{x};j_1^{(\tilde{\sfB})}) - \hat{f}_{\tilde{k}}^{(\tilde{\sfB}, \hat{r})}(\tilde{x};j_{2}^{(\tilde{\sfB})})
		\r|
		<
		\frac{\gamma \l( \log T\r)^{\frac{d}{2\underline{\beta}} + \frac{1}{2}} }{\textcolor{black}{q}^{\hat{r}/2}}
	}.
	\end{equation}
	Note that by the triangle inequality,
	\begin{dmath}\label{eq-GSE-r-upper-bound5}
		\l|
		\hat{f}_k^{(\tilde{\sfB}, \hat{r})}(\tilde{x};j_1^{(\tilde{\sfB})}) - \hat{f}_{\tilde{k}}^{(\tilde{\sfB}, \hat{r})}(\tilde{x};j_{2}^{(\tilde{\sfB})})
		\r|
		\ge
		\l|
		f_{\tilde{k}}(\tilde{x}) - \hat{f}_{\tilde{k}}^{(\tilde{\sfB}, \hat{r})}(\tilde{x};j_1^{(\tilde{\sfB})})
		\r|
		-
		\l|
		f_{\tilde{k}}(\tilde{x}) - \hat{f}_{\tilde{k}}^{(\tilde{\sfB}, \hat{r})}(\tilde{x};j_{2}^{(\tilde{\sfB})})
		\r|.
	\end{dmath}
	Note that since one has $\frac{\gamma \l( \log T\r)^{\frac{d}{2\underline{\beta}} + \frac{1}{2}} }{\textcolor{black}{q}^{1+\hat{r}/2}} \ge \constref{LPR-converge-to-true-value3} \textcolor{black}{q}^{-\beta j_{2}^{\sfB}} $, one can apply Proposition \ref{prop-LPR-converge-to-true-value} to show that second term on the right hand side of above inequality is ``small" with high probability. Namely, one can apply Proposition \ref{prop-LPR-converge-to-true-value} with $n=\textcolor{black}{q}^{\hat{r}}$, $\underline{\mu} = \frac{\underline{\rho}}{\bar{\rho} \textcolor{black}{q}^{-dl}}$, $\bar{\mu} = \frac{\bar{\rho}}{\underline{\rho} \textcolor{black}{q}^{-dl}}$, $\delta=\frac{\gamma \l( \log T\r)^{\frac{d}{2\underline{\beta}} + \frac{1}{2}} }{\textcolor{black}{q}^{1+\hat{r}/2}}$, and $h=\textcolor{black}{q}^{-j_{2}^{\sfB}}$ to obtain
	\begin{align}\label{eq-GSE-r-upper-bound6}
	\rPlr{
		\l|
		f_{\tilde{k}}(\tilde{x}) - \hat{f}_{\tilde{k}}^{(\tilde{\sfB}, \hat{r})}(\tilde{x};j_{2}^{(\tilde{\sfB})})
		\r|
		\ge
		\frac{\gamma \l( \log T\r)^{\frac{d}{2\underline{\beta}} + \frac{1}{2}} }{\textcolor{black}{q}^{1+\hat{r}/2}}
	}
	&\le
	\tilde{\constref{LPR-converge-to-true-value1}} T^{ -\gamma ^2 \tilde{\constref{LPR-converge-to-true-value2}} },
	\end{align}
	where the constants $\tilde{\constref{LPR-converge-to-true-value1}}, \tilde{\constref{LPR-converge-to-true-value2}}$ depend only on $\ubeta ,L,  \underline{\rho}, \bar{\rho}$, and $d$. Now, we show that the first term on the right hand side of \eqref{eq-GSE-r-upper-bound4} cannot get ``small" with high probability. One can write
	\begin{dmath}\label{eq-GSE-r-upper-bound7}
		\l|
		f_{\tilde{k}}(\tilde{x}) - \hat{f}_{\tilde{k}}^{(\tilde{\sfB}, \hat{r})}(\tilde{x};j_1^{(\tilde{\sfB})})
		\r|
		\ge
		\l|
		f_{{\tilde{k}}}(\tilde{x}) - \bm{\Gamma}_{j_1^{(\sfB)}}^{\lfloor\ubeta\rfloor}f_{{\tilde{k}}}(\tilde{x};\tilde{\sfB})
		\r|
		-
		\l|
		\bm{\Gamma}_{j_1^{(\sfB)}}^{\lfloor\ubeta\rfloor}f_{{\tilde{k}}}(\tilde{x};\tilde{\sfB})
		-
		\hat{f}_{\tilde{k}}^{(\tilde{\sfB}, \hat{r})}(\tilde{x};j_1^{(\tilde{\sfB})})
		\r|.
	\end{dmath}
	The first term corresponds to bias and the second term corresponds to stochastic error. Note that by \eqref{eq-GSE-r-upper-bound1}, \eqref{eq-GSE-r-upper-bound2}, and \eqref{eq-GSE-r-upper-bound3}, one has
	\begin{align}\label{eq-GSE-r-upper-bound8}
	\l|
	f_{\tilde{k}}(\tilde{x}) - \bm{\Gamma}_{j_1^{(\sfB)}}^{\lfloor\ubeta\rfloor}f_{\tilde{k}}(\tilde{x};\tilde{\sfB})
	\r|
	&\ge
	\l|
	f_{\tilde{k}}(\hat{x}) - \bm{\Gamma}_{j_1^{(\sfB)}}^{\lfloor\ubeta\rfloor}f_{\tilde{k}}(\hat{x};\tilde{\sfB})
	\r|
	-
	\l| f_{\tilde{k}}(\tilde{x}) - f_{\tilde{k}}(\hat{x}) \r|
	-
	\l|
	\bm{\Gamma}_{j_1^{(\sfB)}}^{\lfloor\ubeta\rfloor}f_{\tilde{k}}(\hat{x};\tilde{\sfB})  - \bm{\Gamma}_{j_1^{(\sfB)}}^{\lfloor\ubeta\rfloor}f_{\tilde{k}}(\tilde{x};\tilde{\sfB})
	\r|
	\nonumber
	\\
	&
	\ge
	b \textcolor{black}{q}^{-l\beta}
	-
	\frac{L}{\log T}\textcolor{black}{q}^{-l\beta}
	-
	\frac{\kappa_0}{\log T}\textcolor{black}{q}^{-l\beta}
	\ge
	\frac{ L\wedge \kappa_0}{2\log T}\textcolor{black}{q}^{-l\beta} \ge \frac{2\gamma \l( \log T\r)^{\frac{d}{2\underline{\beta}} + \frac{1}{2}} }{\textcolor{black}{q}^{\hat{r}/2}}
	\end{align}
	for large enough $T \ge T_0(L,b,\underline{\rho}, \bar{\rho}, d)$. In order to bound the second term on the right hand side of- \eqref{eq-GSE-r-upper-bound7}, we apply Proposition \ref{prop-LPR-converge-to-L(P)-projection}, with $n=\textcolor{black}{q}^{\hat{r}}$,  $\delta=\frac{\gamma \l( \log T\r)^{\frac{d}{2\underline{\beta}} + \frac{1}{2}} }{\textcolor{black}{q}^{1+\hat{r}/2}}$, and $h=\textcolor{black}{q}^{-j_1^{\sfB}}$ to obtain
	\begin{align}\label{eq-GSE-r-upper-bound9}
	\rPlr{
		\l|
		\bm{\Gamma}_{j_1^{(\sfB)}}^{\lfloor\ubeta\rfloor}f_{{\tilde{k}}}(\tilde{x};\tilde{\sfB})
		-
		\hat{f}_{\tilde{k}}^{(\tilde{\sfB}, \hat{r})}(\tilde{x};j_1^{(\tilde{\sfB})})
		\r|
		\ge
		\frac{\gamma \l( \log T\r)^{\frac{d}{2\underline{\beta}} + \frac{1}{2}} }{\textcolor{black}{q}^{1+\hat{r}/2}}
	}
	&\le
	\tilde{\constref{LPR-converge-to-L(P)-projection1} } T^{ -\gamma ^2 \tilde{\constref{LPR-converge-to-L(P)-projection2}} },
	\end{align}
	where the constants $\tilde{\constref{LPR-converge-to-L(P)-projection1}}, \tilde{\constref{LPR-converge-to-L(P)-projection2}}$ depend only on $\ubeta ,L,  \underline{\rho}, \bar{\rho}$, and $d$.
	Putting together \eqref{eq-GSE-r-upper-bound4}, \eqref{eq-GSE-r-upper-bound5}, \eqref{eq-GSE-r-upper-bound7}, and \eqref{eq-GSE-r-upper-bound8}, one obtains
	\begin{align*}
	\rPlr{r^{(\tilde{\sfB})}_{\mathrm{last}} > \hat{r}}
	&\le
	\rPlr{
		\l|
		\bm{\Gamma}_{j_1^{(\sfB)}}^{\lfloor\ubeta\rfloor}f_{{k}}(\tilde{x};\tilde{\sfB})
		-
		\hat{f}_k^{(\tilde{\sfB}, \hat{r})}(\tilde{x};j_1^{(\tilde{\sfB})})
		\r|
		\ge
		\frac{\gamma \l( \log T\r)^{\frac{d}{2\underline{\beta}} + \frac{1}{2}} }{\textcolor{black}{q}^{1+\hat{r}/2}}
	}
	\\
	&+
	\rPlr{
		\l|
		\bm{\Gamma}_{j_1^{(\sfB)}}^{\lfloor\ubeta\rfloor}f_{{k}}(\tilde{x};\tilde{\sfB})
		-
		\hat{f}_k^{(\tilde{\sfB}, \hat{r})}(\tilde{x};j_1^{(\tilde{\sfB})})
		\r|
		\ge
		\frac{\gamma \l( \log T\r)^{\frac{d}{2\underline{\beta}} + \frac{1}{2}} }{\textcolor{black}{q}^{1+\hat{r}/2}}
	}
	\\
	&\le
	\constref{GSE-upper-bound-round1} T^{ -\gamma ^2 \constref{GSE-upper-bound-round2} },
	\end{align*}
	where the last inequality follows from \eqref{eq-GSE-r-upper-bound6} and \eqref{eq-GSE-r-upper-bound9}, and the constants $ \constref{GSE-upper-bound-round1},  \constref{GSE-upper-bound-round2}$ depend only on $\ubeta ,L,  \underline{\rho}, \bar{\rho}$, and $d$.
	This concludes the proof. \hfill$\blacksquare$

	\subsection{Proof of Theorem \ref{theorem-GSE-smoothness-accuracy}}
	Note that for large enough $T$, one has
	\begin{align*}
	&\rPlr{\hat{\beta}_{\texttt{SACB}} \in [\beta- \frac{3(2\overline{ \beta} + d)^2\log_{\textcolor{black}{q}}\log T}{(\underline{\beta}+d-1)\log_{\textcolor{black}{q}} T}   ,\beta]}
	\\
	&\quad
	\le
	\rPlr{
		2l\beta +(\frac{d}{\underline{\beta}} + 1) \log_{\textcolor{black}{q}} \log T
		\le
		r^{(\sfB)}_{\mathrm{last}} \le
		2l\beta +(\frac{d}{\underline{\beta}} + 4) \log_{\textcolor{black}{q}} \log T
	}
	\\
	& \quad \le
	1- \constref{GSE-lower-bound-round1} \textcolor{black}{q}^{ld}  \l( \log T \r)^{\frac{d}{\underline{\beta}}} T^{ -\gamma ^2 \constref{GSE-lower-bound-round2} + \constref{GSE-lower-bound-round3} }
	-
	\constref{GSE-upper-bound-round1} T^{ -\gamma ^2 \constref{GSE-upper-bound-round2}},
	\end{align*}
	where the last inequality follows from Propositions \ref{proposition-GSE-lower-bound-round}, and \ref{proposition-GSE-upper-bound-round}, and the constants $\constref{GSE-lower-bound-round1}, \constref{GSE-lower-bound-round2}, \constref{GSE-lower-bound-round3}>0$ were introduced in  Proposition \ref{proposition-GSE-lower-bound-round}, and the constants $\constref{GSE-upper-bound-round1},\constref{GSE-upper-bound-round2}>0$ were introduced in Proposition \ref{proposition-GSE-upper-bound-round}.
	
Next, we show that with high probability, $T_{\texttt{SACB}} \le \frac{4}{\underline{\rho}} \l( \log T \r)^{\frac{2d}{\underline{\beta}}+4}T^{\frac{(\underline{\beta}+d-1)}{(2\overline{ \beta}+d)}}\eqqcolon \bar{T}_{\texttt{SACB}} $. Note that the smoothness estimation sub-routine terminates when all the {bins} $\sfB \in \cB_l$ have reached round $\bar{r}
	=
	\lceil 2l\ubeta +(\frac{2d}{\underline{\beta}} + 4) \log_{\textcolor{black}{q}} \log T
	\rceil$. That is, $T_{\texttt{SACB}}$ is less than the time step by which $2\sum_{r=\underline{r}}^{\bar{r}} \textcolor{black}{q}^r$ covariates have realized in each $\sfB \in
	\cB_l$. Note that\vspace{-0.1cm}
	\[
	\sum_{r=\underline{r}}^{\bar{r}} \textcolor{black}{q}^r \le \textcolor{black}{q}^{\bar{r}+1} \le
	2 \l( \log T \r)^{\frac{2d}{\underline{\beta}}+4}T^{\frac{2\overline{ \beta}(\underline{\beta}+d-1)}{(2\overline{ \beta}+d)^2}}.\vspace{-0.1cm}
	\]
Let $\bar{N}^{(\sfB)}\coloneqq \sum_{t=1}^{\bar{T}_{\texttt{SACB}}} Z_t$ be the number of covariates that have realized in $
	\sfB$ by $t=\bar{T}_{\texttt{SACB}}$, where $Z_t$'s are i.i.d. Bernoulli random variables with $\rElr{Z_t} \ge \underline{\rho} T^{-\frac{d(\underline{\beta}+d-1)}{(2\overline{ \beta}+d)^2}}$ and $\mathrm{Var}(Z_t) \le \rElr{Z_t^2} \le \bar{\rho} T^{-\frac{d(\underline{\beta}+d-1)}{(2\overline{ \beta}+d)^2}}$. Applying Bernstein's inequality in Lemma \ref{lemma-Bernstein-inequality} to $\bar{N}^{(\sfB)}$ with $a=2 \l( \log T \r)^{\frac{2d}{\underline{\beta}}+4}T^{\frac{2\overline{ \beta}(\underline{\beta}+d-1)}{(2\overline{ \beta}+d)^2}}$ yields:\vspace{-0.2cm}
	\begin{align*}
	\rPlr{\bar{N}^{(\sfB)} < 2 \l( \log T \r)^{\frac{2d}{\underline{\beta}}+4}T^{\frac{2\overline{ \beta}(\underline{\beta}+d-1)}{(2\overline{ \beta}+d)^2}}}
	&\le
	\exp\l( -\frac{a^2}{2\bar{T}_{\texttt{SACB}}\mathrm{Var}(Z_t)  + a} \r)
	\\
	&\le
	\exp\l( -\frac{\underline{\rho}}{4\bar{\rho}+2\underline{\rho}} \l( \log T \r)^{\frac{2d}{\underline{\beta}}+4}T^{\frac{2\overline{ \beta}(\underline{\beta}+d-1)}{(2\overline{ \beta}+d)^2}} \r),
	\end{align*}
	and, by the union bound:\vspace{-0.2cm}
	\begin{align*}
	\rPlr{T_{\texttt{SACB}} > \bar{T}_{\texttt{SACB}}}
	&	\le
	\sum_{\sfB\in\cB_l} \rPlr{\bar{N}^{(\sfB)} < 2 \l( \log T \r)^{\frac{2d}{\underline{\beta}}+4}T^{\frac{2\overline{ \beta}(\underline{\beta}+d-1)}{(2\overline{ \beta}+d)^2}}}
	\\
	&\le
	2 T^{\frac{d(\underline{\beta}+d-1)}{(2\overline{ \beta}+d)^2}}
	\exp\l( -\frac{\underline{\rho}}{4\bar{\rho}+2\underline{\rho}} \l( \log T \r)^{\frac{2d}{\underline{\beta}}+4}T^{\frac{2\overline{ \beta}(\underline{\beta}+d-1)}{(2\overline{ \beta}+d)^2}} \r).
	\end{align*} This concludes the proof. \hfill$\blacksquare$
	\subsection{Proof of Theorem \ref{theorem-GSE-ABSE-regret-bound}}
	 The regret incurred by the \texttt{SACB} policy up to $t = \lfloor\bar{T}_{\texttt{SACB}}\rfloor$ is bounded by
	\begin{align}\label{eq-reg-upper-estimation}
	\mathbb{E}^\pi \left[ \sum\limits_{t=1}^{\lfloor\bar{T}_{\texttt{SACB}}\rfloor} f_{\pi^\ast_t}(X_t) - f_{\pi_t}(X_t) \right]
	&\le
	T\cdot \rPlr{T_{\texttt{SACB}} > \bar{T}_{\texttt{SACB}}} + \bar{T}_{\texttt{SACB}}
	\nonumber \\
	&\overset{(a)}{\le}
	2 T^{1+\frac{d(\underline{\beta}+d-1)}{(2\overline{ \beta}+d)^2}}
	\exp\l( -\frac{\underline{\rho}}{4\bar{\rho}+2\underline{\rho}} \l( \log T \r)^{\frac{2d}{\underline{\beta}}+4}T^{\frac{2\overline{ \beta}(\underline{\beta}+d-1)}{(2\overline{ \beta}+d)^2}} \r)
	\nonumber \\
	&\quad + \frac{4}{\underline{\rho}} \l( \log T \r)^{\frac{2d}{\underline{\beta}}+4}T^{\frac{(\underline{\beta}+d-1)}{(2\overline{ \beta}+d)}}
	\nonumber \\
	&\overset{{(b)}}{=}
	o\l( T^{1-\frac{\beta (\alpha+1)}{2\beta + d}} \r),
	\end{align}
	where (a) follows from Theorem \ref{theorem-GSE-smoothness-accuracy} and (b) holds by $\frac{(\underline{\beta}+d-1)}{(2\overline{ \beta}+d)} \le 1-\frac{\beta (\alpha+1)}{2\beta + d} $ for any $\lbeta \le \beta \le \ubeta$ and $\alpha \le \frac{1}{\min\{1,\beta \} }$.
	Define $\hat{\beta}_T \coloneqq \beta-\frac{3(2\overline{ \beta} + d)^2\log_{\textcolor{black}{q}}\log T}{(\underline{\beta}+d-1)\log_{\textcolor{black}{q}} T} $.
	The regret from $t=\lfloor\bar{T}_{\texttt{SACB}}\rfloor+1$ to $t=T$  is bounded by\vspace{-0.2cm}
	\begin{align}\label{eq-reg-upper-opt-policy}
	\mathbb{E}^\pi \left[ \sum\limits_{t=\lfloor\bar{T}_{\texttt{SACB}}\rfloor+1}^T f_{\pi^\ast_t}(X_t) - f_{\pi_t}(X_t) \right]
	&\le
	T \cdot \rPlr{\hat{\beta}_{\texttt{SACB}} \not\in [\hat{\beta}_T ,\beta]} + \bar{C}_0\l(\log T\r)^{ \iota_0(\hat{\beta}_T, \alpha, d)}T^{1-\frac{\hat{\beta}_T (\alpha+1)}{2\hat{\beta}_T + d}}
	\nonumber \\
	&\le
	\constref{GSE-smoothness-accuracy1} \l( \log T \r)^{\frac{d}{\underline{\beta}}} T^{ -\gamma ^2 \constref{GSE-smoothness-accuracy2}+ \constref{GSE-smoothness-accuracy3} +1 + \frac{d(\underline{\beta}+d-1)}{(2\overline{ \beta}+d)^2}}
	\nonumber \\
	&\quad
	+ C T^{1-\frac{\beta (\alpha+1)}{2\beta + d}} \l(\log T\r)^{\frac{3d(\alpha+1)(2\overline{ \beta} + d)^2}{(2\beta+d)(\beta+d)(\underline{\beta}+d-1)} + \iota_0(\hat{\beta}_T, \alpha, d)},
	\end{align}
	for some constant $C>0$, where the last inequality follows from Corollary \ref{theorem-GSE-smoothness-accuracy} and the constants $\constref{GSE-smoothness-accuracy1}, \constref{GSE-smoothness-accuracy2}$, and $\constref{GSE-smoothness-accuracy3}$ were introduced in Theorem \ref{theorem-GSE-smoothness-accuracy}. Putting together \eqref{eq-reg-upper-estimation} and \eqref{eq-reg-upper-opt-policy} concludes the proof.
	\hfill$\blacksquare$

\subsection{Proof of Corollary \ref{corllary:global-adapt-beta<1}}\vspace{-0.1cm}
The result follows from Theorem \ref{theorem-GSE-ABSE-regret-bound} and the fact that for any $\beta_0\le 1$:\vspace{-0.2cm}
\[
\sup\limits_{\sfP \in \cP(\beta_0, \alpha,d)}\mathcal{R}^{\texttt{ABSE}(\beta_0)}(\sfP;T)  = \cO\left(T^{\zeta(\beta_0, \alpha,d)}\right).\vspace{-0.1cm}
\]
\hfill\hfill$\blacksquare$
\vspace{-1cm}

\subsection{Proof of Corollary \ref{corllary:global-adapt-beta>1}}\vspace{-0.1cm}
The result follows from Theorem \ref{theorem-GSE-ABSE-regret-bound} and since for any problem instance $\sfP \in \cP(\beta_0, \alpha,d)$, and decision regions which satisfy the regularity condition in Assumption \ref{assumption-regularity}, one has for any $\beta_0 \ge 1$:\vspace{-0.2cm}
\[
\mathcal{R}^{\texttt{SmoothBandit}(\beta_0)}(\sfP;T)  = \cO\l((\log T)^{\frac{2\beta_0+d}{2\beta_0}} T^{\zeta(\beta_0, \alpha,d)}\r).\vspace{-0.1cm}
\]
\hfill\hfill$\blacksquare$
\vspace{-1cm}

\subsection{Proof of Remark \ref{remark:beta<1-and-beta>1}}\vspace{-0.1cm}
Note that\vspace{-0.1cm}
\[
\pi_0(\beta_0)
=
\begin{cases}
\texttt{ABSE}(\beta_0) & \text{ if } \beta_0\le 1;
\\
\texttt{SmoothBandit}(\beta_0) & \text{ if } \beta_0> 1.
\end{cases}
\]
Furthermore, for any $\beta_0\le 1$\vspace{-0.2cm}
\[
\sup\limits_{\sfP \in \cP(\beta_0, \alpha,d)}\mathcal{R}^{\texttt{ABSE}(\beta_0)}(\sfP;T)  = \cO\left(T^{\zeta(\beta_0, \alpha,d)}\right),\vspace{-0.1cm}
\]
and for any $\beta_0 > 1 $ and any problem instance $\sfP \in \cP(\beta_0, \alpha,d)$, and decision regions which satisfy the regularity condition in Assumption \ref{assumption-regularity},\vspace{-0.1cm}
\[
\mathcal{R}^{\texttt{SmoothBandit}(\beta_0)}(\sfP;T)  = \cO\l((\log T)^{\frac{2\beta_0+d}{2\beta_0}} T^{\zeta(\beta_0, \alpha,d)}\r).\vspace{-0.1cm}
\]
The result follows from applying Theorem \ref{theorem-GSE-ABSE-regret-bound} with\vspace{-0.1cm}
\[
\iota_0(\beta_0, \alpha, d)
\coloneqq
\begin{cases}
0 & \text{ if } \beta_0 \le 1;
\\
\frac{2\beta_0}{2\beta_0 + d}
& \text{ o.w.}
\end{cases}\vspace{-0.1cm}
\]
\hfill\hfill$\blacksquare$
\vspace{-1cm}

	\section{{Properties of the $L_2(P_X)$-projection}}
	\begin{lemma}\label{lemma-L(P)-projection-closed-form}
		Fix non-negative integers $l$ and $p$, a hypercube $U$ of side-length $\textcolor{black}{q}^{-l^\prime}, l^\prime\in\rR_+$, and some point $x\in U$ and let $K(\cdot)=\Indlr{\|\cdot\|_\infty\le 1}$ and $h=\textcolor{black}{q}^{-l}$. Let $\mu_0$, $\kappa_0$, and $L_0$ be some  constants that only depend on $p, \underline{\rho}, \bar{\rho}$ (introduced in Assumption \ref{assumption-covar-dist}), and $d$. The following statements hold:
		\begin{enumerate}
			\item $
			\bm{\Gamma}^p_h f (x;U)
			=
			R^\top(0) B^{-1} W,
			$
			where we
			define the vector $R(u)\coloneqq \l( u^s \r)_{|s|\le p}$, the matrix $B \coloneqq \l( B_{s_1,s_2}\r)_{|s_1|,|s_2|\le p}$, and the vector $W \coloneqq \l(W_{s}\r)_{|s|\le p}$ with elements
			\[
			B_{s_1,s_2} \coloneqq \int_{\rR^d} u^{s_1+s_2}K(u)p_X(x+hu \mid U)du,
			\qquad
			W_s \coloneqq \int_{\rR^d}  u^{s}f(x+hu)K(u)p_X(x+hu \mid U)du;
			\]
			\item $
			\lambda_{\min}(B) \ge \mu_0 \textcolor{black}{q}^{dl^\prime};
			$
			\item$
			\l|
			\bm{\Gamma}^p_h f (x;U)
			-
			\bm{\Gamma}^p_h f (\hat{x};U)
			\r|
			\le
			\kappa_0 h^{-1} \|\hat{x}-x\|_\infty$ for all $ x,\hat{x} \in U;$
			\item If $f\in \cH(\beta,L)$ for $0<\beta\le p+1$ then, $
			\l|
			\bm{\Gamma}^p_h f (x;U)
			-
			f(x)
			\r|
			\le
			L_0 h^{\beta}
			$
			for all $x\in U$.
		\end{enumerate}
	\end{lemma}
	
	\begin{proof}
		Fix some $x\in U$. Let $\tilde{\theta}(u;p,l,U) \coloneqq \sum_{|s|\le p} \xi_s u^s$ be a polynomial of degree $p$ on $\rR^d$ that minimizes
		\begin{align*}
		&\int_{U} \l| f(u)-\tilde{\theta}\l(\frac{u-x}{h};p,l,U\r) \r|^2K\l( \frac{u-x}{h} \r)p_X(u \mid U) du
		=
		\int_{U}  f^2(u)K\l( \frac{u-x}{h} \r) p_X(u \mid U) du
		\\
		&+ \sum_{|s_1|, |s_2|\le p} \xi_{s_1} \xi_{s_2} \int_{U} \l(\frac{u-x}{h}\r)^{s_1+s_2} K\l( \frac{u-x}{h} \r)p_X(u \mid U) du
		\\
		&- 2  \sum_{|s|\le p} \xi_{s} \int_{U} f(u)\l(\frac{u-x}{h}\r)^{s} K\l( \frac{u-x}{h} \r) p_X(u \mid U) du,
		\end{align*}
		where $h = \textcolor{black}{q}^{-l}$.
		Equivalently, $\tilde{\theta}(u;p,l,U)$ can be characterized by its vector of coefficients $\bxi$ that minimizes
		\begin{equation}\label{eq-L(P)-projection-closed-form1}
		\sum_{|s_1|, |s_2|\le p} \xi_{s_1} \xi_{s_2} \int_{\rR^d} u^{s_1+s_2} K(u) p_X(x+hu \mid U) du
		- 2  \sum_{|s|\le p} \xi_{s} \int_{\rR^d} f(u)u^{s} K(u) p_X(x+hu \mid U) du =
		\bxi^\top B \bxi - 2 W^\top \bxi,
		\end{equation}
		where we define the matrix $B\coloneqq \l( B_{s_1,s_2} \r)_{|s_1|,|s_2|\le p}$ and the vector $W\coloneqq \l( W_{s} \r)_{|s|\le p}$ with elements
		\[
		B_{s_1,s_2}
		\coloneqq
		\int_{\rR^d} u^{s_1+s_2} K(u)p_X(x+hu \mid U) du,
		\qquad
		W_{s}
		\coloneqq
		\int_{\rR^d} f(u)u^{s}K(u) p_X(x+hu \mid U) du.
		\]
		Note that if $B$ is a positive definite matrix then, the minimizer of \eqref{eq-L(P)-projection-closed-form1} is $\bxi = B^{-1}W$, which implies the desired result: $\bm{\Gamma}^p_h f (x;U)
		=
		R^\top(0) B^{-1} W$. In order to show that this is indeed the case, we note that
		\begin{align*}
		\lambda_{\min}(B)
		&=
		\min_{\|Z\|=1}Z^\top B Z
		=
		\min_{\|Z\|=1}
		\int_{\rR^d} \l( \sum_{|s|\le p} Z_s u^s \r)^2 K(u) p_X(x+hu \mid U) du
		\ge
		\min_{\|Z\|=1}
		\frac{\underline{\rho} 2^{dl^{\prime}}}{ \bar{\rho}} \int_{A} \l( \sum_{|s|\le p} Z_s u^s \r)^2  du,
		\end{align*}
		where $A=\l\{ u \in \rR^d:\|u\|_\infty\le 1;  x+hu \in U \r\}$. Note that
		\[
		\lambda[A] \ge h^{-d} \lambda\l[ \Xi(x,h) \cap U \r] \ge  \textcolor{black}{q}^{-d}h^{-d} \lambda\l[ \Xi(x,h) \r] = \textcolor{black}{q}^{-d} \lambda\l[ \Xi(0,1)  \r].
		\]
		Let $\cA$ denote the class of compact subsets of $\Xi(0,1)$ having the Lebesgue measure $\textcolor{black}{q}^{-d} \lambda\l[ \Xi(0,1)  \r]$. Using the previous display, we obtain
		\begin{equation}\label{eq-L(P)-projection-closed-form2}
		\lambda_{\min}(B)
		\ge
		\frac{\underline{\rho}\textcolor{black}{q}^{dl^{\prime}}}{ \bar{\rho}} \min_{\|Z\|\le 1; S \in \cA} \int_{S} \l( \sum_{|s|\le p} Z_s u^s \r)^2  du\eqqcolon
		\frac{\underline{\rho}\textcolor{black}{q}^{dl^{\prime}}}{ \bar{\rho}} \tilde{\mu}_0.
		\end{equation}
		By the compactness argument, the minimum in the above expression exists, and is strictly positive.
		
		In order to prove the last claim in the lemma, note that for any $\hat{x} \in U$,
		\begin{align*}
		\l|
		\bm{\Gamma}^p_h f (x;U)
		-
		\bm{\Gamma}^p_h f (\hat{x};U)
		\r|
		&=
		\l|
		\tilde{\theta}(0;p,l,U)
		-
		\tilde{\theta}\l(\frac{\hat{x}-x}{h};p,l,U\r)
		\r|
		\\
		&=
		\l|
		\sum_{|s|\le p, s\neq (0,\dots,0)} \xi_s \l(\frac{\hat{x}-x}{h}\r)^s
		\r|
		\le
		M h^{-1} \|\hat{x}-x\|_\infty \|\bxi\|.
		\end{align*}
		Also, by \eqref{eq-L(P)-projection-closed-form2}, one has
		\[
		\|\bxi\|
		\le
		\l\|B^{-1}W\r\|
		\le
		\frac{\textcolor{black}{q}^{-dl^{\prime}} \bar{\rho}}{\underline{\rho}} \tilde{\mu}_0^{-1} M^{\frac{1}{2}} \max_{s}|W_s|,
		\]
		and
		\[
		|W_s|
		=
		\l|
		\int_{\rR^d}  u^{s}f(x+hu)K(u)p_X(x+hu|U)du
		\r|
		\le
		\int_{\rR^d}  K(u) p_X(x+hu|U)du
		\le
		\textcolor{black}{q}^{dl^\prime}.
		\]
		Putting together the above three displays, one obtains
		\[
		\l|
		\bm{\Gamma}^p_h f (x;U)
		-
		\bm{\Gamma}^p_h f (\hat{x};U)
		\r|
		\le
		\bar{\rho} \underline{\rho}^{-1} \tilde{\mu}_0^{-1} M^{3/2} h^{-1} \|\hat{x}-x\|_\infty.
		\]
		To prove the last part, define the vector $Z\coloneqq \l( Z_s \r)_{|s|\le p}$ with elements
		\[
		Z_s \coloneqq \frac{h^{|s|}f^{(s)}(x)}{s!} \cdot \Indlr{|s|\le \lfloor \beta \rfloor}.
		\]
		Note that
		\[
		f(x)
		=
		R^\top(0) B^{-1}B Z.
		\]
		As a result, one has
		\[
		\l|
		f(x)
		-
		\bm{\Gamma}^p_h f (x;U)
		\r|
		=
		\l|
		R^\top(0) B^{-1}\l(B Z - W\r)
		\r|
		\le
		\|B^{-1} \|
		\cdot
		\l\|
		B Z - W
		\r\|
		\le
		\frac{\textcolor{black}{q}^{-dl^\prime} \bar{\rho}}  {\underline{\rho}\tilde{\mu}_0} M^{\frac{1}{2}} \max_{s}|(B Z)_s - W_s|,
		\]
		where the last inequality follows from \eqref{eq-L(P)-projection-closed-form2}. Furthermore, one has
		\begin{align*}
		|(B Z)_s - W_s|
		&=
		\l|  \int_{\rR^d}  u^{s}\l(\sum_{|s^\prime| \le \lfloor \beta \rfloor}\frac{(hu)^{s}f^{(s)}(x)}{s!} - f(x+hu)\r)K(u)p_X(x+hu \mid U)du \r|
		\\
		&\le
		\int_{\rR^d}  \l|u^{s}\r| \l|\sum_{|s^\prime| \le \lfloor \beta \rfloor}\frac{(hu)^{s}f^{(s)}(x)}{s!} - f(x+hu) \r|K(u)p_X(x+hu \mid U)du
		\\
		&\le
		\int_{\rR^d}  L h^\beta p_X(x+hu \mid U)du
		=
		Lh^{\beta}\textcolor{black}{q}^{dl^\prime}\,
		\end{align*}
		where the last inequality follows from the assumption that $f \in \cH (\beta, L)$. Putting the last two displays together, the result follows.
		This concludes the proof.
	\end{proof}

	\section{Proofs and analysis for the review of local polynomial regression}\label{appendix-local-poly}
	In this section of the appendix, we provide the proofs for our review of the local polynomial regression estimation method. Fix a set of pairs $\cD = \l\{ (X_i,Y_i) \r\}_{i=1}^n$,a point $x\in \rR^d$,  a bandwidth $h>0$, an integer $p>0$ and a kernel function $K:\rR^d \rightarrow \rR_+$. Define the matrix $Q \coloneqq \l( Q_{s_1,s_2}\r)_{|s_1|,|s_2|\le p}$ and the vector $V \coloneqq \l(V_{s}\r)_{|s|\le p}$ with the elements
	\[
	Q_{s_1,s_2} \coloneqq \sum_{i=1}^n (X_i-x)^{s_1+s_2}K\l( \frac{X_i-x}{h} \r),
	\qquad
	V_s \coloneqq  \sum_{i=1}^n Y_i(X_i-x)^{s}K\l( \frac{X_i-x}{h} \r).
	\]
	Also, define the matrix $U \coloneqq \l( u^s \r)_{|s|\le p}$. The next result from \cite{audibert2007fast} provides a closed-form expression for local polynomial regression at any arbitrary point.
	\begin{lemma}[\protect{\citealt[Proposition 2.1]{audibert2007fast}}]\label{lemma-LPR-closed-form}
		If the matrix $Q$ is positive definite, there exists a polynomial on $\rR^d$ of degree $p$ minimizing \eqref{LPR-minimization-problem}. Its vector of coefficients is given by $\bxi = Q^{-1}V$ and the corresponding local polynomial regression function at point $x$ is given by
		\[
		\hat{\eta}^{\mathrm{LP}}(x; \cD, h, p)
		=
		U(0)^\top Q^{-1} V = \sum_{i=1}^nY_i K\l( \frac{X_i-x}{h} \r) U(0)^\top Q^{-1} U(X_i-x).
		\]
	\end{lemma}

\noindent The following simple extension of Theorem 3.2 in \cite{audibert2007fast} will be one of the main tools to bound our estimation error in our proposed policy.
	\begin{prop}\label{prop-LPR-converge-to-true-value}
		Let $\cD = \l\{ (X_i,Y_i) \r\}_{i=1}^n$ be a set of $n$ i.i.d pairs $(X_i,Y_i)\in \cX\times \rR$. If the marginal density $\mu$ of $X_i$'s satisfies $\underline{\mu} \le \mu(x)\le \bar{\mu}$ for some $0<\underline{\mu}\le \bar{\mu}$ with a support $\cX$  that is a closed hypercube in $\rR^d$ of side-length $\textcolor{black}{q}^{-l}, l\ge0$,  and the function $\eta$ belongs to the H\"{o}lder class of functions $\cH_\cX(\beta, L)$ for some $\beta,L>0$ then, there exist constants $\constvar[LPR-converge-to-true-value1],\constvar[LPR-converge-to-true-value2],\constvar[LPR-converge-to-true-value3]>0$ such that for any $0<h<\textcolor{black}{q}^{-l}$, any $\constref{LPR-converge-to-true-value3}h^\beta<\delta$, any $n\ge1$ and the kernel function $K(\cdot)=\Indlr{\|\cdot \|_{\infty}\le 1}$, the local polynomial estimator $\hat{\eta}^{\mathrm{LP}}(x; \cD, h, p)$ satisfies
		\[
		\l| \hat{\eta}^{\mathrm{LP}}(x; \cD, h, p) - \eta(x)\r| \le \delta
		\]
		with probability at least $1-\constref{LPR-converge-to-true-value1} \exp\l( -\constref{LPR-converge-to-true-value2}nh^d\underline{\mu}^2\bar{\mu}^{-1}\delta^2 \r)$ for all $x\in\cX$. The constants $C_1, C_2, C_3$ depend only on $p, d, L$.
	\end{prop}
	
\noindent The next proposition states that local polynomial regression estimation of a function inside a {hypercube} cannot largely deviate from the $L_2(P_X)$-projection of that function with high probability.
	\begin{prop}\label{prop-LPR-converge-to-L(P)-projection}
		Fix a {hypercube} $U\subseteq (0,1)^d$ with side-length $\textcolor{black}{q}^{-l^\prime}$, $l^\prime \in \rR_+$. Let $\cD = \l\{ (X_i,Y_i) \r\}_{i=1}^n$ be a set of $n$ i.i.d pairs $(X_i,Y_i)\in U\times \rR$. If the marginal density $\mu$ of $X_i$'s satisfies $\mu(\cdot) = p_X(\cdot| U)$, where $p_X$ is the density of a distribution $\bm{\mathrm{P}}_X$ that satisfies Assumption \ref{assumption-covar-dist} then, there exist constants $\constvar[LPR-converge-to-L(P)-projection1],\constvar[LPR-converge-to-L(P)-projection2],\constvar[LPR-converge-to-L(P)-projection3]>0$ such that for any $\delta < \constref{LPR-converge-to-L(P)-projection3}$, any $n\ge1$, $h=\textcolor{black}{q}^{-l}$, $l \ge l^\prime$, and the kernel function $K(\cdot)=\Indlr{\|\cdot \|_{\infty}\le 1}$, the local polynomial estimator $\hat{\eta}^{\mathrm{LP}}(x; \cD, h, p)$ satisfies
		\[
		\l| \hat{\eta}^{\mathrm{LP}}(x; \cD, h, p) - \bm{\Gamma}_{\textcolor{black}{q^{-l}}}^p\eta(x;U)\r| \le \delta
		\]
		with probability at least $1-\constref{LPR-converge-to-L(P)-projection1} \exp\l( -\constref{LPR-converge-to-L(P)-projection2}n\textcolor{black}{q}^{d(l^\prime-l)}\delta^2 \r)$ for all $x\in U$. The constants $\constref{LPR-converge-to-L(P)-projection1},\constref{LPR-converge-to-L(P)-projection2},\constref{LPR-converge-to-L(P)-projection3}$ depend only on $p, \bar{\rho}, \underline{\rho}$, and $d$.
	\end{prop}
	\subsection{Proof of Proposition \ref{prop-LPR-converge-to-true-value}}
	The proof is a simple extension of the proof of Theorem 3.2 in \cite{audibert2007fast}; however, we provide the proof for completeness. Fix $x\in \cX$ and $\delta>0$. Consider the matrices  $B \coloneqq \l( B_{s_1,s_2}\r)_{|s_1|,|s_2|\le p}$ and $\bar{B} \coloneqq \l( \bar{B}_{s_1,s_2}\r)_{|s_1|,|s_2|\le p}$ with the elements
	\[
	B_{s_1,s_2} \coloneqq \int_{\rR^d} u^{s_1+s_2}K(u)\mu(x+hu)du,
	\qquad
	\bar{B}_{s_1,s_2} \coloneqq \frac{1}{nh^d}\sum_{i=1}^n \l(\frac{X_i-x}{h} \r)^{s_1+s_2}K\l( \frac{X_i-x}{h} \r).
	\]
	The smallest eignevalue of $\bar{B}$ satisfies
	\begin{align}\label{eq-LPR-converge-to-true-value1}
	\lambda_{\min}(\bar{B})
	&=
	\min_{\|W\|=1} W^\top \bar{B} W
	\nonumber
	\\
	&\ge
	\min_{\|W\|=1} W^\top B W
	+
	\min_{\|W\|=1} W^\top (\bar{B}-B) W
	\nonumber
	\\
	&\ge
	\min_{\|W\|=1} W^\top B W
	-
	\sum_{|s_1|,|s_2|\le p} |\bar{B}_{s_1,s_2} - B_{s_1,s_2}|.
	\end{align}
	Define $\cX_n \coloneqq \l\{ u\in\rR^d: \|u\|\le 1; x+hu \in \cX \r\}$. For any vector $W$ satisfying $\|W\|=1$, we obtain
	\begin{align*}
	W^\top B W = \int_{\rR^d} \l( \sum_{|s|\le p}W_s u_s \r)^2 K(u) \mu(x+hu) du
	\ge
	\underline{\mu}\int_{\cX_n} \l( \sum_{|s|\le p}W_s u_s \r)^2 du.
	\end{align*}
	Since $\cX$ is a closed hypercube and we have assumed that $h \le l$, we get
	\[
	\lambda[\cX_n]
	\ge
	h^{-d}
	\lambda[\mathrm{Ball}_2(x,h)\cap \cX]
	\ge
	\textcolor{black}{q}^{-d} h^{-d}\lambda[\mathrm{Ball}_2(x,h)]
	\ge
	\textcolor{black}{q}^{-d} \lambda[\mathrm{Ball}_2(0,1)],
	\]
	where $\mathrm{Ball}_2(x,h)$ is the Euclidean ball of radius $h$ centered around $x$.
	
	Let $\cA$ denote the class of all compact subsets of $\mathrm{Ball}_2(0,1)$ having the Lebesgue measure $\textcolor{black}{q}^{-d} \lambda[\mathrm{Ball}_2(0,1)]$. Using the previous display, we obtain
	\begin{equation}\label{eq-LPR-converge-to-true-value2}
	\min_{\|W\|=1} W^\top B W
	\ge
	\underline{\mu} \min_{\|W\|=1; S\in \cA}\int_{S} \l( \sum_{|s|\le p}W_s u_s
	\r)^2 du
	\eqqcolon 2c \underline{\mu}
	\end{equation}
	By the compactness argument, the above minimum exists and is strictly positive.
	
	For $i=1,\dots,n$ and any multi-indices $s_1,s_2$ such that $|s_1|, |s_2|\le p$, define
	\[
	T_i^{(s_1,s_2)}
	\coloneqq
	\frac{1}{h^d} \l(\frac{X_i-x}{h}\r)^{s_1+s_2}K\l( \frac{X_i-x}{h} \r)
	-
	\int_{\rR^d}  u^{s_1+s_2} K(u)\mu(x+hu)du.
	\]
	We have $\rE T_i^{(s_1,s_2)} = 0$, $|T_i^{(s_1,s_2)}| \le 2h^{-d}$, and the following bound on the variance of $T_i^{(s_1,s_2)}$:
	\begin{align*}
	\Var T_i^{(s_1,s_2)}
	&\le
	\frac{1}{h^{2d}}\rElr{\l(\frac{X_i-x}{h}\r)^{2s_1+2s_2}K^2\l( \frac{X_i-x}{h} \r)}
	\\
	&\le
	\frac{1}{h^d}\int_{\rR^d}  u^{2s_1+2s_2}K^2(u)\mu(x+hu|\sfB)du
	\\
	&\le
	\frac{\bar{\mu}}{h^d}\max_{j\le p}\int_{\rR^d}  (1+|u^{4j}|)K^2(u)du
	\eqqcolon
	\frac{\kappa \bar{\mu}}{h^d} .
	\end{align*}
	From Bernstein's inequality, we get
	\[
	\rPlr{|\bar{B}_{s_1,s_2} - B_{s_1,s_2}| > \epsilon}
	=
	\rPlr{\l|\frac{1}{n}\sum_{i=1}^{n}T_i^{(s_1,s_2)}\r| >\epsilon}
	\le
	2 \exp\l(  \frac{-nh^d\epsilon^2}{2\kappa \bar{\mu} + 4\epsilon/3} \r)
	\]
	This inequality along with \eqref{eq-LPR-converge-to-true-value1} and \eqref{eq-LPR-converge-to-true-value2} imply that
	\begin{equation}\label{eq-LPR-converge-to-true-value3}
	\rPlr{\lambda_{\min}(\bar B) \le c\underline{\mu}} \le 2M^2\exp\l(  \frac{-nh^dM^{-4}c^2\underline{\mu}^2}{2\kappa \bar{\mu} + 4M^{-2}c\underline{\mu}/3} \r),
	\end{equation}
	where $M^2$ is the number of elements in the matrix $\bar B$. In what follows assume that $\lambda_{\min}(\bar B) \ge c \underline{\mu}$. Therefore,
	\begin{equation}\label{eq-LPR-converge-to-true-value4}
	\rPlr{
		\l| \hat{\eta}^{\mathrm{LP}}(x; \cD, h, p) - \eta(x)\r| \ge \delta
	}
	\le
	\rPlr{\lambda_{\min}(\bar B) \le c\underline{\mu}}
	+
	\rPlr{
		\l| \hat{\eta}^{\mathrm{LP}}(x; \cD, h, p) - \eta(x)\r| \ge \delta,\, \lambda_{\min}(\bar B) > c\underline{\mu}
	}.
	\end{equation}
	We now evaluate the second term on the right hand side of the above inequality. Define the matrix $Z \coloneqq\l( Z_{i,s} \r)_{}1\le i \le n, |s|\le p$ with elements
	\[
	Z_{i,s}
	\coloneqq
	(X_i-x)^s \sqrt{K\l( \frac{X_i-x}{h} \r)}.
	\]
	The $s$-th column of $Z$ is denoted by $Z_s$, and we introduce $Z^{(\eta)} \coloneqq \sum_{|s|\le \lfloor \beta \rfloor}\frac{\eta^{(s)}(x)}{s!}Z_s$. Since $Q=Z^\top Z$ we get
	\[
	\forall |s| \le \lfloor \beta \rfloor: U^\top (0) Q^{-1} Z^\top Z = \Indlr{s=(0,\dots,0)},
	\]
	hence $R^\top(0)Q^{-1} Z^\top Z^{(\eta)} = \eta(x)$. So we can write
	\[
	\hat{\eta}^{\mathrm{LP}}(x; \cD, h, p) - \eta(x)
	=
	R^\top(0)Q^{-1} \l(V - Z^\top Z^{(\eta)} \r)
	=
	R^\top(0)\bar{B}^{-1} \bm{a},
	\]
	where $\bm{a} \coloneqq \frac{1}{nh^d}H \l(V - Z^\top Z^{(\eta)} \r) \in \rR^M$ and $H$ is a diagonal matrix $H\coloneqq (H_{s_1,s_2})_{|s_1|,|s_2|\le p}$ with elements $H_{s_1,s_2} \coloneqq h^{-s_1} \Indlr{s_1=s_2}$. For $\lambda_{\min}(\bar B) > c \underline{\mu}$, one has\vspace{-0.1cm}
	\begin{equation}\label{eq-LPR-converge-to-true-value5}
	\l| \hat{\eta}^{\mathrm{LP}}(x; \cD, h, p) - \eta(x) \r|
	\le
	\|\bar{B}^{-1}\bm{a} \|
	\le
	\lambda_{\min}^{-1}(\bar B) \|\bm{a} \|
	\le
	c^{-1} \underline{\mu}^{-1} M \max_{s}\|a_s \|,\vspace{-0.1cm}
	\end{equation}
	where $a_s$ are the components of the vector $\bm{a}$ given by\vspace{-0.1cm}
	\[
	a_s
	=
	\frac{1}{nh^d} \sum_{i=1}^{n} \l[ Y_i - \eta_x(X_i) \r]\l(\frac{X_i-x}{h}\r)^{s}K\l( \frac{X_i-x}{h} \r).\vspace{-0.1cm}
	\]
	Note that $\eta_x(X_i)$ is the Taylor expansion of $\eta$ at $x$ and of degree $\lfloor \beta \rfloor$ (not necessarily $p$) evaluated at $X_i$.

\noindent Define:\vspace{-0.4cm}
	\begin{align*}
	T_i^{(s,1)}
	&\coloneqq
	\l[ Y_i - \eta(X_i) \r]\l(\frac{X_i-x}{h}\r)^{s}K\l( \frac{X_i-x}{h} \r),
	\\
	T_i^{(s,2)}
	&\coloneqq
	\l[ \eta(X-i) - \eta_x(X_i) \r]\l(\frac{X_i-x}{h}\r)^{s}K\l( \frac{X_i-x}{h} \r).
	\end{align*}
	One has
	\begin{equation}\label{eq-LPR-converge-to-true-value6}
	|a_s|
	\le
	\l| \frac{1}{n}\sum_{i=1}^n  T_i^{(s,1)}\r|
	+
	\l| \frac{1}{n}\sum_{i=1}^n  \l[T_i^{(s,2)} - \rE T_i^{(s,2)}\r]\r|
	+
	\l| \rE T_i^{(s,2)}\r|.
	\end{equation}
	Note that $ \rE T_i^{(s,1)}=0$, $\l|T_i^{(s,1)}\r| \le 2h^{-d}$, and\vspace{-0.2cm}
	\begin{align*}
	&\Var T_i^{(s,1)}
	\le
	\frac{1}{4h^d}\int_{\rR^d}  u^{2s}K^2(u)\mu(x+hu)du
	\le
	\frac{\kappa \bar{\mu}}{4h^d},
	\\
	&\l|T_i^{(s,2)} - \rE T_i^{(s,2)}\r|
	\le
	Lh^{\beta-d} + L\kappa h^\beta \le Ch^{\beta-d},
	\\
	&\Var T_i^{(s,2)}
	\le
	L^2h^{2\beta-d}\int_{\rR^d} |u^{2s}|K^2(u)\mu(x+hu)
	\le
	L^2\bar{\mu}\kappa h^{2\beta-d}.
	\end{align*}
	From Bernstein's inequality, for $\epsilon_1,\epsilon_2 > 0 $, we obtain\vspace{-0.1cm}
	\[
	\rPlr{\l| \frac{1}{n}\sum_{i=1}^n  T_i^{(s,1)}\r| \ge \epsilon_1}
	\le
	2\exp\l(  \frac{-nh^d\epsilon_1^2}{\kappa \bar{\mu}/2 + 4\epsilon_1/3} \r)\vspace{-0.1cm}
	\]
	and
	\[
	\rPlr{\l| \frac{1}{n}\sum_{i=1}^n  \l[T_i^{(s,2)} - \rE T_i^{(s,2)}\r]\r| \ge \epsilon_2}
	\le
	2\exp\l(  \frac{-nh^d\epsilon_{2}^2}{2L^2\kappa \bar{\mu}h^{2\beta} + 2Ch^\beta \epsilon_2/3} \r).
	\]
	Since also
	\[
	\l| \rE T_i^{(s,2)}\r|
	\le
	Lh^\beta \int_{\rR^d}|u^{s}|K^2(u) \mu(x+hu)du
	\le
	L\kappa \bar{\mu} h^\beta
	\]
	we get, using \eqref{eq-LPR-converge-to-true-value6}, that if $3 L\kappa \bar{\mu} h^\beta c^{-1} \underline{\mu}^{-1} M \le \delta \le 1$ the following inequality holds
	\begin{align*}
	\rPlr{|a_s| \ge \frac{c\underline{\mu}  \delta }{M}}
	&\le
	\rPlr{\l| \frac{1}{n}\sum_{i=1}^n  T_i^{(s,1)}\r| > \frac{c\underline{\mu}  \delta }{3M}}
	+
	\rPlr{\l| \frac{1}{n}\sum_{i=1}^n  \l[T_i^{(s,2)} - \rE T_i^{(s,2)}\r]\r| > \frac{c\underline{\mu} \delta }{3M}}
	\\
	&\le
	4\exp\l( -Cnh^d\underline{\mu}^2\bar{\mu}^{-1}\delta^2 \r).
	\end{align*}
	Combining this inequality with \eqref{eq-LPR-converge-to-true-value3}, \eqref{eq-LPR-converge-to-true-value4}, and \eqref{eq-LPR-converge-to-true-value5}, one has\vspace{-0.1cm}
	\[
	\rPlr{
		\l| \hat{\eta}^{\mathrm{LP}}(x; \cD, h, p) - \eta(x)\r| \ge \delta
	}
	\le
	\constref{LPR-converge-to-true-value1}\exp\l( -\constref{LPR-converge-to-true-value2} nh^d\underline{\mu}^2\bar{\mu}^{-1}\delta^2\r)\vspace{-0.1cm}
	\]
	for $3 L\kappa \bar{\mu} h^\beta c^{-1} \underline{\mu}^{-1} M \le \delta$ (for $\delta>1$, this inequality is obvious since $\eta, \hat{\eta}^{\mathrm{LP}}$ take values in $[0,1]$). The constants $\constref{LPR-converge-to-true-value1}, \constref{LPR-converge-to-true-value2}$ do not depend on the density $\mu$, on its support $\cX$ and the point $x\in \cX$. This concludes the proof. \hfill$\blacksquare$

	\subsection{Proof of Proposition \ref{prop-LPR-converge-to-L(P)-projection}}
	Fix a {bin} $U\subseteq (0,1)^d$ with side-length $\textcolor{black}{q}^{-l^\prime}$, $l^\prime \in \rR_+$. Consider the matrix $B \coloneqq \l( B_{s_1,s_2}\r)_{|s_1|,|s_2|\le p}$ and the vector $W \coloneqq \l(W_{s}\r)_{|s|\le p}$ with elements\vspace{-0.1cm}
	\[
	B_{s_1,s_2} \coloneqq \int_{\rR^d} u^{s_1+s_2}K(u)\mu(x+hu)du,
	\qquad
	W_s \coloneqq \int_{\rR^d}  u^{s}\eta(x+hu)K(u)\mu(x+hu)du,\vspace{-0.1cm}
	\]
	as well as the matrix $\bar{B} \coloneqq \l( \bar{B}_{s_1,s_2}\r)_{|s_1|,|s_2|\le p}$ and the vector $\bar{W} \coloneqq \l(\bar{W}_{s}\r)_{|s|\le p}$ with elements\vspace{-0.1cm}
	\[
	\bar{B}_{s_1,s_2} \coloneqq \frac{1}{nh^d}\sum_{i=1}^n \l(\frac{X_i-x}{h}\r)^{s_1+s_2}K\l( \frac{X_i-x}{h} \r),
	\qquad
	\bar{W}_s \coloneqq \frac{1}{nh^d}\sum_{i=1}^n Y_i\l(\frac{X_i-x}{h}\r)^{s}K\l( \frac{X_i-x}{h} \r).\vspace{-0.1cm}
	\]
By Lemmas \ref{lemma-L(P)-projection-closed-form} and \ref{lemma-LPR-closed-form}, one has\vspace{-0.2cm}
	\begin{align*}
	\l| \hat{\eta}^{\mathrm{LP}}(x; \cD, h, p) - \bm{\Gamma}_{\textcolor{black}{q^{-l}}}^p \eta(x) \r|
	&=
	\l| U(0)^\top Q^{-1} V - U(0)^\top B^{-1} W \r|
	\\
	&=
	\l| U(0)^\top \bar{B}^{-1} \bar{W} - U(0)^\top B^{-1} W \r|
	\\
	&\le
	\l| U(0)^\top B^{-1} \l( \bar{W} - W\r) \r|
	+
	\l| U(0)^\top \l(\bar{B}^{-1} - B^{-1} \r) \bar{W} \r|
	\eqqcolon
	J_1 + J_2.
	\end{align*}
	That is,
	\begin{equation}\label{eq-LPR-converge-to-L(P)-projection1}
	\rPlr{	\l| \hat{\eta}^{\mathrm{LP}}(x; \cD, h, p) - \bm{\Gamma}_{\textcolor{black}{q^{-l}}}^p \eta(x;U) \r| \ge
		\delta}
	\le
	\rPlr{J_1 \ge 3\delta/4} + \rPlr{J_2 \ge \delta/4}.
	\end{equation}
	First, we analyze $J_1$. Note that\vspace{-0.1cm}
	\begin{equation}\label{eq-LPR-converge-to-L(P)-projection2}
	J_1
	\le
	\|B^{-1} \l( \bar{W} - W\r)\|
	\le
	\lambda_{\min}^{-1}(B) \l\| \bar{W} - W\r\|
	\le \mu_0^{-1}\textcolor{black}{q}^{-dl^\prime} \l\| \bar{W} - W\r\|
	\le
	\mu_0^{-1}\textcolor{black}{q}^{-dl^\prime} M \max_{s}\l| \bar{W}_s - W_s\r|,
	\end{equation}
	where the third inequality follows from $\lambda_{\min}(B) \ge \mu_0\textcolor{black}{q}^{dl^\prime}$ by Lemma \ref{lemma-L(P)-projection-closed-form}, and $M$ is the number of elements in the vector $W$. Define:\vspace{-0.1cm}
	\[
	T^{(s)}_i
	\coloneqq
	\frac{1}{h^d} Y_i\l(\frac{X_i-x}{h}\r)^{s}K\l( \frac{X_i-x}{h} \r)
	-
	\int_{\rR^d}  \eta(x+hu)u^s K(u)p_X(x+hu|U)du.\vspace{-0.1cm}
	\]
	We have $\rElr{T^{(s)}_i } = 0$, $\l| T^{(s)}_i \r| \le 2h^{-d}$, and\vspace{-0.1cm}
	\[
	\Varlr{T^{(s)}_i }
	\le
	\frac{1}{h^{2d}}\rElr{\l(\frac{X_i-x}{h}\r)^{2s}K^2\l( \frac{X_i-x}{h} \r)}
	\le
	\frac{1}{h^d}\int_{\rR^d}  u^{2s}K^2(u)p_X(x+hu|U)du
	\le
	\frac{\textcolor{black}{q}^{dl^\prime}}{h^{d}}.\vspace{-0.1cm}
	\]
	By Bernstein's inequality, we get
	\[
	\rPlr{ \l| \bar{W}_s - W_s\r| \ge \epsilon}
	=
	\rPlr{\l|\frac{1}{n}\sum_{i=1}^n T^{(s)}_i \r| > \epsilon}
	\le
	2\exp\l( \frac{-nh^{d}\epsilon^2}{\textcolor{black}{q}^{1+dl^\prime}+ 4\epsilon/3} \r).
	\]
	
\noindent Combining this inequality with  \eqref{eq-LPR-converge-to-L(P)-projection2}, one obtains
	\begin{align}\label{eq-LPR-converge-to-L(P)-projection3}
	\rPlr{J_1 \ge 3\delta/4}
	&\le
	\sum_{|s|\le p} \rPlr{ \l| \bar{W}_s - W_s\r| \ge 3\mu_0\textcolor{black}{q}^{dl^\prime}M^{-1}\delta/4}
	\nonumber\\
	&
	=
	\sum_{|s|\le p} \rPlr{ \l| \frac{1}{n} \sum_{i=1}^n T^{(s)}_i \r| \ge 3\mu_0\textcolor{black}{q}^{dl^\prime}M^{-1}\delta/4}
	\le
	2M \exp\l(
	\frac{-9\mu_0^2M^{-2}\textcolor{black}{q}^{d(l^\prime-l)}n\delta^2/16}{2+\mu_0M^{-1}\delta}
	\r).
	\end{align}
	Now, we analyze $J_2$. Note that
	\begin{equation}\label{eq-LPR-converge-to-L(P)-projection4}
	J_2
	\le
	\l\|\l(\bar{B}^{-1} - B^{-1} \r) \bar{W} \r\|
	\le
	\l\|\bar{B}^{-1} - B^{-1} \r\| \l\| \bar{W} \r\|
	\le M \l\|\bar{B}^{-1} - B^{-1} \r\| \cdot  \max_s |\bar{W}_s|
	\le M \l\|\bar{B}^{-1} - B^{-1} \r\| h^{-d}.
	\end{equation}
	Define $Z \coloneqq \bar{B} - B$. One has
	\[
	\lambda_{\max}(Z)
	\le
	\sum_{|s_1|,|s_2|\le p} |Z_{s_1,s_2}|.
	\]
	Define
	\[
	T_i^{(s_1,s_2)}
	\coloneqq
	\frac{1}{h^d} \l(\frac{X_i-x}{h}\r)^{s_1+s_2}K\l( \frac{X_i-x}{h} \r) -
	\int_{\rR^d} u^{s_1+s_2}K(u)p_X(x+hu\mid U)du.
	\]
	We have $\rElr{T_i^{(s_1,s_2)}} = 0$, $|T_i^{(s_1,s_2)}| \le 2 h^{-d}$, and
	\[
	\Varlr{T_i^{(s_1,s_2)}}
	\le
	\rElr{	\frac{1}{h^{2d}} \l(\frac{X_i-x}{h}\r)^{2s_1+2s_2}K^2\l( \frac{X_i-x}{h} \r) }
	=
	\frac{1}{h^d} \int_{\rR^d}  u^{2s_1+2s_2}K^2(u)p_X(x+hu|U)du
	\le
	\frac{\textcolor{black}{q}^{dl^\prime}}{h^{d}}.
	\]
	By Bernstein's inequality, one obtains
	\begin{align*}
	\rPlr{\lambda_{\max}(Z) \ge \textcolor{black}{q}^{dl^\prime}M^{-1}\mu_0^2 \delta /8 }
	&\le
	\rPlr{\sum_{|s_1|,|s_2|\le p} |Z_{s_1,s_2}| \ge \textcolor{black}{q}^{dl^\prime}M^{-1}\mu_0^2 \delta /8}
	\\
	& \le
	\sum_{|s_1|,|s_2|\le p}\rPlr{ |Z_{s_1,s_2}| \ge \textcolor{black}{q}^{dl^\prime}M^{-3}\mu_0^2 \delta /8}
	\\
	&=
	\sum_{|s_1|,|s_2|\le p}\rPlr{ \l|\frac{1}{n}\sum_{i=1}^nT_i^{(s_1,s_2)}\r| \ge h^{-d}M^{-3}\mu_0^2 \delta /8}
	\\
	&
	\le
	2M^2 \exp\l( \frac{-n\textcolor{black}{q}^{d(l^\prime-l)}M^{-6}\mu_0^4\delta^2/64}{2+M^{-3}\mu_0^2 \delta /6} \r).
	\end{align*}
	By Lemma \ref{lemma-L(P)-projection-closed-form}, $\l\|B^{-1}\r\| \le \textcolor{black}{q}^{-dl^\prime}\mu_0^{-1}$. That is, on the event $\l\{ \lambda_{\max}(Z) \le \textcolor{black}{q}^{dl^\prime}M^{-1}\mu_0^2 \delta /8 \r\}$, one has $\l\| B^{-\frac{1}{2}} Z B^{-\frac{1}{2}}\r\| \le M^{-1}\mu_0 \delta /8 $ in which case if $M^{-1} \mu_0\delta /8 < \frac{1}{2}$, one obtains
	\begin{align*}
	\l\|\bar{B}^{-1} - B^{-1} \r\|
	&=
	\l\|B^{-\frac{1}{2}} \l( \l( I + B^{-\frac{1}{2}} Z B^{-\frac{1}{2}} \r )^{-1} - I \r) B^{-\frac{1}{2}}\r\|
	\\
	&\le
	\l\|B^{-1}\r\| \l\| \l(I + B^{-\frac{1}{2}} Z B^{-\frac{1}{2}} \r )^{-1} - I \r\|
	\\
	&\le
	\textcolor{black}{q}^{-dl^\prime}\mu_0^{-1} \sum_{j=1}^\infty \l\| B^{-\frac{1}{2}} Z B^{-\frac{1}{2}}\r\|^j
	\\
	&\le
	\textcolor{black}{q}^{-dl^\prime}\mu_0^{-1} \sum_{j=1}^\infty(M^{-1} \mu_0\delta /8)^j
	\le
	\textcolor{black}{q}^{-dl^\prime}M^{-1} \delta/4.
	\end{align*}
	This inequality along with \eqref{eq-LPR-converge-to-L(P)-projection4} imply $J_2 \le \delta/4$. In other words,
	\[
	\rPlr{J_2 \ge \delta/4}
	\le
	\rPlr{\lambda_{\max}(Z) \ge \textcolor{black}{q}^{dl^\prime}M^{-1}\mu_0^2 \delta /8 }
	\le
	2M^2 \exp\l( \frac{-n\textcolor{black}{q}^{d(l^\prime-l)}M^{-6}\mu_0^4\delta^2/64}{2+M^{-3}\mu_0^2 \delta /6} \r).
	\]
	Combining this inequality with \eqref{eq-LPR-converge-to-L(P)-projection1} and \eqref{eq-LPR-converge-to-L(P)-projection3} gives
	\[
	\rPlr{	\l| \hat{\eta}^{\mathrm{LP}}(x; \cD, h, p) - \bm{\Gamma}_{\textcolor{black}{q^{-l}}}^p \eta(x;U) \r| \ge
		\delta}
	\le
	\constref{LPR-converge-to-L(P)-projection1} \exp\l( -\constref{LPR-converge-to-L(P)-projection2}n\textcolor{black}{q}^{d(l^\prime-l)}\delta^2 \r)
	\]
	if $M^{-1} \mu_0\delta /8 < \frac{1}{2}$.  This concludes the proof. \hfill$\blacksquare$

	\section{Auxiliary analysis for Section \ref{sec-model-assumptions}}

	\subsection{Analysis of Part 1 of Example \ref{exp:cost-smoothness-misspecification}}
	\paragraph{Step 1.} Following the proof of Theorem 4.1 in \cite{rigollet2010nonparametric}, we first construct a problem instance in $\cP(\beta, \alpha, d)$. Define $M \coloneqq \l \lfloor 2^{-1} c_0^{-1} \l( \frac{2 \log 2 }{T} \r)^{\frac{-\tilde{\beta}}{2\tilde{\beta}+d}} \r\rfloor ^{\frac{1}{\beta}}$ and let $\cB \coloneqq  \l\{ \sfB_m, \, m=1,\dots, M^d \r\}$ be a re-indexed collection of the hypercubes
	\[
	\sfB_m = \sfB_{\sfm} \coloneqq \l\{ x \in [0,1]^d:\, \frac{\sfm_i-1}{2^l} \le x_i \le \frac{\sfm_i}{2^l},\; i\in\{1,\dots,d\} \r\},
	\]
	for $\sfm=(\sfm_1,\dots,\sfm_d)$ with $\sfm_i \in \{1,\dots,M\}$. Consider the regular grid $\cQ = \l\{ a_1, a_2,\dots,  a_{M^{d}}\r\}$, where $a_k$ denotes the center of bin $\sfB_k, k=1,\dots, M^d$. Define $C \coloneqq 2^{\beta - 1} L \wedge \frac{1}{4}$ and let $\phi$ be defined as follows:
	\[
	\phi(x)
	= \begin{cases}
	(1-\|x\|_\infty)^\beta & \text{if } \|x\|_\infty\le1 \\
	0 & \text{o.w.}
	\end{cases}.
	\]

\noindent Define $m\coloneqq \lceil \mu M^{d-\alpha \beta}\rceil$, where $\mu \in (0,1)$ is chosen small enough to ensure $m\le M^d$. Define the payoff functions as follows:
	\[
	f_1(x) = \frac{1}{2} + \sum_{j=1}^{m} M^{-\beta} C \phi\l( M[x-a_j] \r), \qquad f_2(x) = \frac{1}{2},
	\]
	and assume that covariates are distributed uniformly. Similar to the proof of Theorem 4.1 in \cite{rigollet2010nonparametric}, one can show that the margin condition and smoothness condition in Assumptions \ref{assumption-margin} and \ref{assumption-Holder-smoothness} are satisfied for the constructed problem instance.
	
	\paragraph{Step 2.} Next, we lower bound the regret of $\texttt{ABSE}(\tilde{\beta})$ under the constructed problem instance. To do so, we use the same exact terminology and notation as in \cite{perchet2013multi}; for the sake of brevity, we do not re-introduce the notation here. By construction, for all  bins $\sfB$ with $|\sfB| = 2^{-k}$, $k=0,1,\dots, k_0$, we have $\underline{\cI}_{\sfB} = \cK = \{1,2\}$. Define the event $\cW_{\sfB,s} \coloneqq \{\underline{\cI}_{\sfB} \subseteq \rI_{\sfB,s} \} = \{\rI_{\sfB,s} = \cK \} $ and $\cV_{\sfB} \coloneqq \bigcap\limits_{\sfB^\prime\in \cP(\sfB)}\cW_{\sfB^\prime,t_{\sfB}}$. Let\vspace{-0.2cm}
	\[
	\cA_1
	\coloneqq
	\l\{\exists t\le T;  \exists\sfB \in \cL_t; \exists s\le l_{\sfB}: \rI_{\sfB,s} \not= \cK  \text{ and } |\sfB| \ge 2^{-k_0+1}\r\}
	\]
	denote the event where one of the arms is eliminated in at least one of the bins at depth less than $k_0$. One has:\vspace{-0.2cm}
	\begin{align}\label{eq:ABSE-misspecif1-prob-decomp}
	\rPlr{\cA_1}
	&\le
	\sum_{k=1}^{k_0-1}
	\sum_{|\sfB|=2^{-k}}
	\rPlr{\cV_{\sfB} \cap \bar{\cW}_{\sfB,t_{\sfB}}}.
	\end{align}
	Note that for any bin $\sfB$ with $|\sfB| \ge 2^{-k_0+1}$,
	$
	\l|\bar{f}^{(1)}_{\sfB}
	- \bar{f}^{(2)}_{\sfB}\r|
	<
	c_0
	|\sfB|^\beta
	\le
	\frac{\epsilon_{\sfB, l_\sfB}}{2}.
$
	This implies that $\cW_{\sfB}$ can only happen if either $\bar{f}^{(1)}_{\sfB}$ or $\bar{f}^{(2)}_{\sfB}$ does not belong to its respective confidence interval $[\bar{Y}^{(1)}_{\sfB,s} \pm \epsilon_{\sfB, s}]$ or $[\bar{Y}^{(2)}_{\sfB,s} \pm \epsilon_{\sfB, s}]$ for some $s\le l_{\sfB}$. Therefore, since $-\bar{f}^{(i)}_{\sfB}\le Y_s-\bar{f}^{(i)}_{\sfB}\le 1-\bar{f}^{(i)}_{\sfB}$,
	\begin{equation}\label{eq:ABSE-misspecif1-prob-upper}
	\rPlr{\cV_{\sfB} \cap \bar{\cW}_{\sfB,t_{\sfB}}}
	\le
	\rPlr{
		\exists s \le l_\sfB; \exists i\in\cK : \l| \bar{Y}^{(i)}_{\sfB,s} - \bar{f}^{(i)}_{\sfB}\r|	 \ge \frac{\epsilon_{\sfB, s}}{4}
	}
	\le \frac{4 l_{\sfB}}{T |\sfB|^d}.
	\end{equation}
	Putting together \eqref{eq:ABSE-misspecif1-prob-decomp} and \eqref{eq:ABSE-misspecif1-prob-upper}, one obtains
	\begin{align}\label{eq:Part1-A1}
	\rPlr{\cA_1}
	&\le
	\sum_{k=1}^{k_0-1}
	\frac{4 C_l 2^{-2\tilde{\beta}k} \log\l( T2^{(2\tilde{\beta}+d)k} \r)}{T 2^{-kd}}
	\le
	\frac{4 C_l 2^{-(2\tilde{\beta}-d) k_0} \log\l( T^2 \r)}{T}
	\le
	8 C_l  T^{ \frac{-4\tilde{\beta}}{2\tilde{\beta}+d} } \log T.
	\end{align}
	
	\paragraph{Step 3.}
	Let $\tilde{c} \coloneqq 2^{1-d-2\tilde{\beta}}c_0^{-2}\log 2$ and define
	\[
	\cA_2
	\coloneqq
	\l\{\exists t\le \tilde{c}T/2;  \exists\sfB \in \cL_t: |\sfB| \ge 2^{-k_0+1}\r\}
	\]
	to be the event that for some $t\le \tilde{c}T/2$ some bin at depth $k_0$ becomes live. Note that for a bin $\sfB$ to become live by $t=\lfloor \tilde{c}T/2\rfloor$, we need $l_{\sfp(\sfB)}$ number of covariates to fall into its parent $\sfp(\sfB)$ by $t=\lfloor \tilde{c}T/2\rfloor$. Let $Z_{\sfB,t}=\Indlr{X_t \in \sfp(\sfB)}$. Note that $|Z_t| \le 1$, $\rE Z_t = |\sfB|^d$, and $\Var Z_{\sfB,t} \le \rE Z_{\sfB,t}^2 = |\sfB|^d$. Hence, one can apply the Bernstein's inequality in Lemma \ref{lemma-Bernstein-inequality} to to $\sum\limits_{t=1}^{\lfloor \tilde{c}T/2\rfloor}Z_{\sfB,t}$ for $|\sfB| = 2^{-k_0}$ to obtain
	\begin{align}\label{eq:Part1-A2}
	\rPlr{\cA_2}
	&\le
	\sum_{|\sfB| = 2^{-k_0}} \rPlr{\sum_{t=1}^{\lfloor \tilde{c}T/2\rfloor}Z_{\sfB,t} \ge l_{\sfp(\sfB)}}
	\nonumber
	\\
	&\overset{(a)}{\le}
	\sum_{|\sfB| = 2^{-k_0}} \rPlr{\sum_{t=1}^{\lfloor \tilde{c}T/2\rfloor}Z_{\sfB,t} \ge c_0^{-2} |\sfp(\sfB)|^{-2\tilde{\beta}}}
	\nonumber
	\\
	&\le 		
	2^{k_0d}\exp\l(
	-\frac{c_0^{-4}2^{4(k_0-1)\tilde{\beta}-1}/2}{\tilde{c}T2^{(k_0-1)d-1}+c_0^{-2}2^{2(k_0-1)\tilde{\beta}-1}/3}
	\r)
	\nonumber
	\\
	&=
	2^{k_0d}\exp\l(
	-\frac{c_0^{-2}2^{2(k_0-1)\tilde{\beta}-1}/2}{\frac{1}{2}+1/3}
	\r)
	\nonumber
	\\
	&\le
	c_1T^{\frac{d}{2\tilde{\beta}+d}} \exp\l( -c_2 T^{\frac{2\tilde{\beta}}{2\tilde{\beta}+d}} \r) \le c_3T^{-1},
	\end{align}
	for some constants $c_1,c_2,c_3>0$, where (a) follows from  $l_\sfB$
	$
	l_\sfB
	\ge
	c_0^{-2} |\sfB|^{-2\tilde{\beta}}.
	$ by the definition of $l_\sfB$.
	
	\paragraph{Step 4.} Let $S\coloneqq \l\{x\in[0,1]^d: f_1(x)\not = \frac{1}{2}  \r\}$. Define the event
	\[
	\cA_3 \coloneqq \l\{ \sum_{t=1}^{\lfloor \tilde{c}T/2\rfloor}\Indlr{X_t \in S} < \tilde{c}mM^{-d}T/4 \r\}.
	\]
	Define $Z_t \coloneqq \Indlr{X_t \in S}$ and note that $|Z_t| \le 1, \rE Z_t = mM^d$, and $\Var Z_t\le \rE Z_t = mM^d$. As a result we can apply the Bernstein's inequality in Lemma \ref{lemma-Bernstein-inequality} to obtain
	\begin{align}\label{eq:Part1-A3}
	\rPlr{\cA_3}
	&\le
	\exp\l( -  \tilde{c}mM^{-d}T/20\r )
	\le
	\exp\l( -c_5 T^{\frac{2\tilde{\beta}+d-\alpha \tilde{\beta}}{2\tilde{\beta}+d}} \r)
	\overset{(a)}{\le}
	\exp\l( -c_4 T^{\frac{2\tilde{\beta}}{2\tilde{\beta}+d}} \r)
	\le
	c_5 T^{-1},
	\end{align}
	for some constants $c_4,c_5>0$, where (a) follows from the assumption that $\alpha\le \frac{1}{\beta}\le \frac{1}{\tilde{\beta}}$.
	
	\paragraph{Step 5.} Note that on the event $ \bar{\cA}_1 \cap \bar{\cA}_2$, the $\texttt{ABSE}(\tilde{\beta})$ has not eliminated any arms over any region of the covariate space up to time $t=\lfloor \tilde{c}T/2\rfloor$. On the other hand, on the event $\bar{\cA}_3$, up to time $t=\lfloor \tilde{c}T/2\rfloor$, at least $\tilde{c}mM^{-d}T/4$ number of covariates have fallen into $S$, where the first arm is strictly optimal. Recall the definition of the inferior sampling rate in \eqref{eq:inferior-sampling-def}. One has:\vspace{-0.1cm}
	\begin{align*}
	\mathcal{S}^{\texttt{ABSE}(\tilde{\beta})}(\sfP;T)
	&\ge \mathbb{E}^\pi \left[ \sum\limits_{t=1}^T \Indlr{ f_{\pi^\ast_t}(X_t) \neq f_{\pi_t}(X_t)}\middle | \bar{\cA}_1 \cap \bar{\cA}_2 \cap \bar{\cA}_3 \right]
	\rPlr{\bar{\cA}_1 \cap \bar{\cA}_2 \cap \bar{\cA}_3 }
	\\
	&\overset{(a)}{\ge}
	\tilde{c}mM^{-d}T/8 \l(1- 8 C_l  T^{ \frac{-4\tilde{\beta}}{2\tilde{\beta}+d} } \log T  - c_3 T^{-1} -c_5 T^{-1} \r)
	\\
	&
	\ge
	c_6 T^{\frac{\alpha \tilde{\beta}}{2\tilde{\beta}+d}},
	\end{align*}
	for some constant $c_6>0$, where (a) follows from \eqref{eq:Part1-A1}, \eqref{eq:Part1-A2}, and \eqref{eq:Part1-A3}. Using this inequality along with Lemma \ref{lemma-regret-inferior-sampling-rate}, the result follows. \hfill$\blacksquare$

	\subsection{Analysis of Part 2 of Example \ref{exp:cost-smoothness-misspecification}}
\paragraph{Step 1.} let $\tilde k \coloneqq \l\lceil \frac{\log_2\l(T/2\log 2\r)}{ (2\tbeta+1) }\r\rceil$, $M\coloneqq2^{\tilde k}$, and $\cB \coloneqq  \l\{ \sfB_m, \, m=1,\dots, M^{d} \r\}$ be a re-indexed collection of the hypercubes\vspace{-0.1cm}
\[
\sfB_m = \sfB_{\sfm} \coloneqq \l\{ x \in [0,1]^d:\, \frac{\sfm_i-1}{2^l} \le x_i \le \frac{\sfm_i}{2^l},\; i\in\{1,\dots,d\} \r\},\vspace{-0.2cm}
\]
for $\sfm=(\sfm_1,\dots,\sfm_d)$ with $\sfm_i \in \{1,\dots,M\}$. Consider the regular grid $\cQ = \l\{ a_1, a_2,\dots,  a_{M^{d}}\r\}$, where $a_k$ denotes the center of bin $\sfB_k, k=1,\dots, M^d$. Define $C \coloneqq 2^{\beta - 1} L \wedge \frac{1}{4}$ and let $\phi$ be defined as follows:
\[
\phi(x)
= \begin{cases}
	(1-|x_1|)^\beta & \text{if } |x_1|\le1 \\
	0 & \text{o.w.}
\end{cases}.
\]

\noindent Define $m\coloneqq 2 \times \lceil \mu M^{1-\alpha \beta}\rceil$, where $\mu \in (0,1)$ is chosen small enough to ensure $m\le M^d$. Define the payoff functions as follows:
\[
f_1(x) = \frac{1}{2} + \sum_{j=1}^{m} M^{-\beta} C (-1)^j \phi\l( M[x_1-\tilde a_j] \r), \qquad f_2(x) = \frac{1}{2},
\]
where $\tilde a_j = \frac{j + \frac{1}{2}}{M}$ for each $j=1,\ldots,m$, and assume that covariates are distributed uniformly. Similar to the proof of Theorem 4.1 in \cite{rigollet2010nonparametric}, one can show that the margin condition and smoothness condition in Assumptions \ref{assumption-margin} and \ref{assumption-Holder-smoothness} are satisfied for the constructed problem instance.

\paragraph{Step 2.} Next, we lower bound the regret of $\texttt{ABSE}(\tilde{\beta})$ under the constructed problem instance. To do so, we use the same exact terminology and notation as in \cite{perchet2013multi}; for the sake of brevity, we do not re-introduce the notation here. By construction, for all  bins $\sfB$ with $|\sfB| = 2^{-k}$, $k=0,1,\dots, k_0$, we have
\[
\bar{f}^{(i)}_{\sfB} = \frac{1}{P_X(\sfB) }\int_{\sfB} f_{k}(x)ddP_X(x) =
\frac{1}{2}, \qquad i \in \cK.
\]
Define the event $\cW_{\sfB,s} \coloneqq \{\rI_{\sfB,s} = \cK \} $ and $\cV_{\sfB} \coloneqq \bigcap\limits_{\sfB^\prime\in \cP(\sfB)}\cW_{\sfB^\prime,t_{\sfB}}$. Let\vspace{-0.2cm}
\[
\cA_1
\coloneqq
\l\{\exists t\le T;  \exists\sfB \in \cL_t; \exists s\le l_{\sfB}: \rI_{\sfB,s} \not= \cK  \text{ and } |\sfB| \ge 2^{-k_0+1}\r\}
\]
denote the event where one of the arms is eliminated in at least one of the bins at depth less than $k_0$. One has:\vspace{-0.2cm}
\begin{align}\label{eq:ABSE-misspecif2-prob-decomp}
	\rPlr{\cA_1}
	&\le
	\sum_{k=1}^{k_0-1}
	\sum_{|\sfB|=2^{-k}}
	\rPlr{\cV_{\sfB} \cap \bar{\cW}_{\sfB,t_{\sfB}}}.
\end{align}
Note that for any bin $\sfB$ with $|\sfB| \ge 2^{-k_0+1}$,
$
\l|\bar{f}^{(1)}_{\sfB}
- \bar{f}^{(2)}_{\sfB}\r|
=
0
<
\frac{\epsilon_{\sfB, l_\sfB}}{2}.
$
This implies that $\cW_{\sfB}$ can only happen if either $\bar{f}^{(1)}_{\sfB}$ or $\bar{f}^{(2)}_{\sfB}$ does not belong to its respective confidence interval $[\bar{Y}^{(1)}_{\sfB,s} \pm \epsilon_{\sfB, s}]$ or $[\bar{Y}^{(2)}_{\sfB,s} \pm \epsilon_{\sfB, s}]$ for some $s\le l_{\sfB}$. Therefore, since $-\bar{f}^{(i)}_{\sfB}\le Y_s-\bar{f}^{(i)}_{\sfB}\le 1-\bar{f}^{(i)}_{\sfB}$,
\begin{equation}\label{eq:ABSE-misspecif2-prob-upper}
	\rPlr{\cV_{\sfB} \cap \bar{\cW}_{\sfB,t_{\sfB}}}
	\le
	\rPlr{
		\exists s \le l_\sfB; \exists i\in\cK : \l| \bar{Y}^{(i)}_{\sfB,s} - \bar{f}^{(i)}_{\sfB}\r|	 \ge \frac{\epsilon_{\sfB, s}}{4}
	}
	\le \frac{4 l_{\sfB}}{T |\sfB|^d}.
\end{equation}
Putting together \eqref{eq:ABSE-misspecif2-prob-decomp} and \eqref{eq:ABSE-misspecif2-prob-upper}, one obtains
\begin{align}\label{eq:Part2-A1}
	\rPlr{\cA_1}
	&\le
	\sum_{k=1}^{k_0-1}
	\frac{4 C_l 2^{-2\tilde{\beta}k} \log\l( T2^{(2\tilde{\beta}+d)k} \r)}{T 2^{-kd}}
	\le
	\frac{4 C_l 2^{-(2\tilde{\beta}-d) k_0} \log\l( T^2 \r)}{T}
	\le
	8 C_l  T^{ \frac{-4\tilde{\beta}}{2\tilde{\beta}+d} } \log T.
\end{align}

\paragraph{Step 3.}
Let $\tilde{c} \coloneqq 2^{1-d-2\tilde{\beta}}c_0^{-2}\log 2$ and define
\[
\cA_2
\coloneqq
\l\{\exists t\le \tilde{c}T/2;  \exists\sfB \in \cL_t: |\sfB| < 2^{-k_0+1}\r\}
\]
to be the event that for some $t\le \tilde{c}T/2$ some bin at depth $k_0$ becomes live. Note that for a bin $\sfB$ to become live by $t=\lfloor \tilde{c}T/2\rfloor$, we need $l_{\sfp(\sfB)}$ number of covariates to fall into its parent $\sfp(\sfB)$ by $t=\lfloor \tilde{c}T/2\rfloor$. Let $Z_{\sfB,t}=\Indlr{X_t \in \sfp(\sfB)}$. Note that $|Z_t| \le 1$, $\rE Z_t = |\sfB|^d$, and $\Var Z_{\sfB,t} \le \rE Z_{\sfB,t}^2 = |\sfB|^d$. Hence, one can apply the Bernstein's inequality in Lemma \ref{lemma-Bernstein-inequality} to to $\sum\limits_{t=1}^{\lfloor \tilde{c}T/2\rfloor}Z_{\sfB,t}$ for $|\sfB| = 2^{-k_0}$ to obtain
\begin{align}\label{eq:Part2-A2}
	\rPlr{\cA_2}
	&\le
	\sum_{|\sfB| = 2^{-k_0}} \rPlr{\sum_{t=1}^{\lfloor \tilde{c}T/2\rfloor}Z_{\sfB,t} \ge l_{\sfp(\sfB)}}
	\nonumber
	\\
	&\overset{(a)}{\le}
	\sum_{|\sfB| = 2^{-k_0}} \rPlr{\sum_{t=1}^{\lfloor \tilde{c}T/2\rfloor}Z_{\sfB,t} \ge c_0^{-2} |\sfp(\sfB)|^{-2\tilde{\beta}}}
	\nonumber
	\\
	&\le 		
	2^{k_0d}\exp\l(
	-\frac{c_0^{-4}2^{4(k_0-1)\tilde{\beta}-1}/2}{\tilde{c}T2^{(k_0-1)d-1}+c_0^{-2}2^{2(k_0-1)\tilde{\beta}-1}/3}
	\r)
	\nonumber
	\\
	&=
	2^{k_0d}\exp\l(
	-\frac{c_0^{-2}2^{2(k_0-1)\tilde{\beta}-1}/2}{\frac{1}{2}+1/3}
	\r)
	\nonumber
	\\
	&\le
	c_1T^{\frac{d}{2\tilde{\beta}+d}} \exp\l( -c_2 T^{\frac{2\tilde{\beta}}{2\tilde{\beta}+d}} \r) \le c_3T^{-1},
\end{align}
for some constants $c_1,c_2,c_3>0$, where (a) follows from  $l_\sfB$
$
l_\sfB
\ge
c_0^{-2} |\sfB|^{-2\tilde{\beta}}.
$ by the definition of $l_\sfB$.

\paragraph{Step 2.} Let $S_1 \coloneqq \l\{x\in[0,1]^d: f_1(x) > \frac{1}{2}  \r\}$ and $S_2 \coloneqq \l\{x\in[0,1]^d: f_1(x) < \frac{1}{2}  \r\}$. Define the events
\[
\cA_{31} \coloneqq \l\{ \sum_{t=1}^{T}\Indlr{X_t \in S_1} < \tilde{c}mM^{-d}T/8 \r\}, \qquad \cA_{32} \l\{ \sum_{t=1}^{T}\Indlr{X_t \in S_2} < \tilde{c}mM^{-d}T/8 \r\},
\]
and let $\cA_{3} = \cA_{31} \cap \cA_{32}$.
Define $Z_t \coloneqq \Indlr{X_t \in S_1}$ and note that $|Z_t| \le 1, \rE Z_t = \tilde{c}mM^{-1}/4$, and $\Var Z_t\le \rE Z_t = \tilde{c}mM^{-1}/4$. As a result we can apply the Bernstein's inequality in Lemma \ref{lemma-Bernstein-inequality} to obtain
\begin{align*}
	\rPlr{\cA_{31}}
	&\le
	\exp\l( -  \tilde{c}mM^{-1}T/40\r )
	\le
	\exp\l( -c_5 T^{\frac{2\tilde{\beta}+d-\alpha {\beta}}{2\tilde{\beta}+d}} \r)
	\overset{(a)}{\le}
	\exp\l( -c_4 T^{\frac{2\tilde{\beta}}{2\tilde{\beta}+d}} \r)
	\le
	c^\prime_5 T^{-1},
\end{align*}
for some constants $c_4,c^\prime_5>0$, where (a) follows from the assumption that $\alpha\le \frac{1}{\beta}$. A similar upper bound can be shown for the probability of the event $\cA_{32}$, which implies
\begin{align}\label{eq:Part2-A3}
	\rPlr{\cA_{3}}
	&
	\le
	c_5 T^{-1},
\end{align}
for some constants $c_5>0$.

\paragraph{Step 5.} Note that on the event $ \bar{\cA}_1 \cap \bar{\cA}_2$, the $\texttt{ABSE}(\tilde{\beta})$ has not eliminated any arms over any region of the covariate space up to time $t=\lfloor \tilde{c}T/2\rfloor$. On the other hand, on the event $\bar{\cA}_3$, up to time $t=\lfloor \tilde{c}T/2\rfloor$, at least $\tilde{c}mM^{-d}T/8$ number of covariates have fallen into $S_1$ and also $S_2$, where the first arm and the second arm are strictly optimal, respectively. Recall the definition of the inferior sampling rate in \eqref{eq:inferior-sampling-def}. One has:\vspace{-0.1cm}
\begin{align*}
	\mathcal{S}^{\texttt{ABSE}(\tilde{\beta})}(\sfP;T)
	&\ge \mathbb{E}^\pi \left[ \sum\limits_{t=1}^T \Indlr{ f_{\pi^\ast_t}(X_t) \neq f_{\pi_t}(X_t)}\middle | \bar{\cA}_1 \cap \bar{\cA}_2 \cap \bar{\cA}_3 \right]
	\rPlr{\bar{\cA}_1 \cap \bar{\cA}_2 \cap \bar{\cA}_3 }
	\\
	&\overset{(a)}{\ge}
	\tilde{c}mM^{-1}T/8 \l(1- 8 C_l  T^{ \frac{-4\tilde{\beta}}{2\tilde{\beta}+d} } \log T  - c_3 T^{-1} -c_5 T^{-1} \r)
	\\
	&
	\ge
	c_6 T^{\frac{\alpha \beta}{2\tilde{\beta}+d}},
\end{align*}
for some constant $c_6>0$, where (a) follows from \eqref{eq:Part2-A1}, \eqref{eq:Part2-A2}, and \eqref{eq:Part2-A3}. Using this inequality along with Lemma \ref{lemma-regret-inferior-sampling-rate}, the result follows. \hfill$\blacksquare$

	\section{Proof of auxiliary lemmas}\label{app:auxiliary}

	\subsection{Proof of Lemma \ref{lemma-multip-presrves-smoothness}}
	\begin{lemma*}
		Suppose $f\in \cH_{\rR^d}(\beta, L)$ for some $\cX \in \rR^d$,  $0<\beta\le 1$ and $L>0$, and define the function $g$ such that $g(x) = C^{-\beta} f(Cx)$ for all $x\in \rR^d$ and some $C>0$. Then, $g \in   \cH_{\rR^d}(\beta, L)$.
	\end{lemma*}
	
	\begin{proof}
		For any $x,y\in \rR^d$, one has
		\begin{align*}
		\l| g(x) - g(y)\r|
		&=
		C^{-\beta} \l|  f(Cx) -  f(Cy)\r|
		\le
		C^{-\beta}L \l\|  Cx -  Cy\r\|_\infty^\beta
		=
		L \l\|  x -  y\r\|_\infty^\beta.
		\end{align*}
		This concludes the proof.
	\end{proof}

	\subsection{Proof of Lemma \ref{lemma-max-presrves-smoothness}}
	\begin{lemma*}
		Suppose $f,g\in \cH_{\cX}(\beta, L)$ for some $\cX \subseteq \rR^d$,  $0<\beta\le 1$ and $L>0$, and define the functions $h_1 \coloneqq \max(f,g)$ and $h_2 \coloneqq \min(f,g)$. Then, $h_1, h_2\in   \cH_{\cX}(\beta, L)$.
	\end{lemma*}
	\begin{proof}
		We only prove the result for the function $h_1$. A similar analysis can be used for $h_2.$ Fix some $x,y\in \cX$. If $h_1(x)=f(x)$ and $h_1(y)=f(y)$, or $h_1(x)=g(x)$ and $h_1(y)=g(y)$ then, one has
		\[
		|h_1(x) - h_1(y)| \le L\|x-y \|_\infty^\beta.
		\]
		Now suppose $h_1(x)=f(x)$ and $h_1(y)=g(y)$. Without loss of generality, assume that $f(x) \le g(y) $ then, one has
		\[
		|h_1(x) - h_1(y)| \le
		|g(x) - g(y)| \le
		L\|x-y \|_\infty^\beta.
		\]
		The case $h_1(x)=f(x)$ and $h_1(y)=g(y)$ can be analyzed similarly. This concludes the proof.
	\end{proof}

	\subsection{Proof of Lemma \ref{Tsybakov2008introduction-lemma-2.6}}\vspace{-0.1cm}
	\begin{lemma*}
		Let $\rho_0,\rho_1$ be two probability distributions supported on some set $\cX$, with $\rho_0$ absolutely continuous with respect to $\rho_1$. Then for any measurable function $\Psi : \cX \rightarrow \{0, 1\}$, one has:\vspace{-0.2cm}
		\begin{equation*}
		\mathbb{P}_{\rho_0} \{\Psi(X) = 1\} + \mathbb{P}_{\rho_1} \{\Psi(X) = 0\} \ge \frac{1}{2} \exp(-\mathrm{KL}(\rho_0,\rho_1)).\vspace{-0.1cm}
		\end{equation*}
	\end{lemma*}
	
	\begin{proof}
		Define $\mathcal{B}$ to be the event that $\Psi(X) = 1$. One has\vspace{-0.1cm}
		\begin{align*}
		\mathbb{P}_{\rho_0} \{\Psi(X) = 1\} + \mathbb{P}_{\rho_1} \{\Psi(X) = 0\}
		=
		\mathbb{P}_{\rho_0} \{\mathcal{B}\} + \mathbb{P}_{\rho_1} \{\bar{\mathcal{B}}\}
		\ge
		\int \min \{d\rho_0, d\rho_1\}
		\ge
		\frac{1}{2} \exp(-\mathrm{KL}(\rho_0,\rho_1)),\vspace{-0.1cm}
		\end{align*}
		where the last inequality follows from \citealt[Lemma 2.6]{Tsybakov2008introduction}.
	\end{proof}

	\vspace{-0.1cm}
	\subsection{Proof of Lemma \ref{lemma:product-Holder-smooth-funcs}}\vspace{-0.1cm}
	\begin{lemma*}
		Suppose $f,g\in \cH_{\cX}(\beta, L)$ for some $\cX \subseteq [0,1]$,  $\beta>0$, and $L>0$, and define the function $h \coloneqq f \cdot g$ as the product of $f$ and $g$. Then, $h \in \cH(\beta, L^\prime)$ for some $L^\prime > 0$.
	\end{lemma*}

	\begin{proof}
		Note that $h$ is $\lfloor \beta \rfloor$ times continuously differentiable. Hence, we only need to show that there exists some $L^\prime > 0$ such that
		for any $x,x^\prime \in \cX$,\vspace{-0.2cm}
		\[
		\l| h(x^\prime) -h_x(x^\prime) \r|
		\le
		L^\prime\|x-x^\prime\|_\infty^{\beta}.\vspace{-0.2cm}
		\]
		By the triangle inequality, one has\vspace{-0.15cm}
		\begin{align}\label{eq:product-Holder-smooth-funcs1}
			\l| h(x^\prime) -h_x(x^\prime) \r|
			\le
			\l| (f\cdot g)_x(x^\prime) -f_x(x^\prime)\cdot g_x(x^\prime) \r|
			&+
			\l| f_x(x^\prime)\cdot g(x^\prime)-f_x(x^\prime)\cdot g_x(x^\prime) \r|
			\nonumber
			\\
			&+
			\l| f_x(x^\prime)\cdot g(x^\prime)-f(x^\prime)\cdot g(x^\prime) \r|.
		\end{align}
		Since $\cX \subseteq [0,1]$ and $f,g\in \cH_{\cX}(\beta, L)$, one has for some $L_1, L_2>0$: \vspace{-0.15cm}
		\begin{align}\label{eq:product-Holder-smooth-func2}
			\l| f_x(x^\prime)\cdot g(x^\prime)-f_x(x^\prime)\cdot g_x(x^\prime) \r|
			=
			|f_x(x^\prime)| \cdot \l|  g(x^\prime)- g_x(x^\prime) \r|
			\le
			L_1 \|x-x^\prime\|_\infty^{\beta};
			\end{align}
			\vspace{-1cm}
			\begin{align}
			\label{eq:product-Holder-smooth-func3}
			  \l| f_x(x^\prime)\cdot g(x^\prime)-f(x^\prime)\cdot g(x^\prime) \r|
			  =
			  \l| f_x(x^\prime) -f(x^\prime)\r| |g(x^\prime)|
			  \le
			  L_2 \|x-x^\prime\|_\infty^{\beta}.
		\end{align}
		Furthermore,  let $\{a_s\}_{0 \le s \le \lfloor \beta \rfloor}$, $\{b_s\}_{0 \le s \le \lfloor \beta \rfloor}$, and $\{c_s\}_{0 \le s \le \lfloor \beta \rfloor}$ be the coefficients of the Taylor expansions    $f_x(x^\prime), g_x(x^\prime)$, and $h_x(x^\prime)$, respectively. Notably, $c_s = \sum_{s^\prime = 0}^s a_{s^\prime} b_{s-s^\prime}.$ This equality implies that \vspace{-0.2cm}
		\begin{equation}\label{eq:product-Holder-smooth-funcs4}
		\l| (f\cdot g)_x(x^\prime) -f_x(x^\prime)\cdot g_x(x^\prime) \r|
		=
		\l|
		\sum_{s=\lfloor \beta \rfloor+1}^{2\lfloor \beta \rfloor}
		\sum_{s^\prime =0 }^s a_{s^\prime} b_{s-s^\prime} (x-x^\prime)^s
		\r|
		\le
		L_3 \|x-x^\prime\|_\infty^{\beta},\vspace{-0.2cm}
		\end{equation}
				for some $L_3>0$. Then, the result follows from putting together \eqref{eq:product-Holder-smooth-funcs1}, \eqref{eq:product-Holder-smooth-func2}, \eqref{eq:product-Holder-smooth-func3}, and \eqref{eq:product-Holder-smooth-funcs4}.
	\end{proof}

\color{black}

\section{Numerical Analysis}\label{appendix:numerics}

\subsection{Different parameters}
In this section we provide the auxiliary results of the numerical study in \S\ref{sec-numerics}, which also includes experiments with different values for smoothness parameter $\beta$ and horizon length $T$. Let $\hat{\mathcal{R}}^\pi(\sfP;T)$ be empirical cumulative regret of policy $\pi$ in the simulation then, the average Relative Loss, $\mathrm{RL}^\pi(\sfP;T) $, with respect the \texttt{ABSE}($\beta$) is defined as follows:
\[
\text{RL}^\pi(\sfP;T)
\coloneqq
\frac{\hat{\mathcal{R}}^\pi(\sfP;T) - \hat{\mathcal{R}}^{\texttt{ABSE}(\beta)}(\sfP;T)}{\hat{\mathcal{R}}^{\texttt{ABSE}(\beta)}(\sfP;T)}.
\]
Table \ref{table-exp211_fixed_T} (Table \ref{table-exp212_fixed_T}) provides the average cumulative regret and Relative Loss for  Setting I (Setting II) for a fixed horizon length $T=2\times 10^6$ and different smoothness parameters $\beta \in \{0.85, 0.9, 0.95\}$ ($\beta \in \{0.45, 0.5, 0.55\}$). Table \ref{table-exp211_fixed_beta} (Table \ref{table-exp212_fixed_beta}) provides the average cumulative regret and Relative Loss for  Setting I (Setting II) for a fixed smoothness parameter $\beta = 0.9$ ($\beta = 0.5$) and different horizon lengths $T \in \{2\times 10^6, 2.5\times 10^6, 3\times 10^6\}$. These tables show that the results of numerical analysis in $\S$\ref{sec-numerics} are consistent across different smoothness values and larger horizon lengths.

\begin{table}[H]
	\centering
	\scriptsize
	\makebox[\textwidth][c]{
	\begin{tabular}{|l|l|l|l|l|l|l|l|l|l|l|l|l|l|l|l|}
		\hline
		&
		\textbf{\texttt{ABSE}(\boldmath{$ \beta$}) } &
		\texttt{SACB}&
		\multicolumn{13}{c|}{\textbf{\texttt{ABSE}(\boldmath{$ \tilde \beta$}) }}
		\\
		\boldmath{$\beta$} &  &  & \textbf{0.4} & \textbf{0.45} & \textbf{0.5} & \textbf{0.55} & \textbf{0.6} & \textbf{0.65} & \textbf{0.7} & \textbf{0.75} & \textbf{0.8} & \textbf{0.85} & \textbf{0.9} & \textbf{0.95} & \textbf{1}
		\\
		\hline
		\boldmath{$0.85$} & 3.38 & 4.08 & 4.29 & 4.30 & 4.30 & 4.30 & 4.30 & 4.30 & 4.29 & 3.92& 4.13 & 3.38 & 3.10 & 2.98 & 2.73\\
		\hline
		\boldmath{$0.9$}& 1.88 & 2.30 & 2.72 & 2.72 & 2.72 & 2.72 & 2.72 & 2.73 & 2.59 & 2.42 & 2.39 & 2.01 & 1.88 & 1.71 & 1.58 \\
		\hline
		\boldmath{$0.95$}& 2.01 & 2.63 & 3.38 & 3.38 & 3.39 & 3.39 & 3.38 & 3.39 & 3.06 & 2.92 & 2.75 & 2.35 & 2.24 & 2.01 & 1.82
		\\
		\hline
		
	\end{tabular}
}
	\scriptsize
	\makebox[\textwidth][c]{
	\begin{tabular}{|l|l|l|l|l|l|l|l|l|l|l|l|l|l|l|}
		\hline
		&
		\texttt{SACB}&
		\multicolumn{13}{c|}{\textbf{\texttt{ABSE}(\boldmath{$ \tilde \beta$}) }}
		\\
		\boldmath{$\beta$} &  & \textbf{0.4} & \textbf{0.45} & \textbf{0.5} & \textbf{0.55} & \textbf{0.6} & \textbf{0.65} & \textbf{0.7} & \textbf{0.75} & \textbf{0.8} & \textbf{0.85} & \textbf{0.9} & \textbf{0.95} & \textbf{1}
		\\
		\hline
		\boldmath{$0.85$} & 20\% & 27\% & 27\% & 27\% & 27\% & 27\% & 27\% & 26\% & 16\% & 22\% & 0\% & -8\% & -11\% & -19\%\\
		\hline
		\boldmath{$0.9$} & 22\% & 44\% & 44\% & 45\% & 45\% & 45\% & 45\% & 37\% & 28\% & 27\% & 7\% & 0\% & -8\% & -15\%
		\\
		\hline
		\boldmath{$0.95$}& 31\% & 68\% & 68\% & 68\% & 68\% & 68\% & 68\% & 52\% & 45\% & 36\% & 17\% & 11\% & 0\% & -9\% \\
		\hline
		
	\end{tabular}
}
	\caption{\small Results for Setting I with fixed horizon $T=2\times 10^6$, and different choices of $\beta$. \textit{Above:} Average cumulative regret divided by $10^4$; \textit{Below:} Relative Loss. }
	\label{table-exp211_fixed_T}
\end{table}

\begin{table}[H]
	\centering
	\scriptsize
	\makebox[\textwidth][c]{
	\begin{tabular}{|l|l|l|l|l|l|l|l|l|l|l|l|l|l|l|l|}
		\hline
		&
		\textbf{\texttt{ABSE}(\boldmath{$ \beta$}) } &
		\texttt{SACB}&
		\multicolumn{13}{c|}{\textbf{\texttt{ABSE}(\boldmath{$ \tilde \beta$}) }}
		\\
		\boldmath{$\beta$} &  &  & \textbf{0.4} & \textbf{0.45} & \textbf{0.5} & \textbf{0.55} & \textbf{0.6} & \textbf{0.65} & \textbf{0.7} & \textbf{0.75} & \textbf{0.8} & \textbf{0.85} & \textbf{0.9} & \textbf{0.95} & \textbf{1}
		\\
		\hline
		\boldmath{$0.45$} & 0.75 & 1.27 & 0.55 & 0.75 & 1.04 & 1.46 & 2.08 & 2.91 & 4.04 & 5.55 & 5.92 & 5.92 & 5.92 & 5.91 & 5.91\\
		\hline
		\boldmath{$0.5$}& 0.81 & 1.16 & 0.48 & 0.62 & 0.81 & 1.10 & 1.52 & 2.09 & 2.88 & 3.89 & 4.08 & 4.06 & 4.07 & 4.08 & 4.07
		 \\
		\hline
		\boldmath{$0.55$}&  0.88 & 1.39 & 0.46 & 0.54 & 0.66 & 0.88 & 1.16 & 1.56 & 2.09 & 2.75 & 2.77 & 2.77 & 2.77 & 2.77 & 2.76\\
		\hline
		
	\end{tabular}
}
%
	\scriptsize
	\makebox[\textwidth][c]{
	\begin{tabular}{|l|l|l|l|l|l|l|l|l|l|l|l|l|l|l|}
		\hline
		&
		\texttt{SACB}&
		\multicolumn{13}{c|}{\textbf{\texttt{ABSE}(\boldmath{$ \tilde \beta$}) }}
		\\
		\boldmath{$\beta$} &  & \textbf{0.4} & \textbf{0.45} & \textbf{0.5} & \textbf{0.55} & \textbf{0.6} & \textbf{0.65} & \textbf{0.7} & \textbf{0.75} & \textbf{0.8} & \textbf{0.85} & \textbf{0.9} & \textbf{0.95} & \textbf{1}
		\\
		\hline
		\boldmath{$0.45$} & 68\% & -25\% & 0\% & 38\% & 94\% & 176\% & 286\% & 436\% & 635\% & 685\% & 685\% & 685\% & 683\% & 684\% \\
		\hline
		\boldmath{$0.5$} & 42\% & -41\% & -24\% & 0\% & 35\% & 86\% & 156\% & 252\% & 376\% & 399\% & 396\% & 397\% & 399\% & 397\%
		 \\
		\hline
		\boldmath{$0.55$}& 58\% & -47\% & -38\% & -24\% & 0\% & 32\% & 77\% & 138\% & 213\% & 214\% & 214\% & 214\% & 215\% & 214\% \\
		\hline
		
	\end{tabular}
}
	\caption{\small Results for Setting II with fixed horizon $T=2\times 10^6$, and different choices of $\beta$. \textit{Above:} Average cumulative regret divided by $10^3$t; \textit{Below:} Relative Loss. }
\label{table-exp212_fixed_T}
\end{table}

\begin{table}[H]
\makebox[\textwidth][c]{
	\scriptsize
	\begin{tabular}{|l|l|l|l|l|l|l|l|l|l|l|l|l|l|l|l|}
		\hline
		&
		\textbf{\texttt{ABSE}(\boldmath{$ \beta$}) } &
		\texttt{SACB}&
		\multicolumn{13}{c|}{\textbf{\texttt{ABSE}(\boldmath{$ \tilde \beta$}) }}
		\\
		\boldmath{$T$} &  &  & \textbf{0.4} & \textbf{0.45} & \textbf{0.5} & \textbf{0.55} & \textbf{0.6} & \textbf{0.65} & \textbf{0.7} & \textbf{0.75} & \textbf{0.8} & \textbf{0.85} & \textbf{0.9} & \textbf{0.95} & \textbf{1}
		\\
		\hline
		\boldmath{$2\times10^6$} & 1.88 & 2.30 & 2.72 & 2.72 & 2.72 & 2.72 & 2.72 & 2.73 & 2.59 & 2.42 & 2.39 & 2.01 & 1.88 & 1.71 & 1.58\\
		\hline
		\boldmath{$2.5\times10^6$}&  1.97 & 2.45 & 2.94 & 2.95 & 2.95 & 2.95 & 2.95 & 2.95 & 2.94 & 2.55 & 2.64 & 2.23 & 1.97 & 1.90 & 1.66\\
		\hline
		\boldmath{$3\times10^6$}& 2.09 & 2.60 & 3.14 & 3.14 & 3.14 & 3.14 & 3.14 & 3.11 & 3.14 & 2.71 & 2.70 & 2.44 & 2.09 & 2.03 & 1.74
		\\
		\hline
		
	\end{tabular}
}
\centering
	\makebox[\textwidth][c]{
	\scriptsize
	\begin{tabular}{|l|l|l|l|l|l|l|l|l|l|l|l|l|l|l|}
		\hline
		&
		\texttt{SACB}&
		\multicolumn{13}{c|}{\textbf{\texttt{ABSE}(\boldmath{$ \tilde \beta$}) }}
		\\
		\boldmath{$T$}  &  & \textbf{0.4} & \textbf{0.45} & \textbf{0.5} & \textbf{0.55} & \textbf{0.6} & \textbf{0.65} & \textbf{0.7} & \textbf{0.75} & \textbf{0.8} & \textbf{0.85} & \textbf{0.9} & \textbf{0.95} & \textbf{1}
		\\
		\hline
		\boldmath{$2\times10^6$} & 22\% & 44\% & 44\% & 45\% & 45\% & 45\% & 45\% & 37\% & 28\% & 27\% & 7\% & 0\% & -8\% & -15\%\\
		\hline
		\boldmath{$2.5\times10^6$} & 23\% & 49\% & 49\% & 49\% & 49\% & 49\% & 49\% & 48\% & 29\% & 33\% & 12\% & 0\% & -3\% & -15\%
		\\
		\hline
		\boldmath{$3\times10^6$} & 24\% & 49\% & 49\% & 49\% & 50\% & 50\% & 48\% & 50\% & 29\% & 28\% & 16\% & 0\% & -3\% & -16\% \\
		\hline
		
	\end{tabular}
}
		\caption{\small Results for Setting I with fixed smoothness $\beta=0.9$ and different values of horizon length $T$. \textit{Above:} Average cumulative regret divided by $10^4$; \textit{Below:} Relative Loss. }
	\label{table-exp211_fixed_beta}
\end{table}

\begin{table}[H]
\centering
\makebox[\textwidth][c]{
	\scriptsize
	\begin{tabular}{|l|l|l|l|l|l|l|l|l|l|l|l|l|l|l|l|}
		\hline
		&
		\textbf{\texttt{ABSE}(\boldmath{$ \beta$}) } &
		\texttt{SACB}&
		\multicolumn{13}{c|}{\textbf{\texttt{ABSE}(\boldmath{$ \tilde \beta$}) }}
		\\
		\boldmath{$T$} &  &  & \textbf{0.4} & \textbf{0.45} & \textbf{0.5} & \textbf{0.55} & \textbf{0.6} & \textbf{0.65} & \textbf{0.7} & \textbf{0.75} & \textbf{0.8} & \textbf{0.85} & \textbf{0.9} & \textbf{0.95} & \textbf{1}
		\\
		\hline
		\boldmath{$2\times10^6$} &  0.81 & 1.16 & 0.48 & 0.62 & 0.81 & 1.10 & 1.52 & 2.09 & 2.88 & 3.89 & 4.08 & 4.06 & 4.07 & 4.08 & 4.07\\
		\hline
		\boldmath{$2.5\times10^6$}& 0.85 & 1.42 & 0.50 & 0.64 & 0.85 & 1.17 & 1.61 & 2.23 & 3.08 & 4.21 & 5.07 & 5.08 & 5.08 & 5.09 & 5.07
		\\
		\hline
		\boldmath{$3\times10^6$}& 0.89 & 1.46 & 0.52 & 0.67 & 0.89 & 1.20 & 1.69 & 2.36 & 3.28 & 4.53 & 6.09 & 6.08 & 6.09 & 6.11 & 6.08 \\
		\hline
		
	\end{tabular}
}
\makebox[\textwidth][c]{
	\scriptsize
	\begin{tabular}{|l|l|l|l|l|l|l|l|l|l|l|l|l|l|l|}
		\hline
		&
		\texttt{SACB}&
		\multicolumn{13}{c|}{\textbf{\texttt{ABSE}(\boldmath{$ \tilde \beta$}) }}
		\\
		\boldmath{$T$}  &  & \textbf{0.4} & \textbf{0.45} & \textbf{0.5} & \textbf{0.55} & \textbf{0.6} & \textbf{0.65} & \textbf{0.7} & \textbf{0.75} & \textbf{0.8} & \textbf{0.85} & \textbf{0.9} & \textbf{0.95} & \textbf{1}
		\\
		\hline
		\boldmath{$2\times10^6$} & 42\% & -41\% & -24\% & 0\% & 35\% & 86\% & 156\% & 252\% & 376\% & 399\% & 396\% & 397\% & 399\% & 397\% \\
		\hline
		\boldmath{$2.5\times10^6$}& 67\% & -40\% & -24\% & 0\% & 37\% & 88\% & 162\% & 261\% & 394\% & 495\% & 496\% & 496\% & 498\% & 495\%
		\\
		\hline
		\boldmath{$3\times10^6$} & 64\% & -40\% & -23\% & 0\% & 35\% & 90\% & 165\% & 269\% & 409\% & 585\% & 584\% & 584\% & 587\% & 584\% \\
		\hline
		
	\end{tabular}
}
	\caption{\small Results for Setting II with fixed smoothness $\beta=0.5$ and different values of horizon length $T$. \textit{Above:} Average cumulative regret divided by $10^3$; \textit{Below:} Relative Loss. }
	\label{table-exp212_fixed_beta}
\end{table}

\subsection{Different payoff functions} \label{apendix-simulation-random-payoffs}
In addition, we consider a different Setting III with random payoffs:
	\[
	f_k(x)
	=
	\begin{cases}
	\frac{1  + L_1(\frac{1}{2})^\beta - L_1x^\beta}{2} & \text{ if } 0 \le  x \le \frac{1}{2};\\
	\frac{1}{2} + g(2x - 1)   & \text{ if } \frac{1}{2} <  x \le 1,
	\end{cases}
	\qquad \forall k \in \cK,
	\]
	where $g$ is a random fractional Brownian motion \citep{mandelbrot1968fractional} with parameter $H=\beta$ (which is guaranteed to generate a $\beta$-H\"{o}lder function almost surely), or a random Brownian Bridge (which is guaranted to generate a function that is $\beta$-H\"{o}lder for any $\beta < \frac{1}{2}$ almost surely), and scaled to the range $[-\frac{1}{2}, \frac{1}{2}]$.  We note that the first case in the definition of the above payoffs is to make sure payoffs are self-similar.

Table \ref{table-fbm} (Table \ref{table-bb}) provides the average cumulative regret and Relative Loss when payoffs are random fractional Brownian motions with $H = 0.5 $ (Brownian bridge) for a fixed horizon length $T=2\times 10^6$. We note that one could draw the same high-level conclusions that we made based on the results of Settings I and II, also through the results of this new setting. However, since the first two settings are designed to demonstrate the cost of smoothness under-estimation for \texttt{ABSE}, they tend to result in a larger gap between the performance of \texttt{ABSE}$(\tilde{\beta})$ and \texttt{ABSE}$(\beta)$. In addition, since randomly generated functions do not behave as abruptly as in Setting II, one needs less exploration to identify the sub-optimal arm in each region, and hence, over-estimation of smoothness, in the range $\tilde{\beta} < 1$ that we have considered, results in better performance.

\begin{table}[H]
	\centering
	\makebox[\textwidth][c]{
		\scriptsize
		\begin{tabular}{|l|l|l|l|l|l|l|l|l|l|l|l|l|}
			\hline
			&
			\textbf{\texttt{ABSE}(\boldmath{$ \beta$}) } &
			\texttt{SACB}&
			\multicolumn{10}{c|}{\textbf{\texttt{ABSE}(\boldmath{$ \tilde \beta$}) }}
			\\
			\boldmath{$T$} &  &  & \textbf{0.1} & \textbf{0.2} & \textbf{0.3} & \textbf{0.4} & \textbf{0.5} & \textbf{0.6} & \textbf{0.7} & \textbf{0.8} & \textbf{0.9} & \textbf{1}
			\\
			\hline
			\boldmath{$2\times10^6$} &  8.19 & 14.39 & 19.69 & 18.39 & 18.25 & 14.58 & 8.19 & 4.97 & 3.04 & 2.02 & 1.33 & 0.95 \\
			\hline

		\end{tabular}
	}
	\makebox[\textwidth][c]{
		\scriptsize
		\begin{tabular}{|l|l|l|l|l|l|l|l|l|l|l|l|l|}
			\hline
			&
			\texttt{SACB}&
			\multicolumn{10}{c|}{\textbf{\texttt{ABSE}(\boldmath{$ \tilde \beta$}) }}
			\\
			\boldmath{$T$} &  & \textbf{0.1} & \textbf{0.2} & \textbf{0.3} & \textbf{0.4} & \textbf{0.5} & \textbf{0.6} & \textbf{0.7} & \textbf{0.8} & \textbf{0.9} & \textbf{1}
			\\
			\hline
			\boldmath{$2\times10^6$} & 75\% & 140\% & 124\% & 122\% & 77\% & 0\% & -39\% & -62\% & -75\% & -83\% & -88\%\\
			\hline

		\end{tabular}
	}
	\caption{\small Results when payoffs are random fractional Brownian motion with $H=0.5$ \textit{Above:} Average cumulative regret divided by $10^4$; \textit{Below:} Relative Loss. }
	\label{table-fbm}
\end{table}

\begin{table}[H]
	\centering
	\makebox[\textwidth][c]{
		\scriptsize
		\begin{tabular}{|l|l|l|l|l|l|l|l|l|l|l|l|l|}
			\hline
			&
			\textbf{\texttt{ABSE}(\boldmath{$ \beta$}) } &
			\texttt{SACB}&
			\multicolumn{10}{c|}{\textbf{\texttt{ABSE}(\boldmath{$ \tilde \beta$}) }}
			\\
			\boldmath{$T$} &  &  & \textbf{0.1} & \textbf{0.2} & \textbf{0.3} & \textbf{0.4} & \textbf{0.5} & \textbf{0.6} & \textbf{0.7} & \textbf{0.8} & \textbf{0.9} & \textbf{1}
			\\
			\hline
			\boldmath{$2\times10^6$} &  5.59 & 5.60 & 6.02 & 5.66 & 5.62 & 5.60 & 5.59 & 5.24 & 4.26 & 3.54 & 2.67 & 2.22 \\
			\hline

		\end{tabular}
	}
	\makebox[\textwidth][c]{
		\scriptsize
		\begin{tabular}{|l|l|l|l|l|l|l|l|l|l|l|l|l|}
			\hline
			&
			\texttt{SACB}&
			\multicolumn{10}{c|}{\textbf{\texttt{ABSE}(\boldmath{$ \tilde \beta$}) }}
			\\
			\boldmath{$T$} &  & \textbf{0.1} & \textbf{0.2} & \textbf{0.3} & \textbf{0.4} & \textbf{0.5} & \textbf{0.6} & \textbf{0.7} & \textbf{0.8} & \textbf{0.9} & \textbf{1}
			\\
			\hline
			\boldmath{$2\times10^6$} & 0.12\% & 7.74\% & 1.26\% & 0.54\% & 0.18\% & 0.0\% & -6.25\% & -23.79\% & -36.7\% & -52.27\% & -60.16\%  \\
			\hline

		\end{tabular}
	}
	\caption{\small Results when payoffs are random Brownian bridges. \textit{Above:} Average cumulative regret divided by $10^4$; \textit{Below:} Relative Loss. }
	\label{table-bb}
\end{table}

\section{Adaptivity to smoothness with full-feedback} \label{appendix-full-feedback}
In this section, we consider a setting where at every time step, the agent observes the reward of both arms as opposed to the reward of the selected arm. We refer to this setting as the full-feedback setting. We establish that even with full feedback the optimal regret rate is characterized by \eqref{eq:opt_rate}:

\begin{theorem}
	Fix some \Holder\  exponent $0< \beta \le 1$, some margin parameter $0 < \alpha \le \max\l\{1,\frac{1}{\beta}\r\}$ and covariate dimension $d$. Then, in the full-feedback setting, for any horizon length $T \ge 1$ and admissible policy $\pi$, the worst-case regret is lower bounded as follows
	\[
	\sup_{\sfP \in \cP(\beta, \alpha,d)}\mathcal{R}^\pi(\sfP;T)
	\ge
	C T^{\eta(\beta, \alpha, d)},
	\]
	where $\eta(\beta, \alpha, d)$ is given in \eqref{eq:opt_rate} and $C$ is independent of $T$.
\end{theorem}
\begin{proof}
	The following proof  adopts the lower bound proof in \cite{rigollet2010nonparametric} and  consists of the following steps:
	
	\paragraph{Step 1 (Preliminaries).} Recall the definition of inferior sampling rate $\mathcal{S}^\pi(\sfP;T)$ in \eqref{eq:inferior-sampling-def}. By Lemma \ref{lemma-regret-inferior-sampling-rate}, it would suffice to show that there exists a problem instances $\sfP \in  \cP(\beta, \alpha,d)$ such that
	\[
	\mathcal{S}^\pi(\sfP;T)
	\ge
	C^\prime T^{1-\frac{\beta \alpha}{2\beta + d}},
	\]
	for some $C^\prime$ independent of $T$.  Fix a uniform covariate distribution $\bm{\mathrm{P}}_X$. For any policy $\pi$ and function $f:[0,1]^d \rightarrow [0,1]$, denote by $\cS^\pi(f;T)$ the inferior sampling rate of $\pi$ when $\bm{\mathrm{P}}_X$ is the covariate distribution, $\rElr{Y_{1,t} \;\middle|\; X_t} = f(X_t)$, and $\rElr{Y_{2,t} \;\middle|\; X_t} = \frac{1}{2}$.  Notably, the oracle policy $\pi_f^\ast$ is given by $\pi_f^\ast(x) = 2-\Indlr{f(x) \ge \frac{1}{2}}$. We further denote by $\rP_{\pi,f}$ and $\rE_{\pi,f}$ the corresponding probability and expectation.
	
	\paragraph{Step 2 (Problem construction).}  Let $\Delta = T^{-\frac{\beta}{2\beta + d}}$, $M=\mu \Delta^{\alpha - \frac{d}{\beta}}$ for some constant $\mu > 0$, and \vspace{-0.2cm}
	\begin{align*}
	\psi(x)
	\coloneqq
	\begin{cases}
	\l| 1-\|x\|_\infty \r|^\beta & \text{if } 0\le \|x\|_\infty \le 1;\\
	0 & \text{o.w.}
	\end{cases}
	\end{align*}
	Note that $\psi\in \cH_{\rR^d}(\beta,1)$. Define the hypercube $H_0 \coloneqq [0, \Delta^{\frac{\alpha}{d}}]^d$, and let grid $G$ partition this hypercube into $M $ disjoint hypercubes $\l( H_m \r)_{m\in \{1,\dots,M\}}$ of equal side-length. Let $a_m\in \rR^d, m\in \{1,\dots,M\}$, be the center of the hypercube $H_m$. Define the function\vspace{-0.15cm}
	\[
	\phi_m(x) \coloneqq \Delta \cdot \psi_{\gamma}\l(  \Delta^{-\frac{1}{\beta}} [x-a_m]\r) .\vspace{-0.15cm}
	\]
	By Lemmas \ref{lemma-multip-presrves-smoothness}, $\phi_m \in \cH(\beta,L)$ since $C_{\phi}\le L$.  For a given vector $\vectorgreek{\omega} \in \{-1,1\}^M$ with elements $\omega_m$, define the first arm's payoff function as follows:
	\[
	f^{\vectorgreek{\omega}}(x)
	\coloneqq
	\frac{1}{2} + \sum_{m=1}^M C_\phi \omega_m\phi_m(x).
	\]
	
	\paragraph{Step 3 (Desirable event).} For $m \in \{1,\dots,M\}$, define
	$
	Q_{m}\coloneqq \sum_{t=1}^T  \Indlr{X_t \in H_m} \eqqcolon  \sum_{t=1}^T Z_{m,t}
	$
	to be the number of times periods at which the realized covariates belong to the hypercube $H_m$. Define
	$
	\cA\coloneqq \l\{\exists m \in \{1,\dots,M\}:  Q_{m} < \frac{1}{2}T\Delta^{\frac{d}{\beta}} \text{ or } Q_{m} > \frac{3}{2}T\Delta^{\frac{d}{\beta}} \r\}
	$
	to be the event where $Q_{m}$ is less than $\frac{1}{2}T\Delta^{\frac{d}{\beta}} $ or larger than $\frac{3}{2}T\Delta^{\frac{d}{\beta}} $ for at least one value of $m\in\{1,\dots,M\}$. Note that\vspace{-0.15cm}
	\[
	\rPlr{\cA}
	\le
	\sum_{m=1}^{M}
	\rPlr{Q_{m} < \frac{1}{2}T\Delta^{\frac{d}{\beta}} } + \rPlr{Q_{m} < \frac{3}{2}T\Delta^{\frac{d}{\beta}} }.\vspace{-0.1cm}
	\]
	In order to bound each of the summands on the right hand side of the above inequality, one may apply Bernstein's inequality in Lemma \ref{lemma-Bernstein-inequality} to $Q_m$: Note that since $ \Delta^{\frac{d}{\beta}}  T\le \rE Z_{m,t} \le T\Delta^{\frac{d}{\beta}} $, $|Z_{m,t}|\le 1$, and $\Var Z_{m,t} \le \rE Z_{m,t}^2 \le 2\bar{\rho}\Delta^{\frac{d}{\beta}} $, one obtains:\vspace{-0.2cm}
	\begin{align*}
	\rPlr{\cA}
	&
	\overset{(a)}{\le}
	M\exp\l(-c_1T\Delta^{\frac{d}{\beta}} /5\r)
	\\
	&
	\overset{}{\le}
	\mu T^{\frac{\alpha\beta -d}{2\beta +d}}
	\exp\l( -c_2 T^{\frac{2\beta}{2\beta+d}} \r)
	\overset{}{\le}
	c_3 T^{-3},
	\end{align*}
	for large enough $T$ and constants $c_1,c_2,c_3>0$, where (a) follows from the definition of $M$ and $\Delta$.
	For any problem instance $\sfP$ and horizon length $T$, denote the inferior sampling rate of $\pi$ when the event $\cA$ does not occur by\vspace{-0.15cm}
	\[
	\bar{\cS}^\pi(\sfP; T) \coloneqq \mathbb{E}^\pi \left[ \sum\limits_{t=1}^T \Indlr{ f_{\pi^\ast_t}(X_t) \neq f_{\pi_t}(X_t)} \;\middle |\; \bar{\cA} \; \right].\vspace{-0.15cm}
	\]
	 Note that\vspace{-0.15cm}
	\begin{equation*}\label{eq-inf-sampling-rate-A-fails-full-feed}
	\l( 1- \rPlr{\cA} \r)\bar{\cS}^\pi(\sfP; T)
	\le
	\cS^\pi(\sfP; T)
	\le
	\bar{\cS}^\pi(\sfP; T) + T \rPlr{\cA},\vspace{-0.2cm}
	\end{equation*}
	which implies that\vspace{-0.05cm}
	\begin{equation}\label{eq-relationship-S-and-S-bar-full-feed}
	\l| \bar{\cS}^\pi(\sfP; T)
	-
	\cS^\pi(\sfP; T)\r|
	\le c_4T^{-2}.\vspace{-0.05cm}
	\end{equation}
	For the rest of the proof, all probabilities and expectations will be computed conditional on $\bar \cA$.

	\paragraph{Step 4 (Lower bounding information sampling rate).} 	One has
	\begin{align}\label{eq-full-feed-lower-1}
		\sup_{\vectorgreek{\omega} \in \{-1,1\}^M}\bar{\cS}^\pi(f^{\vectorgreek{\omega}}; T)
		&=
		\sup_{\vectorgreek{\omega} \in \{-1,1\}^M}
		\sum_{t = 1}^T \rE_{\pi,f^\vomega} \l[\Indlr{2\pi_t(X_t) \neq 3 - \sgn(f^{\vectorgreek{\omega}}(X_t))} \;\middle |\; \bar{\cA} \; \r]
		\nonumber
		\\
		&=
		\sup_{\vectorgreek{\omega} \in \{-1,1\}^M}
		\sum_{m=1}^M\sum_{t = 1}^T \rE_{\pi,f^\vomega} \l[\Indlr{2\pi_t(X_t) \neq 3 - \omega_m, X_t \in H_m} \;\middle |\; \bar{\cA} \; \r]
		\nonumber
		\\
		&\ge
		\frac{1}{2^M}
		\sum_{m=1}^M\sum_{t = 1}^T\sum_{\vectorgreek{\omega} \in \{-1,1\}^M} \rE_{\pi,f^\vomega} \l[\Indlr{2\pi_t(X_t) \neq 3 - \omega_m, X_t \in H_m} \;\middle |\; \bar{\cA} \; \r]
	\end{align}
	Now observe that the summation $\sum_{\vectorgreek{\omega} \in \{-1,1\}^M}[\dots]$ can be decomposed as
	\[
	R_{m,t}
	\coloneqq
	\sum_{\vectorgreek{\omega}_{[-m]} \in \{-1,1\}^{M-1}} \sum_{i \in _\{-1,1\}} \rE_{\pi,f^{\vomega_{[-m]}^i}} \l[\Indlr{2\pi_t(X_t) \neq 3 - i, X_t \in H_m} \;\middle |\; \bar{\cA} \; \r],
	\]
	where $\vomega_{[-m]} = (\omega_1, \dots, \omega_{m-1}, \omega_{m+1},\dots, \omega_M)$ and $\vomega_{[-m]}^{i} = (\omega_1, \dots, \omega_{m-1}, i, \omega_{m+1},\dots, \omega_M)$ for $i\in \{-1, 1\}$. Using Lemma \ref{Tsybakov2008introduction-lemma-2.6}, and denoting by $\bar{P}_{X}^m$ the conditional probability $\rPlr{\cdot \;\middle |\; \bar{\cA}, X_t \in H_m \; }$, one has
	\begin{align}\label{eq-full-feed-lower-2}
		\sum_{i \in _\{-1,1\}} \rE_{\pi,f^{\vomega_{[-m]}^i}} \l[\Indlr{2\pi_t(X_t) \neq 3 - i, X_t \in H_m} \;\middle |\; \bar{\cA} \; \r]
		&=
		\Delta^{\frac{d}{\beta}} \sum_{i \in _\{-1,1\}} \rE_{\pi,f^{\vomega_{[-m]}^i}} \l[\bar{P}_{X}^m\l\{2\pi_t(X_t) \neq 3 - i\r\}\;\middle |\; \bar{\cA} \;  \r]
		\nonumber
		\\
		&\ge
		\Delta^{\frac{d}{\beta}} \exp\l[\mathrm{KL}\l(\bar{\rP}^t_{\pi,f^{\vomega_{[-m]}^{-1}}} \times \bar{P}_{X}^m, \bar{\rP}^t_{\pi,f^{\vomega_{[-m]}^1}}\times \bar{P}_{X}^m\r)\r]
		\nonumber
		\\
		&=
		\Delta^{\frac{d}{\beta}} \exp\l[\mathrm{KL}\l(\bar{\rP}^t_{\pi,f^{\vomega_{[-m]}^{-1}}}, \bar{\rP}^t_{\pi,f^{\vomega_{[-m]}^1}}\r)\r],
	\end{align}
	where the probability measures $\bar{\rP}_{\pi,f^{\vomega_{[-m]}^{i}}}$, for $ i\in \{-1,1\}$ are conditional on the event $\bar{\cA}$. For any $t=2,\dots, T$, let $\cF_t$ denote the $\sigma$-algebra generated by the information available at time $t$ immediately after observing $X_t$, i.e., $\cF_t = \sigma\l(X_t,\l(\pi_s, X_s,Y_{\pi_s,s}\r)_{s=1,\dots,t-1} \r)$. Define the conditional distribution $\bar{\rP}^{\cdot | \cF_t}_{\pi,f}$ of the random couple $(X_t, Y_{\pi_t,t})$ conditioned on $\cF_t$ and the event $\bar{\cA}$. Applying the chain rule for KL divergence, we find that for any $t=1,\dots, T$ and any functions $f,g: \cX \rightarrow [0,1]$, we have
	\begin{align*}
		\mathrm{KL}\l(\bar{\rP}^t_{\pi,f}, \bar{\rP}^t_{\pi,g}\r)
		&=
		\mathrm{KL}\l(\bar{\rP}^{t-1}_{\pi,f}, \bar{\rP}^{t-1}_{\pi,g}\r) +
		\bar{\rE}^{t-1}_{\pi,f}\l[	\mathrm{KL}\l(\bar{\rP}^{\cdot | \cF_t}_{\pi,f}, \bar{\rP}^{\cdot | \cF_t}_{\pi,g}\r) \r]
		\\
		&=
		\mathrm{KL}\l(\bar{\rP}^{t-1}_{\pi,f}, \bar{\rP}^{t-1}_{\pi,g}\r) +
		\bar{\rE}^{t-1}_{\pi,f}\l[	\mathrm{KL}\l(\bar{\rP}^{Y_{\pi_t,t} | \cF_t}_{\pi,f}, \bar{\rP}^{Y_{\pi_t,t} | \cF_t}_{\pi,g}\r) \r],
	\end{align*}
	where $\bar{\rP}^{Y_{\pi_t,t} | \cF_t}_{\pi,f}$ denotes the conditional distribution of $Y_{\pi_t,t}$ given $\cF_t$ and $\bar{A}$. Since for any $\vomega$ we have $\rElr{Y_{\pi_t,t}} \in (\frac{1}{2} - \tau, \frac{1}{2} + \tau)$ for some $\tau \in (0,\frac{1}{2})$, we can apply Lemma 4.1 in \cite{rigollet2010nonparametric} to derive the following bound:
	\begin{align*}
		\l[	\mathrm{KL}\l(\bar{\rP}^{Y_{\pi_t,t} | \cF_t}_{\pi,f^{\vomega_{[-m]}^{-1}}}, \bar{\rP}^{Y_{\pi_t,t} | \cF_t}_{\pi,f^{\vomega_{[-m]}^{1}}}\r) \r]
		&\le
		\frac{1}{\kappa^2} \l(f^{\vomega_{[-m]}^{-1}} - f^{\vomega_{[-m]}^{1}} \r)^2 \Indlr{\pi_t(X_t)=1, X_t \in H_m}
		\\
		&\le
		\frac{4C_\phi^2\Delta^2}{\kappa^2} \l(f^{\vomega_{[-m]}^{-1}} - f^{\vomega_{[-m]}^{1}} \r)^2 \Indlr{X_t \in H_m}.
	\end{align*}
	By induction, the last two displays yield that for any $t = 1,\dots, T$,
	\[
	\mathrm{KL}\l(\bar{\rP}^t_{\pi,f^{\vomega_{[-m]}^{-1}}}, \bar{\rP}^t_{\pi,f^{\vomega_{[-m]}^1}}\r)
	\le
	\frac{4C_\phi^2\Delta^2}{\kappa^2} \sum_{t=1}^T \Indlr{X_t \in H_m}
	\overset{(a)}{\le}
	\frac{6\bar{\rho}C_\phi^2\Delta^2T\Delta^{\frac{d}{\beta}}}{\kappa^2}
	\overset{(b)}{=}
	\frac{6\bar{\rho}C_\phi^2}{\kappa^2},
	\]
	where (a) follows since we assume event $\bar{\cA}$ holds and (b) follows from the definition of $\Delta$. Combining the above inequality and \eqref{eq-full-feed-lower-2} one has
	\[
	R_{m,t}
	\ge
	\frac{2^{M-1}\Delta^{\frac{d}{\beta}}}{4} \exp\l(-	\frac{6\bar{\rho}C_\phi^2}{\kappa^2}\r).
	\]
	This inequality along with \eqref{eq-full-feed-lower-1}, results in the desired lower bound for the inferior sampling rate:
	\[
	\sup_{\vectorgreek{\omega} \in \{-1,1\}^M}\bar{\cS}^\pi(f^{\vectorgreek{\omega}}; T)
		\ge
		\frac{2^{M-1}MT\Delta^{\frac{d}{\beta}}}{4\cdot 2^M} \exp\l(-	\frac{6\bar{\rho}C_\phi^2}{\kappa^2}\r)
		\ge C^\prime T^{1-\frac{\beta \alpha}{2\beta+d}},
	\]
	for some constant $C^\prime$ independent of $T$. This inequality along with \eqref{eq-relationship-S-and-S-bar-full-feed} gives the desired result for large enough $T$. This concludes the proof.
\end{proof}

We next detail a meta policy (see Algorithm \ref{algorithm-FFAB}) that is smoothness-adaptive under the full-feedback assumption. This policy  integrates some collection of non-adaptive policies $\{\pi_0(\beta_0)\}_{\beta_0 \in [\lbeta , \ubeta]}$ that are rate-optimal under accurate tuning of the smoothness parameter.

The key idea of this policy is to consider a collection of smoothness parameters $\beta_j
\leftarrow
\lbeta + j \delta \beta_T
$ over interval $[\lbeta, \ubeta]$ and to initialize a separate policy $\pi_0(\beta_j)$ for each smoothness parameter. Our meta policy keeps track of the regret that each of these policies have incurred up to each time step $t$ (which can be computed due to the full-feedback assumption) and selects the arm suggested by the policy that has incurred minimum regret so far. Let $j^\ast \coloneqq \max\l\{j: \beta_j \le \beta\r\}$ be the largest $j$ for which $\beta_j$ is smaller than the true smoothness parameter $\beta$. We know that policy $\pi_0(\beta_{j^\ast})$ will incur a regret that is near-optimal. As a result, by following the policy that has incurred minimum regret up to each time step, we will make sure that the cumulative regret of our meta policy will not exceed the regret of $\pi_0(\beta_{j^\ast})$ up to a multiplicative factor which depends on the number of policies $\pi_0(\beta_{j})$. We next formalize this policy.

\bigskip
\begin{algorithm}[H]\footnotesize
	\caption{Full-Feedback Adaptive Bandits}\label{algorithm-FFAB}
	\textbf{Input:} Set of non-adaptive policies $\{\pi_0(\beta_0)\}_{\beta_0 \in [\lbeta , \ubeta]}$, horizon length $T$, minimum and maximum smoothness exponents $\underline{\beta}$ and $\ubeta$, a tuning parameter $\gamma$.
	\\
	\textbf{Initialize:} $\delta\beta_T \leftarrow  \frac{(2\overline{ \beta} + d)^2}{(\underline{\beta}+d-1)\log T}$, $
	\beta_j
	\leftarrow
	\lbeta + j \delta \beta_T
	$ and $R_{j, 1} \leftarrow 0$ for $j\in \l\{0, 1, \dots, \left\lfloor \frac{(\underline{\beta}+d-1)(\ubeta - \lbeta)\log T}{(2\overline{ \beta} + d)^2} \right\rfloor\r\}$\\
	\For{$t = 1, \dots$}{
		Determine the policy with minimum regret so far: $j_t \leftarrow \argmin R_{j, t}$ \\		
		Pull the arm suggested by the policy $\pi_0(\beta_{j_t})$: $\pi_t \leftarrow \pi_0(\beta_{j_t}, X_t)$ \\	
		Observe feedback $Y_{k,t}$ for all $k \in \cK$	\\
		Advance the policy $\pi_0(\beta_{j_t})$ by feeding observation $Y_{\pi_t,t}$ into it\\
		Update the regret of the policy $\pi_0(\beta_{j_t})$: $R_{j,t} \leftarrow R_{j, t - 1} + \Indlr{j_t = j} \cdot\l( \max_{k\in \cK} Y_{k,t} - Y_{\pi_t,t} \r)$
			}
\end{algorithm}\normalsize
\bigskip

 The next theorem establishes that, when coupled with appropriate off-the-shelf non-adaptive policies, the above meta policy guarantees optimal regret rate up to poly-logarithmic terms, and smoothness-adaptive performance as stated in Definition \ref{def:adaptivity}.

\begin{theorem}[Smoothness-adaptive policy under full-feedback]\label{theorem-GSE-ABSE-regret-bound}
	Let $\pi$ be the \texttt{SACB} policy detailed in Algorithm~\ref{algorithm-FFAB}, and let $\{\pi_0(\beta_0)\}_{\beta_0 \in [\lbeta , \ubeta]}$ be a set of non-adaptive policies such that if initialized with the true smoothness parameter, for any $\lbeta \le \beta_0 \le \ubeta$, $\alpha \le \frac{1}{\min\{1,\beta_0 \}}$, and $T\ge 1$, it satisfies the following upper bound on the regret:\vspace{-0.15cm}
	\[
	\sup_{\sfP \in \cP(\beta_0, \alpha, d)}\mathcal{R}^{\pi_0(\beta_0)}(\sfP;T)
	\le
	\bar{C}_0 \l( \log T \r)^{\iota_0(\beta_0, \alpha, d)} T^{\zeta(\beta_0, \alpha,d)},\vspace{-0.05cm}
	\]
	for some $\iota_0(\beta_0, \alpha, d)$ and a constant $\bar{C}_0>0$ that is independent of $T$, where the function $\zeta(\beta_0, \alpha,d)$ is given in (\ref{eq:opt_rate}).
	Then, there exists $\bar{C}>0$, such that for any problem instance $\sfP\in \cP(\beta, \alpha, d)$ with $\lbeta \le \beta \le \ubeta$, $\alpha \le \frac{1}{\min\{1,\beta \}}$ and any horizon length $T$:
	\vspace{-0.10cm}
	\[
	\mathcal{R}^\pi(\sfP;T)
	\le
	\bar{C}  T^{\zeta(\beta, \alpha,d)} \l( \log T \r)^{1+\iota_0(\beta - \delta\beta, \alpha, d)} .\vspace{0.05cm}
	\]
\end{theorem}

\begin{proof}
	For each $j$, let $\bar{t}_j$ be the last time step the corresponding policy has been selected. Recall the definition of $j^\ast \coloneqq \max\l\{j: \beta_j \le \beta\r\}$ which is the largest $j$ for which $\beta_j$ is smaller than the true smoothness parameter $\beta$. For each $j$, one has\vspace{-0.15cm}
	\[
	\rElr{R_{j, \bar{t}_j}}
	 \overset{(a)}{\le}
	  \rElr{R_{j, \bar{t}_j - 1}} + 1
	  \overset{(b)}{\le}
	  \rElr{R_{j^\ast, \bar{t}_j - 1}} + 1
	  \le
	  \rElr{R_{j^\ast, \bar{t}_{j^\ast} - 1}} + 1
	  \le
	  \bar{C}_0 \l( \log T \r)^{\iota_0(\beta_{j^\ast}, \alpha, d)} T^{\zeta(\beta_{j^\ast}, \alpha,d)} + 1.\vspace{-0.15cm}
	\]
	Noting $\beta_{j^\ast} \ge \beta - \delta\beta_T$ and that $\mathcal{R}^\pi(\sfP;T) = \sum_{j} \rElr{R_{j, \bar{t}_j}}$, the result follows.
\end{proof}
\color{black}

\end{document}